\documentclass[11pt]{article}

\usepackage[margin=1in]{geometry}

\renewcommand{\paragraph}{\vspace{0.2cm} \noindent \textbf}

\usepackage{authblk}
\usepackage{amsmath}
\usepackage{amsthm}
\usepackage{amssymb}
\usepackage{thmtools}
\usepackage{xcolor}

\usepackage{natbib}

\usepackage{enumitem}

\usepackage{hyperref}

\usepackage[utf8]{inputenc} 
\usepackage[T1]{fontenc}    
\usepackage{hyperref}       
\usepackage{url}            
\usepackage{booktabs}       
\usepackage{amsfonts}       
\usepackage{nicefrac}       
\usepackage{microtype}      
\usepackage{xcolor}         
\usepackage{graphicx}

\usepackage{subcaption}

\usepackage{thmtools}

\newtheorem{theorem}{Theorem}

\newtheorem{lemma}[theorem]{Lemma}
\newtheorem{definition}{Definition}
\newtheorem{claim}[theorem]{Claim}

\newtheorem{corollary}[theorem]{Corollary}

\newtheorem{example}{Example}
\newtheorem{assumption}{Assumption}
\renewcommand\thmcontinues[1]{Continued}

\newcommand{\scalingexp}{\nu}

\newcommand{\Nlead}{N_I}
\newcommand{\Nentr}{N_E}
\newcommand{\Nentropt}{N_E^*}
\newcommand{\Nentroptmodified}{\tilde{N}_E^{*}}
\newcommand{\constrlead}{\tau_I}
\newcommand{\constrentr}{\tau_E}
\newcommand{\param}{\beta}
\newcommand{\paramone}{\beta_1}
\newcommand{\paramtwo}{\beta_2}
\newcommand{\estparam}{\hat{\beta}}

\newcommand{\DF}{\mathcal{D}_F}
\newcommand{\DC}{\mathcal{D}_W}

\newcommand{\LossPerf}{L_1}
\newcommand{\LossAlign}{L_2}

\newcommand{\mix}{\alpha}
\newcommand{\reg}{\lambda}
\newcommand{\mixlead}{\alpha_I}
\newcommand{\reglead}{\lambda_I}
\newcommand{\mixentr}{\alpha_E}
\newcommand{\regentr}{\lambda_E}

\newcommand{\Matrixdiff}{B^{\texttt{df}}}
\newcommand{\Matrixhom}[1]{B^{\texttt{sn}}}
\newcommand{\Matrixmixed}[1]{B^{\texttt{mx}}}

\newcommand{\DegreesFreedomStandard}{Q}

\DeclareMathOperator{\Tr}{Tr}

\DeclareMathOperator{\argmin}{argmin}

\def\comments{1}

\ifnum\comments=1
    \newcommand{\mynote}[2]{\marginpar{\color{#1}\sf \tiny #2}}
    \newcommand{\myinline}[2]{{\color{#1}\sf [{#2}]}}
\else
    \newcommand{\mynote}[2]{}
    \newcommand{\myinline}[2]{}    
\fi

\hypersetup{
    colorlinks,
    linkcolor={red!50!black},
    citecolor={blue!50!black},
    urlcolor={blue!80!black}
}

\title{Safety vs.~Performance: How Multi-Objective Learning Reduces Barriers to Market Entry}

\author{Meena Jagadeesan}
\author{Michael I. Jordan}
\author{Jacob Steinhardt}

\affil{\textit{University of California, Berkeley}}

\begin{document}

\maketitle

\begin{abstract}
Emerging marketplaces for large language models and other large-scale machine learning (ML) models appear to exhibit market concentration, which has raised concerns about whether there are insurmountable barriers to entry in such markets. In this work, we study this issue from both an economic and an algorithmic point of view, focusing on a phenomenon that \textit{reduces} barriers to entry.  Specifically, an incumbent company risks reputational damage unless its model is sufficiently aligned with safety objectives, whereas a new company can more easily avoid reputational damage. To study this issue formally, we define a multi-objective high-dimensional regression framework that captures reputational damage, and we characterize the number of data points that a new company needs to enter the market. Our results demonstrate how multi-objective considerations can fundamentally reduce barriers to entry---the required number of data points can be significantly smaller than the incumbent company's dataset size. En route to proving these results, we develop scaling laws for high-dimensional linear regression in multi-objective environments, showing that the scaling rate becomes slower when the dataset size is large, which could be of independent interest.
    
\end{abstract}

\section{Introduction}

Large language models and other large-scale machine learning (ML) models have led to an important shift in the information technology landscape, one which has significant economic consequences.  Whereas earlier generations of ML models provided the underpinnings for platforms and services, new models---such as language models---are themselves the service. This has led to new markets where companies offer language models as their service and compete for user usage. As in other markets, it is important to reason about market competitiveness: in particular, to what extent there are barriers to entry for new companies.

A widespread concern about these markets is that new companies face insurmountable \textit{barriers to entry} that drive  market concentration \citep{VK23}. The typical argument is that incumbent companies with high market share can purchase or capture significant amounts of data and compute,\footnote{Large companies can afford these resources since the marketplace is an economy of scale (i.e., fixed costs of training significantly exceed per-query inference costs). They also generate high volumes of data from user interactions.} and then invest these resources into the training of models that achieve even higher performance \citep{K20}. This suggests that the company's market share would further increase, and that the scale and scope of this phenomenon would place incumbent companies beyond the reach of new companies trying to enter the market. The scale is in fact massive---language assistants such as ChatGPT and Gemini each have hundreds of millions of users \citep{cook_2024_ai_tools}. In light of the concerns raised by policymakers \citep{VK23} and regulators \citep{whitehouse_2023_aG_Exec_order, eu_digital_markets_act_2022} regarding market concentration, it is important to investigate the underlying economic and algorithmic mechanisms at play.
 
While standard arguments assume that market share is determined by model performance, the reality is that the incumbent company risks reputational damage if their model violates safety-oriented objectives. For example, incumbent companies face public and regulatory scrutiny for their model's safety violations---such as threatening behavior \citep{timegpt}, jailbreaks \citep{WHS23}, and releasing dangerous information \citep{whitehouse_2023_aG_Exec_order}---even when the model performs well in terms of helpfulness and usefulness to users. In contrast, new companies face less regulatory scrutiny since compliance requirements often prioritize models trained with more resources \citep{whitehouse_2023_aG_Exec_order, CaliforniaLegislation2024}, and new companies also may face less public scrutiny given their smaller user bases. 

In this work, we use a multi-objective learning framework to show that the threat of reputational damage faced by the incumbent company can reduce barriers to entry. For the incumbent, the possibility of reputational damage creates pressure to align with safety objectives in addition to optimizing for performance. Safety and performance are not fully aligned, so improving safety can reduce performance as a side effect. Meanwhile, the new company faces less of a risk of reputational damage from safety violations. The new company can thus enter the marketplace with significantly less data than the incumbent company, a phenomenon that our model and results formalize. 

\paragraph{Model and results.}
We analyze a stylized marketplace based on multi-objective linear regression (Section \ref{sec:model}). The performance-optimal output and the safety-optimal output are specified by two different linear functions of the input $x$. The marketplace consists of two companies: an incumbent company and a new company attempting to enter the market. Each company receives their own unlabelled training dataset, decides what fraction of training data points to label according to the performance-optimal vs. safety-optimal outputs, and then runs ridge regression. The new company requires a less stringent level of safety to avoid reputational damage than the incumbent company. We characterize the \textit{market-entry threshold} $\Nentropt$ (Definition \ref{def:marketentrythreshold}) which captures how much data the new company needs to outperform the incumbent company.

First, as a warmup, we characterize $\Nentropt$ when the new company faces no safety constraint and the incumbent company has infinitely many data points (Section \ref{sec:warmup}). Our key finding is that the new company can enter the market with finite data, even when the incumbent company has infinite data (Theorem \ref{thm:tradeoffwarmup}; Figure \ref{fig:warmup}). Specifically, we show that the threshold $\Nentropt$ is finite; moreover, it is increasing in the correlation (i.e., the alignment) between performance and safety, and it is decreasing in a problem-specific scaling law exponent. 

Next, we turn to more general environments where the incumbent has finite data $\Nlead < \infty$ (Section \ref{subsec:finitedata}). We find that the threshold $\Nentropt$ scales sublinearly with the incumbent's dataset size $\Nlead$, as long as $\Nlead$ is sufficiently large. In fact, the threshold $\Nentropt$ scales at a slower rate as $\Nlead$ increases: that is, $\Nentropt = \Theta(\Nlead^c)$ where the exponent $c$ is decreasing in $\Nlead$ (Theorem \ref{thm:finitedata}; Figure \ref{fig:boundB}). For example, for concrete parameter settings motivated by language models \citep{H22}, the exponent $c$ decreases from $1$ to $0.75$ to $0$ as $\Nlead$ increases. In general, the exponent $c$ takes on up to three different values depending on $\Nlead$, and is strictly smaller than $1$ as long as $\Nlead$ is sufficiently large.

Finally, we turn to environments where the new company also faces a nontrivial safety constraint, assuming for simplicity that the incumbent company again has infinite data (Section \ref{subsec:alignment}). We find that $\Nentropt$ is finite as long as the new company faces a strictly weaker safety constraint than the incumbent. When the two safety thresholds are closer together, the new company needs more data and in fact needs to scale up their dataset size at a faster rate: that is, $\Nentropt = \Theta(D^{-c})$, where $D$ measures the difference between the safety thresholds and where the exponent $c$ increases as $D$ decreases 
(Theorem \ref{thm:alignment}; Figure \ref{fig:boundtau}). For the parameter settings in \citep{H22}, the exponent $c$ changes from $-2.94$ to $-3.94$ to an even larger value as $D$ decreases. In general, the exponent $c$ takes on up to three different values.

\paragraph{Technical tool: Scaling laws.} 
To prove our results, we derive scaling laws for \textit{multi-objective} high-dimensional linear regression, which could be of independent interest (Section \ref{subsec:technicaltool}; Figure \ref{fig:scalinglaw}). We study optimally-regularized ridge regression where some of the training data is labelled according to the primary linear objective (capturing performance) and the rest is labelled according to an alternate linear objective (capturing safety).

We characterize data-scaling laws for both the loss along the primary objective and the excess loss along the primary objective relative to an infinite-data ridgeless regression. Our scaling laws quantify the rate at which the loss (Theorem \ref{thm:scalinglawoptreginformal}; Figure \ref{fig:loss}) and the excess loss (Theorem \ref{thm:scalinglawoptregexcessinformal}; Figure \ref{fig:excessloss}) decay with the dataset size $N$, and how this rate is affected by the fraction of data labelled according to each objective and other problem-specific quantities. Our analysis improves upon recent works on scaling in multi-objective environments \citep[e.g.,][]{JMS24, SBS24} by allowing for non-identity covariances and problem-specific regularization, which leads to new insights about scaling laws as we describe below.  

Our results reveal that the scaling rate becomes slower as the dataset size increases, illustrating that multi-objective scaling laws behave qualitatively differently from classical single-objective environments. While a typical scaling exponent in a single-objective environment takes on a single value across all settings of $N$, the scaling exponent for multi-objective environments decreases as $N$ increases. In particular, the scaling exponent takes on \textit{three different values} depending on the size of $N$ relative to problem-specific parameters. 
The intuition is that the regularizer must be carefully tuned to $N$ in order to avoid overfitting to training data labelled according to the alternate objective, which in turn results in the scaling exponent being dependent on $N$ (Section \ref{sec:scalinglaws}).

\paragraph{Discussion.}
Altogether, our work highlights the importance of looking beyond model performance when evaluating market entry in machine learning marketplaces. Our results highlight a disconnect between market entry in single-objective environments versus more realistic multi-objective environments. More broadly, a company's susceptibility to reputational damage affects how they train their model to balance between different objectives. As we discuss in Section \ref{sec:discussion}, these insights have nuanced implications for regulators who wish to promote both market competitiveness and safety compliance, and also generalize beyond language models to online platforms.

\subsection{Related work}

Our work connects to research threads on \textit{competition between model providers} as well as \textit{scaling laws and high-dimensional linear regression}.

\paragraph{Competition between model providers.} Our work contributes to an emerging line of work studying how competing model providers strategically design their machine learning pipelines to attract users. Model-provider actions range from choosing a function from a model class \citep{BT17, BT19, JJSH23}, to selecting a regularization parameter \citep{IK22}, to choosing an error distribution over user losses \citep{FGHJN19}, to making data purchase decisions \citep{DEJKS19, KGZ22}, to  deciding whether to share data \citep{GT23}, to selecting a bandit algorithm \citep{AMSW20, JJH23}. While these works assume that model providers win users solely by maximizing (individual-level or population-level) accuracy, our framework incorporates the role of \textit{safety violations} in impacting user retention implicitly via reputational damage. Moreover, our focus is on quantifying the barriers to market entry, rather than analyzing user welfare or the equilibrium decisions of model providers.

Other related work includes the study of competition between algorithms \citep{IKLMPT11, KR21}, retraining dynamics under user participation decisions \citep{HSNL18, GZKZ21, DCRMF22, SD24, SuD24}, the bargaining game between a foundation model company and a specialist \citep{LKH24}, and the market power of an algorithmic platform to shape user populations \citep{PZMH20, HJM22, MCH2024}. 

Our work also relates to platform competition \citep{J21, C21}, the emerging area of competition policy and regulation of digital marketplaces \citep{stiger19, VK23, Cen2023, UK2023AIFoundationModels}, the study of how antitrust policy impacts innovation in classical markets \citep{B07, SW07}, and industrial organization more broadly \citep{T88}. For example, recent work examines how increased public scrutiny from inclusion in the S\&P 500 can harm firm performance \citep{Bennett2023}, how privacy regulation impacts firm competition \citep{GA20, FJMM24}, how regulatory inspections affect incentives to comply with safety constraints \citep{H88, FJ23},  and how data-driven network effects can reduce innovation \citep{PS21}. 

\paragraph{Scaling laws and high-dimensional linear regression.} Our work also contributes to an emerging line of work on scaling laws which study how model performance changes with training resources. Empirical studies have demonstrated that increases to scale often reliably improve model performance \citep[e.g.,][]{K20, hernandez2021scaling, H22}, but have also identified settings where scaling behavior is more nuanced \citep[e.g.,][]{MRBST23, GSH23}. We build on a recent mathematical characterization of scaling laws based on high-dimensional linear regression 
~\citep[e.g.,][]{HMRT19, BCP20, B21, CLKZ21, WHS22, bach, weithesis, PDT24, BAP24, MZFY24, LWKBL24, AZP24}. However, while these works focus on \textit{single-objective} environments where all of the training data is labelled with outputs from a single predictor, we consider \textit{multi-objective} environments where some fraction of the training data is labelled according to an alternate predictor. 

We note that a handful of recent works similarly move beyond single-objective environments and study scaling laws where the training data comes a mixture of different data sources. \citet{JMS24, SBS24} study high-dimensional ridge regression in a similar multi-objective environment to our setup. However, these results assume an \textit{identity covariance} and focus on fixed regularization or no regularization. In contrast, we allow for richer covariance matrices that satisfy natural power scaling (Section \ref{subsec:assumptions}), and we analyze optimally tuned regularization. Our analysis of these problem settings yields new insights about scaling behavior: for example, the scaling rate becomes slower with dataset size (Theorems \ref{thm:scalinglawoptreginformal}-\ref{thm:scalinglawoptregexcessinformal}). Other related works study scaling laws under mixtures of covariate distributions
\citep{H21}, under data-quality heterogeneity \citep{GMLRK24}, under data addition \citep{shen2024data}, under mixtures of AI-generated data and real data \citep{DFYCK24, GSD24}, and with respect to the contribution of individual data points \citep{CJHZ24}.

More broadly, our work relates to collaborative learning \citep{BHPQ17, MSS19, SKHL20, HJZ22}, federated learning~\citep[see][for a survey]{YLCT19}, optimizing data mixtures~\citep[e.g.,][]{RWRJ21, XY23}, and adversarial robustness~\citep[e.g.,][]{RXYDL20}. 
Finally, our work relates to non-monotone scaling laws in strategic environments \citep{JJH23, HM24}, where increases to scale can worsen equilibrium social welfare.

\section{Model}\label{sec:model}

We define our linear-regression-based marketplace (Section \ref{subsec:linearregression}),   justify the design choices of our model (Section \ref{subsec:discussion}), and then delineate our statistical assumptions (Section \ref{subsec:assumptions}).

\subsection{Linear regression-based marketplace}\label{subsec:linearregression}

We consider a marketplace where two companies fit linear regression models in a multi-objective environment. 

\paragraph{Linear regression model.} To formalize each company's machine learning pipeline, we consider the multi-objective, high-dimensional linear regression model described below. This multi-objective environment aims to capture how ML models are often trained to balance multiple objectives which are in tension with each other, and we consider linear regression since it has often accurately predicted scaling trends of large-scale machine learning models (see Section \ref{subsec:discussion} for additional discussion). 

More concretely, given an input $x \in \mathbb{R}^P$, let $\langle \beta_1, x\rangle$ be the output that targets performance maximization, and let $\langle \beta_2, x\rangle$ be the output that targets safety maximization. Given a linear predictor $\beta$, the performance loss is evaluated via a population loss, $\LossPerf(\param) = \mathbb{E}_{x \sim \DF}[(\langle \paramone, x\rangle - \langle \beta, x\rangle)^2]$, and the safety violation is captured by a loss $\LossAlign(\param) = \mathbb{E}_{x \sim \DF}[(\langle \paramtwo, x\rangle - \langle \beta, x\rangle)^2]$, where $\DF$ is the input distribution. 

The model provider implicitly determines how to balance $\beta_1$ and $\beta_2$ when determining how to label their training dataset. In particular, each model provider is given an unlabelled training dataset $X \in \mathbb{R}^{N \times P}$ with $N$ inputs drawn from $\DF$. To generate labels, they select the fraction $\alpha \in [0,1]$ of training data to label according to each objective. They then sample a fraction $\mix$ of the training data uniformly from $X$ and label it as $Y_i = \langle \beta_1, X_i\rangle$; the remaining $1 - \mix$ fraction is labelled as $Y_i = \langle \beta_2, X_i\rangle$. The model provider fits a ridge regression on the labelled training dataset with least-squares loss $\ell(y, y') = (y-y')^2$, and thus solves: $\estparam(\mix, \reg, X) = \argmin_{\beta} \left(\frac{1}{N}  \sum_{i=1}^N (Y_i - \langle \beta, X_i \rangle)^2 + \reg ||\beta||^2_2 \right)$. 

\paragraph{Marketplace.}
The marketplace contains two companies, an \textit{incumbent company} $I$ already in the market and a \textit{new (entrant) company} $E$ trying to enter the market. At a high level, each company $C \in \left\{I, E\right\}$ 
faces reputational damage if their safety violation exceeds their safety constraint $\tau_C$. Each company company $C$ is given $N_C$ unlabelled data points sampled from $\DF$, and selects a mixture parameter $\mix_C$ and regularizer $\reg_C$ to maximize their performance given their safety constraint $\tau_C$. We assume that the incumbent company $I$ faces a stricter safety constraint, $\constrlead < \constrentr$, due to increased public or regulatory scrutiny (see Section \ref{subsec:discussion} for additional discussion). 

When formalizing how the model providers choose hyperparameters, we make the following simplications. First, rather than work directly with the performance and safety losses of the ridge regression estimator, we assume for analytic tractability that they approximate these losses by $\LossPerf^* := \LossPerf^*(\beta_1, \beta_2, \DF, \reg, N, \alpha)$ and $\LossAlign^* := \LossAlign^*(\beta_1, \beta_2, \DF, \alpha)$ defined as follows. 
\begin{itemize}[leftmargin=*, nosep]
    \item \textit{Performance:} We define $\LossPerf^*$ to be a \textit{deterministic equivalent} $\LossPerf^{\texttt{det}}(\beta_1, \beta_2, \Sigma, \reg, N, \alpha)$ which we derive in Lemma \ref{lemma:sollichvariant}. The deterministic equivalent ~\citep[cf.][]{HLN07} is a tool from random matrix theory that is closely linked to the Marčenko-Pastur law \citep{marchenko}. Under standard random matrix assumptions (Assumption \ref{assumption:MP}), the deterministic equivalent asymptotically approximates the loss $L_1(\hat{\beta}(\alpha, \lambda, X))$ when $X$ is constructed from $N$ i.i.d.\ samples from $\DF$ (see Appendix \ref{appendix:machinery} for additional discussion). 
    \item \textit{Safety:} For analytic simplicity, in the main body of the paper, we define $\LossAlign^*$ to be the safety violation of the infinite-data ridgeless regression estimator with mixture parameter $\alpha$.\footnote{The infinite-data ridgeless regression estimator is 
    $\argmin_{\beta} \left(\alpha \cdot \mathbb{E}_{x \sim \DF}[\langle \beta - \beta_1, x\rangle^2] + (1-\alpha)\cdot  \mathbb{E}_{x \sim \DF}[\langle \beta - \beta_2, x\rangle^2] \right)$. For this specification, the dataset size $N$ and the regularization parameter $\reg$ only affect $\LossPerf^*$ and not $\LossAlign^*$, which simplifies our analysis in Sections \ref{sec:warmup}-\ref{sec:general} and enables us to obtain tight characterizations.} In Appendix \ref{appendix:extension}, we instead define $\LossAlign^*$ analogously to $\LossPerf^*$---i.e., as a deterministic equivalent $L_2^{\text{det}}(\beta_1, \beta_2, \DF, \lambda, N, \alpha)$---and extend our model and results to this more complex setting.\footnote{We directly extend our results in Section \ref{sec:warmup}, and we also show relaxed versions of our results in Section \ref{sec:general}.}
\end{itemize}
Second, we assume that $(\beta_1, \beta_2) \sim \DC$ for some joint distribution $\DC$ and that the model providers take expectations when choosing hyperparameters, since it will be easier to specify assumptions in Section \ref{subsec:alignment} over distributions of predictors. 

Within this setup, a company $C$ faces reputational damage if the safety violation exceeds a certain threshold: 
\[\mathbb{E}_{(\beta_1, \beta_2) \sim \DC} [L_2^*(\beta_1, \beta_2, \DF, \alpha_C)] > \tau_C.\]
We assume that the safety thresholds for the two companies satisfy the following inequalities: 
\begin{equation}
\label{eq:safetythreshold}
 \constrentr >_{(A)} \constrlead \ge_{(B)} \mathbb{E}_{(\beta_1, \beta_2) \sim \DC}[\LossAlign^*(\beta_1, \beta_2, \DF, 0.5)].
\end{equation}
Here, inequality (A) captures the notion that the incumbent needs to achieve higher safety to avoid reputational damage. Inequality (B) guarantees that both companies, $C \in \left\{I, E \right\}$, can set the mixture parameter $\mix_C \ge 0.5$ without facing reputational damage, and thus ensures that the safety constraint does not dominate the company's optimization task.\footnote{More specifically, inequality (B) ensures that the safety constraint still allows both companies to label 50\% of their training data according to the performance-optimal outputs.}

The company selects $\mix \in [0.5,1]$ and $\reg \in (0,1)$ to maximize their performance subject to their safety constraint, as formalized by the following optimization program:\footnote{Technically, the optimum might be achieved at $\reg = 0$ or $\reg = 1$, and the $\min$ should be replaced by a $\inf$.} 
\begin{equation*}
(\mix_C, \reg_C) = \argmin_{\mix \in [0.5, 1], \reg \in (0,1)} \mathbb{E}_{\DC}[\LossPerf^*(\beta_1, \beta_2, \DF, \reg, N_C, \alpha)] \text{ s.t. } \mathbb{E}_{\DC}[\LossAlign^*(\beta_1, \beta_2, \DF, \alpha)] \le \tau_C.
\end{equation*}

\paragraph{Market-entry threshold.} We define the market-entry threshold to capture the minimum number of data points $\Nentr$ that the new company needs to collect to achieve better performance than the incumbent company while avoiding reputational damage. 
\begin{definition}
\label{def:marketentrythreshold}
The \textit{market-entry threshold} $\Nentropt(\Nlead, \constrlead, \constrentr, \DC, \DF)$
is the \textit{minimum value} of $\Nentr \in \mathbb{Z}_{\ge 1}$ such that $\mathbb{E}_{\DC}[\LossPerf^*(\beta_1, \beta_2, \DF, \regentr, \Nentr, \mixentr)] \le \mathbb{E}_{\DC}[\LossPerf^*(\beta_1, \beta_2, \DF, \reglead, \Nlead, \mixlead)]$.  
\end{definition}

The goal of our work is to analyze the function $\Nentropt(\Nlead,\constrlead, \constrentr, \DC,  \DF)$.

\subsection{Model discussion}\label{subsec:discussion}

Now that we have formalized our statistical model, we discuss and justify our design choices in greater detail. We defer a discussion of limitations to Section \ref{sec:discussion}.

\paragraph{Presence of competing objectives.} Our multi-objective formulation is motivated by how  ML models are often trained to balance multiple objectives which are in tension with each other. In some cases, the pretraining objective is in tension with the finetuning objective \citep{WHS23}. For example, the fine-tuning of a language model to be more aligned with user intent can degrade performance---e.g., because the model hedges too much---which creates an ``alignment tax'' \citep{OJAWMZASR22}. In other cases,  fine-tuning approaches themselves balance multiple objectives such as helpfulness (which can be mapped to performance in our model) and harmlessness (which can be mapped to safety in our model) \citep{bairlhf2022}. These objectives can be in tension with one another, for example if the user asks for dangerous information. 

\paragraph{High-dimensional linear regression as a statistical model.} We focus on high-dimensional linear regression due to its ability to capture scaling trends observed in large-scale machine learning models such as language models, while still retaining analytic tractability. In particular, in single-objective environments, scaling trends for high-dimensional linear regression recover the empirically observed power-law scaling of the loss with respect to the dataset size \citep{K20, CLKZ21, WHS22}. Moreover, from an analytic perspective, the structural properties of high-dimensional linear regression make it possible to characterize the loss using random matrix machinery (see Appendix \ref{appendix:machinery}).

\paragraph{Impact of market position on model provider constraint $\tau$.} Our assumption that $\constrentr > \constrlead$ (inequality (A) in \eqref{eq:safetythreshold}) is motivated by how large companies face greater reputational damage from safety violations than smaller companies. One driver of this unevenness in reputational damage is \textit{regulation}: for example, recent regulation and policy \citep{whitehouse_2023_aG_Exec_order, CaliforniaLegislation2024} places stricter requirements on companies that use significant amounts of compute during training. In particular, these companies face more stringent compliance requirements in terms of safety assessments and post-deployment monitoring. Another driver of uneven reputational damage is \textit{public perception}: we expect that the public is more likely to uncover safety violations for large companies, due to the large volume of user queries to the model. In contrast, for small companies, safety violations may be undetected or subject to less public scrutiny.

\subsection{Assumptions on linear regression problem}\label{subsec:assumptions}

To simplify our characterization of scaling trends, we follow prior work on high-dimensional linear regression~\citep[see, e.g.,][]{CLKZ21, WHS22} and make the following empirically motivated power-law assumptions. Let $\Sigma = \mathbb{E}_{x \sim \DF}[xx^T]$ be the covariance matrix, and let $\lambda_i$ and $v_i$ be the eigenvalues and eigenvectors, respectively. We require the eigenvalues to decay with scaling exponent $\gamma > 0$ according to $\lambda_i = i^{-1-\gamma}$ for $1 \le i \le P$. 
For the alignment coefficients $\langle \beta_j, v_i\rangle$, it is cleaner to enforce power scaling assumptions in expectation, so that we can more easily define a correlation parameter. We require that for some $\delta > 0$, the alignment coefficients satisfy $\mathbb{E}_{\DC}[\langle \beta_j, v_i \rangle^2] = i^{-\delta}$, where $v_i$ is the $i$th eigenvector of $\Sigma$, for $j \in \left\{1,2\right\}$ and $1 \le i \le P$. We also introduce a similar condition on the joint alignment coefficients, requiring that for some $\rho \in [0,1)$, it holds that $\mathbb{E}_{\DC}[\langle \beta_1, v_i \rangle \langle \beta_2, v_i \rangle ] = \rho \cdot i^{-\delta}$. Finally, we assume an overparameterized limit where the number of parameters $P \rightarrow \infty$ approaches infinity. Below, we provide an example which satisfies these assumptions. 
\begin{example}
\label{example:generation}
Suppose that the covariance $\Sigma$ is a  diagonal matrix with diagonal given by $\lambda_i = i^{-1-\gamma}$. Let the joint distribution over $\beta_1$ and $\beta_2$ be a multivariate Gaussian such that:
\[
\mathbb{E}_{\DC}[ (\beta_{j_1})_{i_1} (\beta_{j_2})_{i_2}] = 
\begin{cases}
 0 & \text{ if }  1 \le j_1, j_2 \le 2, 1 \le i_1 \neq i_2 \le P \\
  i_1 ^{-\delta} & \text{ if } 1 \le j_1 = j_2 \le 2, 1 \le i_1 = i_2 \le P \\
\rho \cdot i_1 ^{-\delta} &\text{ if } 1 \le j_1 \neq j_2 \le 2, 1 \le i_1 = i_2 \le P.
\end{cases}
\]
This implies that $\mathbb{E}_{\DC}[\langle \beta_j, v_i \rangle^2] = i^{-\delta}$ and $\mathbb{E}_{\DC}[\langle \beta_1, v_i \rangle \langle \beta_2, v_i \rangle ] = \rho \cdot i^{-\delta}$. 
\end{example}

We adopt the random matrix theory assumptions on the covariance matrix and linear predictors from \citet{bach} (see Assumption \ref{assumption:MP} in Appendix \ref{appendix:machinery}), which guarantee that the Marčenko-Pastur law holds \citep{marchenko}. That is, the covariance $(\hat{\Sigma} + \lambda I)^{-1}$ of the samples can be approximated by a deterministic quantity (see Appendix \ref{appendix:MP} for a more detailed discussion). We leverage this Marčenko-Pastur law to derive a deterministic equivalent $L_1^{\texttt{det}}$ for the performance loss $L_1(\hat{\beta}(\alpha, \lambda, X))$ of the ridge regression estimator (Lemma \ref{lemma:sollichvariant}).

\begin{figure}[t!]
    \centering
    \begin{subfigure}[b]{0.48\textwidth}
        \centering
        \includegraphics[width=\textwidth]{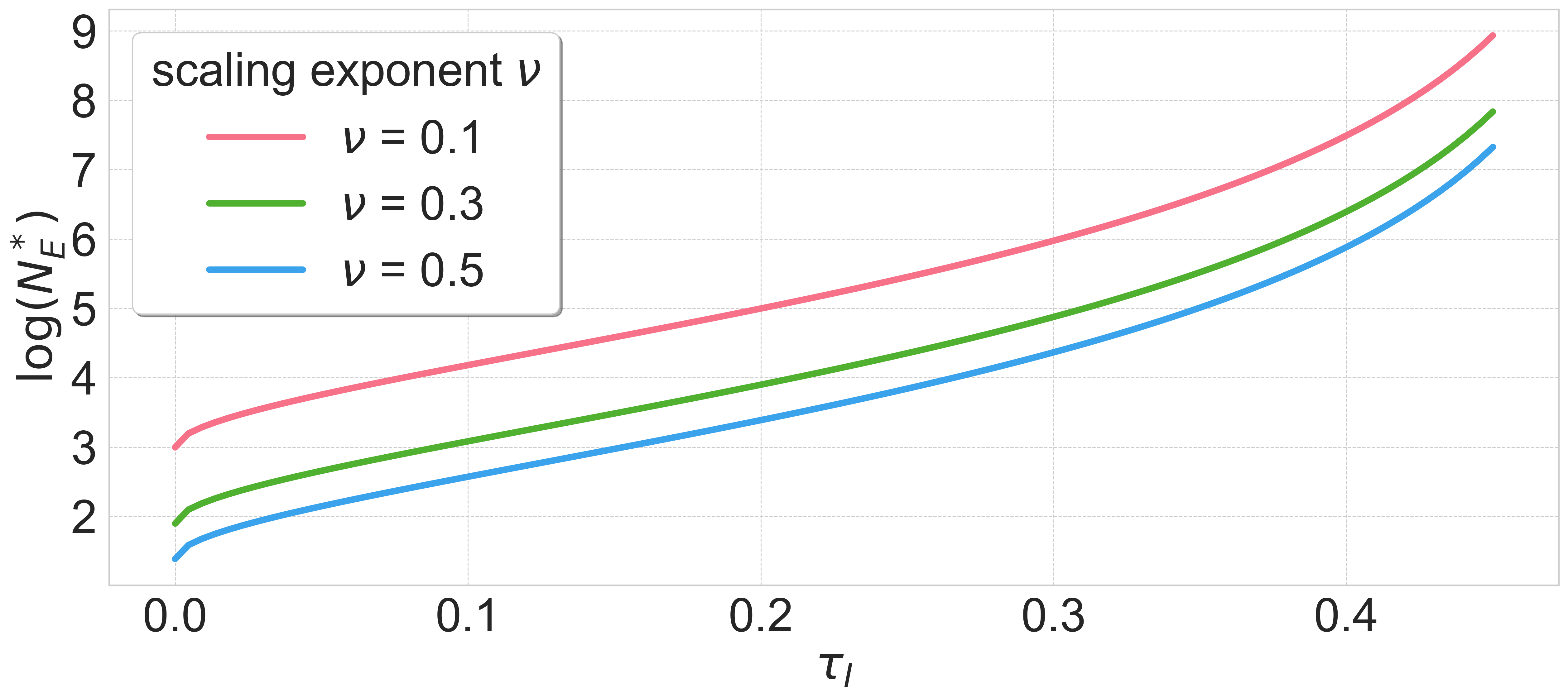}
    \end{subfigure}
    \hfill
    \begin{subfigure}[b]{0.48\textwidth}
        \centering
        \includegraphics[width=\textwidth]{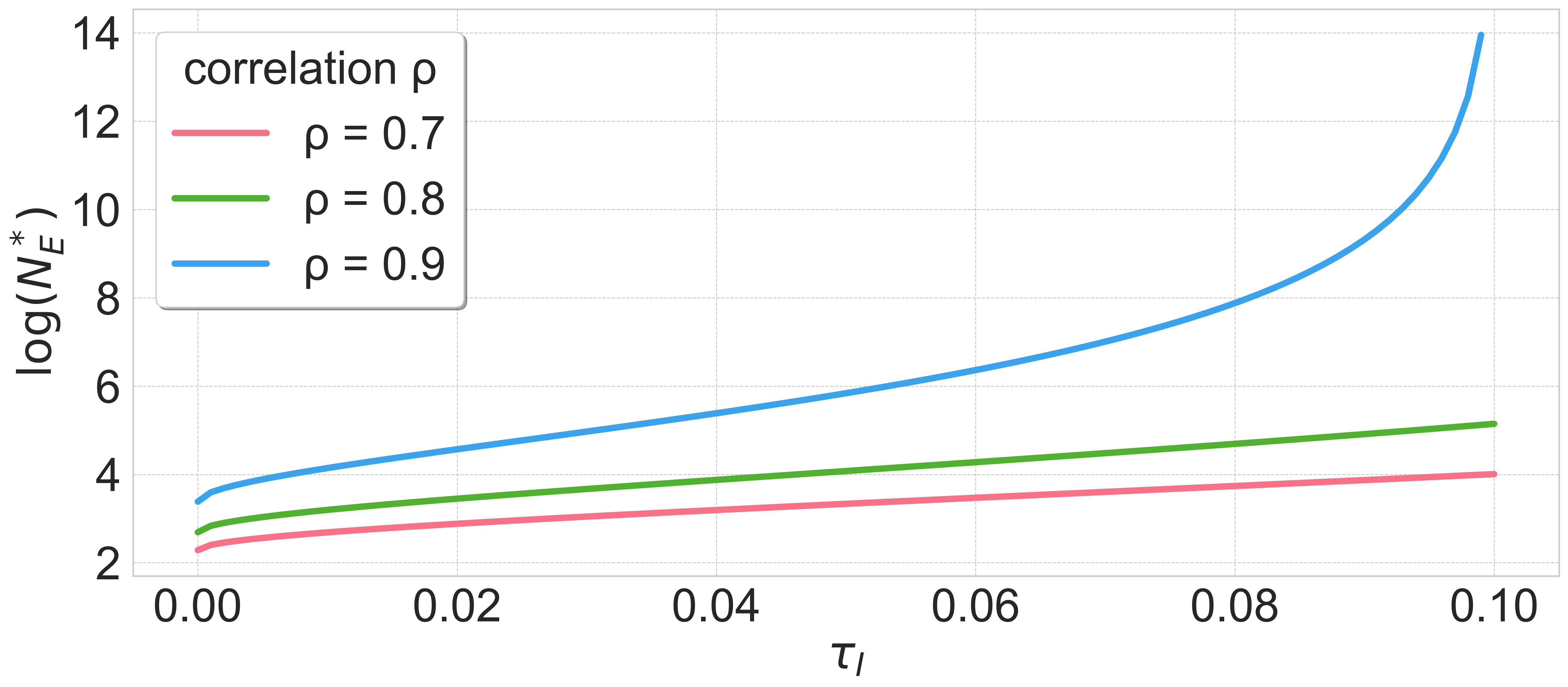}
    \end{subfigure}
    \caption{Market-entry threshold $\Nentropt$ as a function of the incumbent's safety constraint $\constrlead$, when the incumbent has infinite data and entrant has no safety constraint (Theorem \ref{thm:tradeoffwarmup}). The  plots show varying values of the scaling exponent $\scalingexp$ where the correlation parameter $\rho = 0.5$ is held fixed (left) and varying values of  $\rho$ where $\scalingexp = 0.34$ is held fixed (right). The market-entry threshold  $\Nentropt$ is finite. It is also higher when the constraint $\constrlead$ is weaker, when the correlation $\rho$ is stronger, and when the scaling exponent $\scalingexp$ is lower.  }
    \label{fig:warmup}
\end{figure}

\section{Warm Up: Infinite-Data Incumbent and Unconstrained Entrant}\label{sec:warmup}

As a warmup, we analyze the market entry $\Nentropt$ threshold in a simplified environment where the incumbent has infinite data and the new company faces no safety constraint. In this result, we place standard power-law scaling assumptions on the covariance and alignment coefficients (Section \ref{subsec:assumptions}) and we characterize the threshold $\Nentropt$ up to constants (Theorem \ref{thm:tradeoffwarmup}; Figure \ref{fig:warmup}). 
\begin{theorem}
\label{thm:tradeoffwarmup}
Suppose that power-law scaling holds for the eigenvalues and alignment coefficients, with scaling exponents $\gamma, \delta > 0$ and correlation coefficient $\rho \in [0,1)$, and suppose that $P = \infty$. 
Suppose that the incumbent company has infinite data (i.e., $\Nlead= \infty$), and that the entrant faces no constraint on their safety (i.e., $\constrentr = \infty$). Suppose that the safety constraint $\constrlead$ satisfies \eqref{eq:safetythreshold}. Then, it holds that:\footnote{Throughout the paper, we allow $\Theta()$ and $O()$ to hide implicit constant which depends on the scaling exponents $\gamma, \delta$.}
\[\Nentropt(\infty, \constrlead, \infty, \DC, \DF) = \Theta\left(\left(\sqrt{L^*(\rho)} - \sqrt{\min(\constrlead, L^*(\rho))} \right)^{-2/\scalingexp} \right),\]
where $L^*(\rho) = \mathbb{E}_{\DC}[(\beta_1 - \beta_2)^T \Sigma (\beta_1 - \beta_2)] = \Theta(1 - \rho)$, and where $\scalingexp := \min(2(1+\gamma), \delta + \gamma)$. 
\end{theorem}

The intuition is as follows. The safety constraint $\constrlead$ forces the incumbent company to partially align their predictor with the safety objective $\beta_2$. Since $\beta_1$ and $\beta_2$ point in different directions, this reduces the performance of the incumbent along $\beta_1$ as a side effect, resulting in strictly positive loss with respect to performance. On the other hand, since the new company faces no safety constraint, the new company can optimize entirely for performance along $\beta_1$. This means that the new company can enter the market as long as their finite data error is bounded by the incumbent's performance loss. We formalize this intuition in the following proof sketch.

\begin{proof}[Proof sketch of Theorem \ref{thm:tradeoffwarmup}]
The incumbent chooses the \textit{infinite-data} ridgeless estimator $\beta(\alpha, 0)$ with mixture parameter $\alpha \in [0,1]$ tuned so the safety violation is $\constrlead$ (Lemma \ref{lemma:popparetofrontier}). The resulting performance loss is $\sqrt{L^*(\rho)} - \sqrt{\min(\constrlead, L^*(\rho))}$. Since the new company has no safety constraint, they choose the \textit{single-objective} ridge regression estimator where $\alpha = 1$ and where $\reg$ is chosen optimally.\footnote{We formally rule out the possibility that $\alpha \neq 1$ using our multi-objective scaling law in Theorem \ref{thm:scalinglawoptreginformal}.} Theorem \ref{thm:scalinglawoptreginformal} (or alternatively, 
existing analyses of high-dimensional linear regression~\citep[e.g.,][]{CLKZ21, WHS22}) demonstrate the loss follows a scaling law of the form $\inf_{\reg > 0} L_1(\hat{\beta}(1, \reg, X)) = \Theta\left(
N^{-\scalingexp} \right)$ where $\scalingexp :=\min(2(1+\gamma), \delta +\gamma)$. The full proof is in Appendix \ref{appendix:proofssec3}. 
\end{proof}

Theorem \ref{thm:tradeoffwarmup} reveals that the market-entry threshold is \textit{finite} as long as (1) the safety constraint $\constrlead$ places nontrivial restrictions on the incumbent company and (2) the safety and performance objectives are not perfectly correlated. This result captures the notion that the new company can enter the market even after the incumbent company has accumulated an infinite amount of data.

Theorem \ref{thm:tradeoffwarmup} further illustrates how the market-entry threshold changes with other parameters (Figure \ref{fig:warmup}). When safety and performance objectives are more correlated (i.e., when $\rho$ is higher), the market-entry threshold increases, which increases barriers to entry. When the safety constraint for the incumbent is weaker (i.e., when $\constrlead$ is higher), the market-entry threshold also increases. Finally, when the power scaling parameters of the covariance and alignment coefficients increase, which increases the scaling law exponent $\scalingexp$, the market-entry threshold decreases.

 \section{Generalized Analysis of the Market-entry Threshold}\label{sec:general}

To obtain a more fine-grained characterization of the market-entry threshold, we now consider more general environments. Our key technical tool is \textit{multi-objective scaling laws}, which capture the performance of ridge regression in high-dimensional, multi-objective environments with finite data  (Section \ref{subsec:technicaltool}). Using these scaling laws, we characterize the market-entry threshold when the incumbent has finite data (Section \ref{subsec:finitedata}) and when the new company has a safety constraint (Section \ref{subsec:alignment}).

Our results in this section uncover the following conceptual insights about market entry. First, our main finding from Section \ref{sec:warmup}---that the new company can enter the market with significantly less data than the incumbent---applies to these generalized environments.  Moreover, our characterizations of $\Nentropt$ exhibit a \textit{power-law-like dependence} with respect to the incumbent's dataset size (Theorem \ref{thm:finitedata}) and the difference in safety requirement for the two companies (Theorem \ref{thm:alignment}). Interestingly, the scaling exponent $c$ is not a constant across the full regime and instead takes on up to three different values. As a consequence, the new company can afford to scale up their dataset at a slower rate as the incumbent's dataset size increases, but needs to scale up their dataset at a faster rate as the two safety constraints become closer together. Proofs are deferred to Appendix \ref{appendix:proofssecgeneral}.

\subsection{Technical tool: Scaling laws in multi-objective environments}\label{subsec:technicaltool}

In this section, we give an overview of multi-objective scaling laws (see Section \ref{sec:scalinglaws} for a more formal treatment and derivations). Our scaling laws capture how the ridge regression loss $L_1(\hat{\beta}(\alpha, \lambda, X))$ along the primary objective $\beta_1$ scales with the dataset size $N$, when the regularizer $\reg$ is optimally tuned to both $N$ and problem-specified parameters. We show scaling laws for both the loss $\inf_{\reg \in (0, 1)} \mathbb{E}[L_1(\hat{\beta}(\alpha, \lambda, X))]$ and the excess loss  $\inf_{\reg \in (0, 1)} (\mathbb{E}[L_1(\hat{\beta}(\alpha, \lambda, X)) - L_1(\beta(\alpha, 0))])$ where $\beta(\alpha, 0)$ is the infinite-data ridgeless regression estimator.

\begin{figure}[t!]
    \centering
    \begin{subfigure}[b]{0.48\textwidth}
        \centering
        \includegraphics[width=\textwidth]{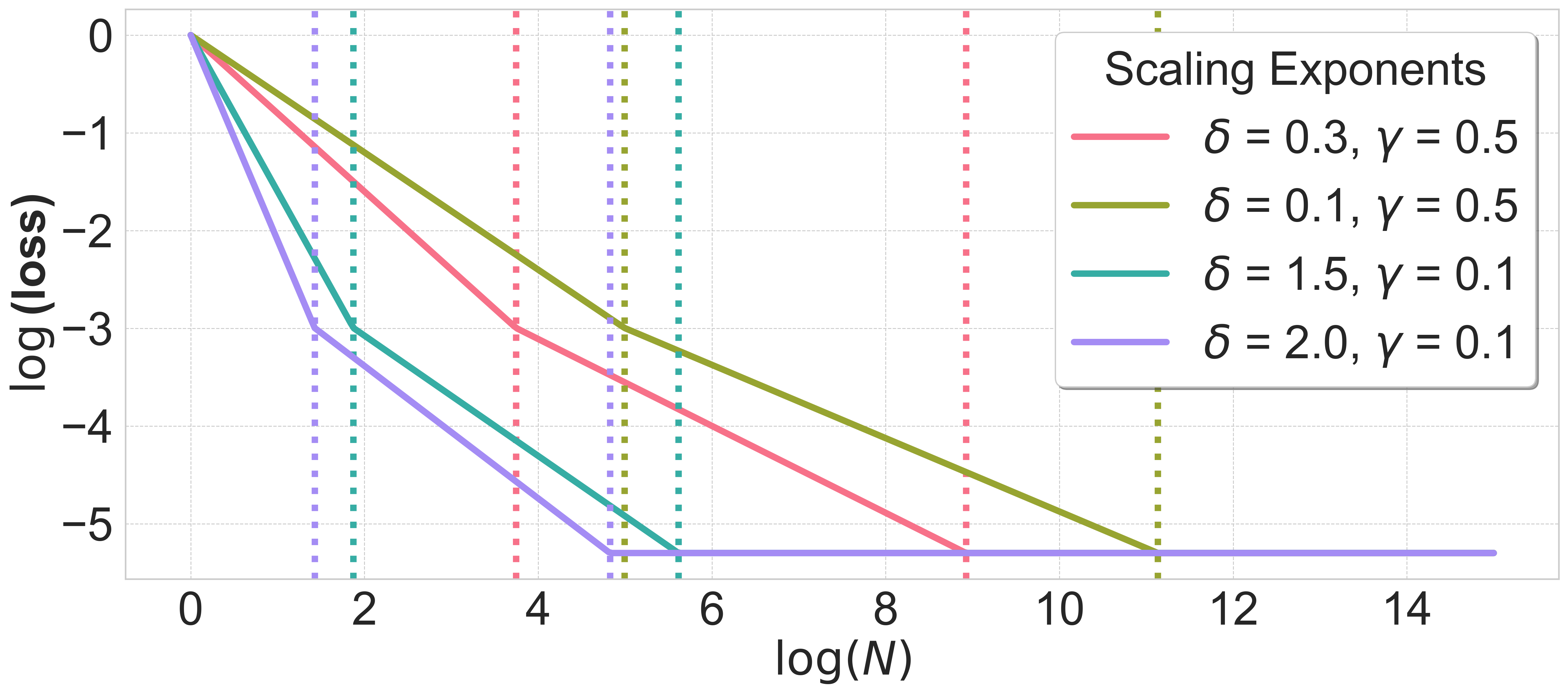}
        \caption{Scaling law for loss (up to constants)}
        \label{fig:loss}
    \end{subfigure}
    \hfill
    \begin{subfigure}[b]{0.48\textwidth}
        \centering
        \includegraphics[width=\textwidth]{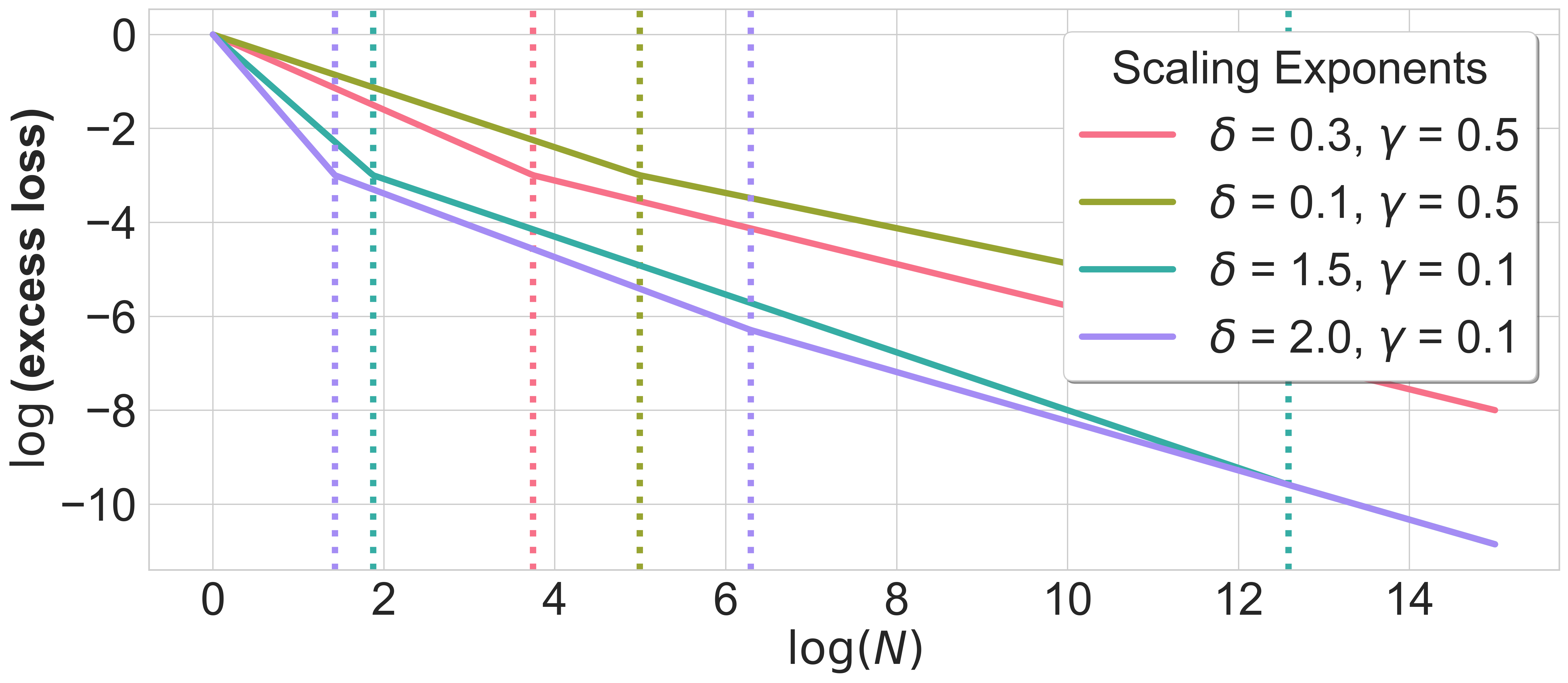}
       \caption{Scaling law for excess loss (up to constants)}
        \label{fig:excessloss}
    \end{subfigure}
    \caption{Data scaling laws for multi-objective environments where a fraction $\alpha = 0.9$ of the data is labelled according to the primary objective and a fraction $1-\alpha = 0.1$ is labelled according to the secondary objective. The plots show, up to constants, the loss $\Theta(\inf_{\reg \in (0, 1)} \mathbb{E}[L_1(\hat{\beta}(\alpha, \lambda, X))])$  (left, Theorem \ref{thm:scalinglawoptreginformal}) and excess loss $\Theta(\inf_{\reg \in (0, 1)} (\mathbb{E}[L_1(\hat{\beta}(\alpha, \lambda, X)) - L_1(\beta(\alpha, 0))]))$ (right, Theorem \ref{thm:scalinglawoptregexcessinformal}) as a function of the total number of training data points $N$. The loss and excess loss both take the form $N^{-c}$, but where the scaling exponent $c$ takes on \textit{multiple (two or three) different values} depending on the size of $N$ relative to other parameters. The scaling exponent is smaller when $N$ is larger, thus demonstrating that the scaling rate becomes slower as the dataset size $N$ increases.}
    \label{fig:scalinglaw}
\end{figure}

\paragraph{Scaling law for the loss.}
We first describe the scaling law for $\inf_{\reg \in (0, 1)} \mathbb{E}[L_1(\hat{\beta}(\alpha, \lambda, X))]$ (Theorem \ref{thm:scalinglawoptreginformal}; Figure \ref{fig:loss}). 
\begin{theorem}[Informal Version of Corollary \ref{cor:scalinglawoptreg}]
\label{thm:scalinglawoptreginformal}
Suppose that the power-law scaling assumptions from Section \ref{subsec:assumptions} hold with exponents $\gamma, \delta > 0$ and correlation coefficient $\rho \in [0,1)$. Suppose also that $P = \infty$ and $\alpha \ge 0.5$. Then, a deterministic equivalent for the expected loss  under optimal regularization  $\inf_{\reg \in (0, 1)} \mathbb{E}[L_1(\hat{\beta}(\alpha, \lambda, X))]$ scales according to $N^{-\scalingexp^*}$,  where the scaling exponent $\scalingexp^*$ is defined to be:   
\[ \scalingexp^* =
\begin{cases}
\scalingexp &\text{ if } N \le (1-\alpha)^{-\frac{1}{\scalingexp}}(1-\rho)^{-\frac{1}{\scalingexp}} \\
 \frac{\scalingexp}{\scalingexp + 1} &\text{ if } (1-\alpha)^{-\frac{1}{\scalingexp}}(1-\rho)^{-\frac{1}{\scalingexp}} 
 \le N \le  (1-\alpha)^{-\frac{2+\scalingexp}{\scalingexp}} (1-\rho)^{-\frac{1}{\scalingexp}}
\\
0 &\text{ if } N \ge (1-\alpha)^{-\frac{2+\scalingexp}{\scalingexp}} (1-\rho)^{-\frac{1}{\scalingexp}},
\end{cases}
\]
for $\scalingexp := \min(2(1+\gamma), \delta+\gamma)$. 
\end{theorem}

Theorem \ref{thm:scalinglawoptreginformal} (Figure \ref{fig:loss}) illustrates that the scaling rate becomes slower as the dataset size $N$ increases. In particular, while the scaling exponent in single-objective environments is captured by a \textit{single} value, Theorem \ref{thm:scalinglawoptreginformal} illustrates that the scaling exponent  $\scalingexp^*$ in multi-objective environments takes on \textit{three different values}, depending on the size of $N$ relative to other parameters. 
When $N$ is small (the first regime), the scaling exponent $\scalingexp^* = \scalingexp$ is identical to that of the single-objective environment given by $\beta_1$. When $N$ is a bit larger (the second regime), the scaling exponent \textit{reduces} to $\scalingexp^* = \scalingexp/(\scalingexp + 1) < \scalingexp$. To make this concrete, if we take $\scalingexp = 0.34$ to be an empirically estimated scaling law exponent for language models \citep{H22}, this would mean that $\scalingexp^* \approx 0.34$ in the first regime and $\scalingexp^* \approx 0.25$ in the second regime. Finally, when $N$ is sufficiently large (the third regime), the scaling exponent reduces all the way to $\scalingexp^* = 0$ and the only benefit of additional data is to improve constants on the loss. 

\paragraph{Scaling law for the excess loss.}
We next turn to the excess loss, $\inf_{\reg \in (0, 1)} (\mathbb{E}[L_1(\hat{\beta}(\alpha, \lambda, X)) - L_1(\beta(\alpha, 0))])$, which is normalized by the loss of the infinite-data ridgeless predictor $\beta(\alpha, 0)$. We show that the excess loss exhibits the same scaling behavior as the loss when $N$ is sufficiently small, but exhibits different behavior when $N$ is sufficiently large (Theorem \ref{thm:scalinglawoptregexcessinformal}; Figure \ref{fig:excessloss}). 
\begin{theorem}[Informal Version of Corollary \ref{cor:scalinglawoptregexcess}]
\label{thm:scalinglawoptregexcessinformal}
Suppose that the power-law scaling assumptions from Section \ref{subsec:assumptions} hold with exponents $\gamma, \delta > 0$ and correlation coefficient $\rho \in [0, 1)$. Suppose also that $P = \infty$ and $\alpha \ge 0.75$. Then, a deterministic equivalent for the expected loss  under optimal regularization  $\inf_{\reg \in (0, 1)} (\mathbb{E}[L_1(\hat{\beta}(\alpha, \lambda, X)) - L_1(\beta(\alpha, 0))])$ scales according to $N^{-\scalingexp^*}$,  where the scaling exponent $\scalingexp^*$ is defined to be:   
\[ \scalingexp^* =
\begin{cases}
\scalingexp &\text{ if } N \le (1-\alpha)^{-\frac{1}{\scalingexp}}(1-\rho)^{-\frac{1}{\scalingexp}} \\
 \frac{\scalingexp}{\scalingexp + 1} &\text{ if } (1-\alpha)^{-\frac{1}{\scalingexp}}(1-\rho)^{-\frac{1}{\scalingexp}} \le N \le  (1-\alpha)^{-\frac{\scalingexp'+1}{\scalingexp - \scalingexp'}}(1-\rho)^{-\frac{\scalingexp'+1}{\scalingexp - \scalingexp'}} 
\\
\frac{\scalingexp'}{\scalingexp' + 1} &\text{ if } N \ge  (1-\alpha)^{-\frac{\scalingexp'+1}{\scalingexp - \scalingexp'}}(1-\rho)^{-\frac{\scalingexp'+1}{\scalingexp - \scalingexp'}},
\end{cases}
\]
for $\scalingexp := \min(2(1+\gamma), \delta+\gamma)$ and $\scalingexp' := \min(1+\gamma, \delta+\gamma)$.
\end{theorem}

 Theorem \ref{thm:scalinglawoptregexcessinformal} (Figure \ref{fig:excessloss}) again shows that the scaling rate can become slower as the dataset size $N$ increases, and again reveals three regimes of scaling behavior. While the first two regimes of Theorem \ref{thm:scalinglawoptregexcessinformal} resemble the first two regimes of Theorem \ref{thm:scalinglawoptreginformal}, the third regime of Theorem \ref{thm:scalinglawoptregexcessinformal} (where $N \ge  (1-\alpha)^{-\frac{\scalingexp'+1}{\scalingexp - \scalingexp'}}(1-\rho)^{-\frac{\scalingexp'+1}{\scalingexp - \scalingexp'}}$) behaves differently. In this regime, the scaling exponent for the excess loss is $\frac{\scalingexp'}{\scalingexp' + 1}$, rather than zero---this captures the fact that additional data can nontrivially improve the excess loss even in this regime, even though it only improves the loss up to constants. In terms of the magnitude of the scaling exponent $\frac{\scalingexp'}{\scalingexp' + 1}$, it is \textit{strictly smaller} than the scaling exponent $\frac{\nu}{\nu +1}$ when $\delta > 1$ and \textit{equal} to the scaling exponent  $\frac{\nu}{\nu +1}$ when $\delta \le 1$.

\subsection{Finite data for the incumbent}\label{subsec:finitedata}

We compute $\Nentropt$ when the incumbent has finite data and the new company has no safety constraint (Theorem \ref{thm:finitedata}; Figure \ref{fig:boundB}). The market-entry threshold $\Nentropt$ depends on the incumbent's dataset size $\Nlead$, the incumbent's performance loss $G_I$ if they were to have infinite data but face the same safety constraint, the scaling exponents $\gamma, \delta$, and the correlation coefficient $\rho$.  
\begin{theorem}
\label{thm:finitedata}
Suppose that the power-law scaling holds for the eigenvalues and alignment coefficients with scaling exponents $\gamma, \delta > 0$ and correlation coefficient $\rho \in [0, 1)$, and suppose that $P = \infty$. Assume that $\constrentr = \infty$. Suppose that the safety constraint $\constrlead$ satisfies \eqref{eq:safetythreshold}. Then we have that $\Nentropt = \Nentropt(\Nlead, \constrlead, \infty, \DC, \DF)$ satisfies: 
\[
\Nentropt := 
\begin{cases}
\Theta\big(\Nlead\big) &\text{ if }  \Nlead \le G_I^{-\frac{1}{2 \scalingexp}} (1-\rho)^{-\frac{1}{2\scalingexp}}  \\
\Theta\bigg(\Nlead^{\frac{1}{\scalingexp+1}} \cdot G_I^{-\frac{1}{2(\scalingexp+1)}} (1-\rho)^{-\frac{1}{2(\scalingexp+1)}}\bigg)  &\text{ if } G_I^{-\frac{1}{2 \scalingexp}} (1-\rho)^{-\frac{1}{2\scalingexp}} \le \Nlead  \le G_I^{-\frac{1}{2} - \frac{1}{\scalingexp}}(1-\rho)^{\frac{1}{2}}  \\
\Theta\bigg(G_I^{-\frac{1}{\scalingexp}}\bigg)  &\text{ if }  \Nlead \ge G_I^{-\frac{1}{2} - \frac{1}{\scalingexp}}(1-\rho)^{\frac{1}{2}},
\end{cases}
\]
where $L^*(\rho) = \mathbb{E}_{\DC}[(\beta_1 - \beta_2)^T \Sigma (\beta_1 - \beta_2)] = \Theta(1 - \rho)$, where $G_I := (\sqrt{L^*(\rho)} - \sqrt{\min(\constrlead, L^*(\rho))})^2$, and where $\scalingexp = \min(2(1+\gamma), \delta + \gamma)$.
\end{theorem}

The market-entry threshold in Theorem \ref{thm:finitedata} exhibits three regimes of behavior depending on $\Nlead$. In particular, the market-entry threshold takes the form $\Nentropt = \Theta(\Nlead^c)$ where $c$ decreases from $1$ (in the first regime) to $\frac{1}{\nu + 1}$ (in the second regime) to $0$ (in the third regime) as $\Nlead$ increases. To connect this to large-language-model marketplaces, we directly set $\scalingexp = 0.34$ to be the empirically estimated scaling law exponent for language models \citep{H22}; in this case, the scaling exponent $c$ ranges from $1$ to $0.75$ to $0$. The fact that there are three regimes come from the scaling law derived in Theorem \ref{thm:scalinglawoptreginformal}, as the following proof sketch illustrates. 

\begin{proof}[Proof sketch]
The key technical tool is the scaling law for the loss $\inf_{\reg \in (0, 1)} \mathbb{E}[L_1(\hat{\beta}(\alpha, \lambda, X))]$ (Theorem \ref{thm:scalinglawoptreginformal}), which has three regimes of scaling behavior for different values of $N$. We apply the scaling law to analyze the performance of the incumbent, who faces a safety constraint and has finite data. Analyzing the performance of the new company---who faces no safety constraint---is more straightforward, given that the new company can set $\mixentr = 1$. We compute $\Nentropt$ as the number of data points needed to match the incumbent's performance level.  The full proof is deferred to Appendix \ref{appendix:proofssubsec41}. 
\end{proof}

\begin{figure}[t!]
    \centering
    \begin{subfigure}[b]{0.48\textwidth}
        \centering
        \includegraphics[width=\textwidth]{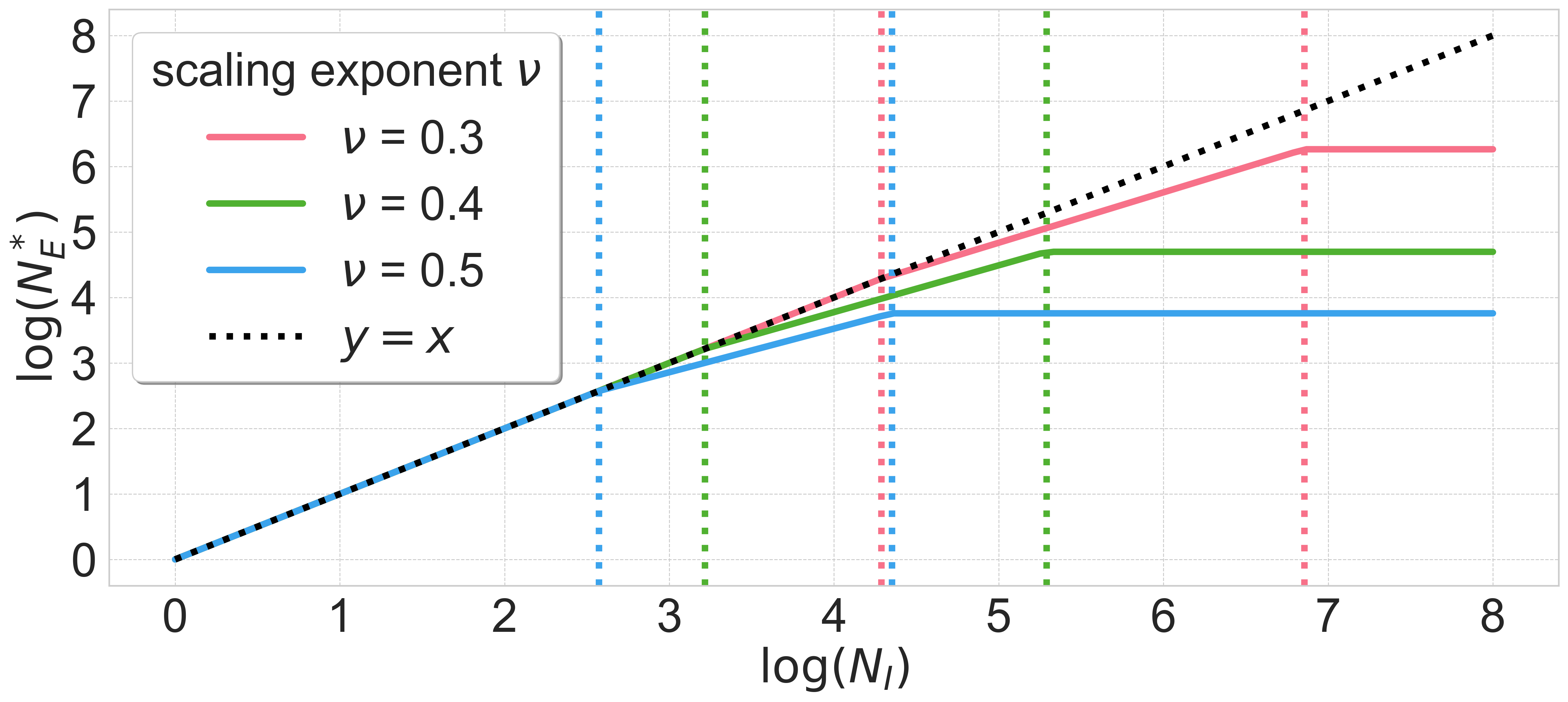}
    \end{subfigure}
    \hfill
    \begin{subfigure}[b]{0.48\textwidth}
        \centering
        \includegraphics[width=\textwidth]{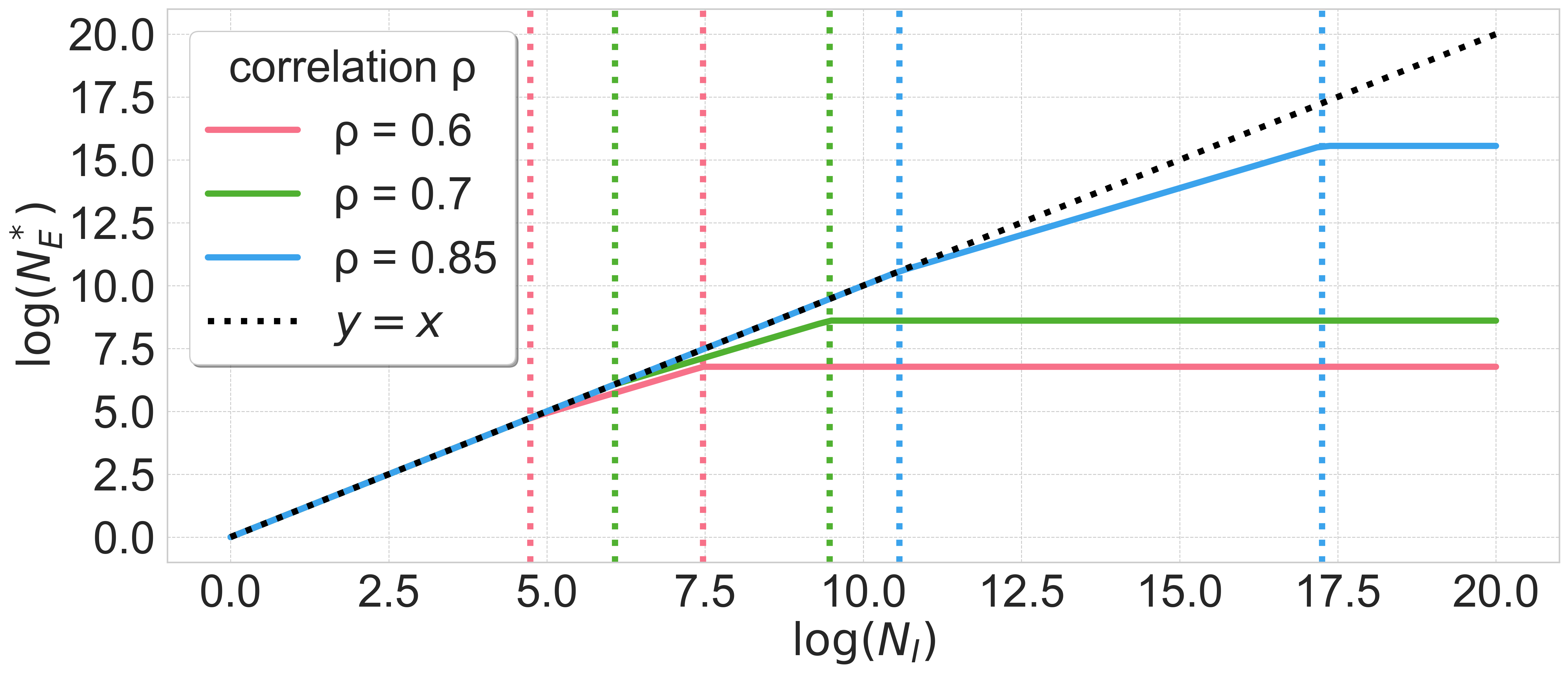}
    \end{subfigure}
    \caption{The market-entry threshold $\Nentropt$  as a function of the incumbent dataset size $\Nlead$, when the new company has no safety constraint (Theorem \ref{thm:finitedata}). The  plots show varying values of the scaling exponent $\scalingexp$ where the correlation parameter $\rho = 0.5$ is held fixed (left) and varying values of  $\rho$ where $\scalingexp = 0.34$ is held fixed (right). When $\Nlead$ is sufficiently large, the market-entry threshold $\Nentropt$ is asymptotically less than $\Nlead$ (i.e., below the dotted black line). Each curve is the union of three line segments with slope decreasing in $\Nlead$, demonstrating that the new company can afford to scale up their dataset at a slower rate as $\Nlead$ increases.}
    \label{fig:boundB}
\end{figure}

Theorem \ref{thm:finitedata} reveals that the new company can enter the market with $\Nentropt = o(\Nlead)$ data, as long as the incumbent's dataset is sufficiently large (i.e., $\Nlead \ge G_I^{-\frac{1}{2 \scalingexp}} (1-\rho)^{-\frac{1}{2\scalingexp}}$). The intuition is when there is sufficient data, the multi-objective scaling exponent is worse than the single-objective scaling exponent (Theorem \ref{thm:scalinglawoptreginformal}).
The incumbent thus faces a worse scaling exponent than the new company, so the new company can enter the market with asymptotically less data. 

The three regimes in Theorem \ref{thm:finitedata} further reveal that the market-entry threshold $\Nentropt$ scales at a slower rate as the incumbent's dataset size $\Nlead$ increases (Figure \ref{fig:boundB}). The intuition is that the multi-objective scaling exponent $\scalingexp^*$ faced by the incumbent decreases as dataset size increases, while the single-objective scaling exponent faced by the new company is constant in dataset size (Theorem \ref{thm:scalinglawoptreginformal}). The incumbent thus becomes less efficient at using additional data to improve performance, while the new company's efficiency in using additional data remains unchanged. 

Theorem \ref{thm:finitedata} also offers finer-grained insight into the market-entry threshold in each regime. In the first regime, where the incumbent's dataset is small, the threshold $\Nentropt$  matches the incumbent dataset size---the new company does not benefit from having a less stringent safety constraint.  In the second (intermediate) regime, the new company can enter with a dataset size proportional to $\Nlead^{1/(\scalingexp+1)}$. This \textit{polynomial speedup} illustrates that the new company can more efficiently use additional data to improve performance than the incumbent company.  A caveat is that this regime is somewhat restricted in that the ratio of the upper and lower boundaries is bounded. In the third regime, where the incumbent's dataset size is large, the market-entry threshold $\Nentropt$ matches the market-entry threshold from Theorem \ref{thm:tradeoffwarmup} where the incumbent has \textit{infinite} data.

\subsection{Safety constraint for the new company}\label{subsec:alignment}

We compute $\Nentropt$ when the new company has a nontrivial safety constraint and the incumbent has infinite data. For this result, we strengthen the conditions on $\constrentr$ and $\constrlead$ from \eqref{eq:safetythreshold}, instead requiring: 
\begin{equation}
\label{eq:safetythresholdnew}
 \constrentr >_{(A)} \constrlead \ge_{(B)} \mathbb{E}_{(\beta_1, \beta_2) \sim \DC}[\LossAlign^*(\beta_1, \beta_2, \DF, 0.75)], 
\end{equation}
where \eqref{eq:safetythresholdnew} replaces the $0.5$ with a $0.75$ in the right-most quantity.\footnote{Inequality (B) in \eqref{eq:safetythresholdnew} requires that the safety constraint still allows both company to label 75\% of their training data according to performance-optimal outputs. We make this modification, since our analysis of multi-objective scaling laws for the \textit{excess} loss assumes $\alpha \ge 0.75$ (see Section \ref{subsec:scalingexcess}).}

We state the result below (Theorem \ref{thm:alignment}; Figure \ref{fig:boundtau}). The market-entry threshold $\Nentropt$ depends on the incumbent's safety constraint $\constrlead$, the performance loss $G_I$ (resp. $G_E$) if the incumbent (resp. new company) had infinite data and faced the same safety constraint, the difference $D = G_I - G_E$ in infinite-data performance loss achievable by the incumbent and new company, the scaling exponents $\gamma, \delta$, and the correlation coefficient $\rho$.

\begin{figure}[t!]
    \centering
    \begin{subfigure}[b]{0.48\textwidth}
        \centering
        \includegraphics[width=\textwidth]{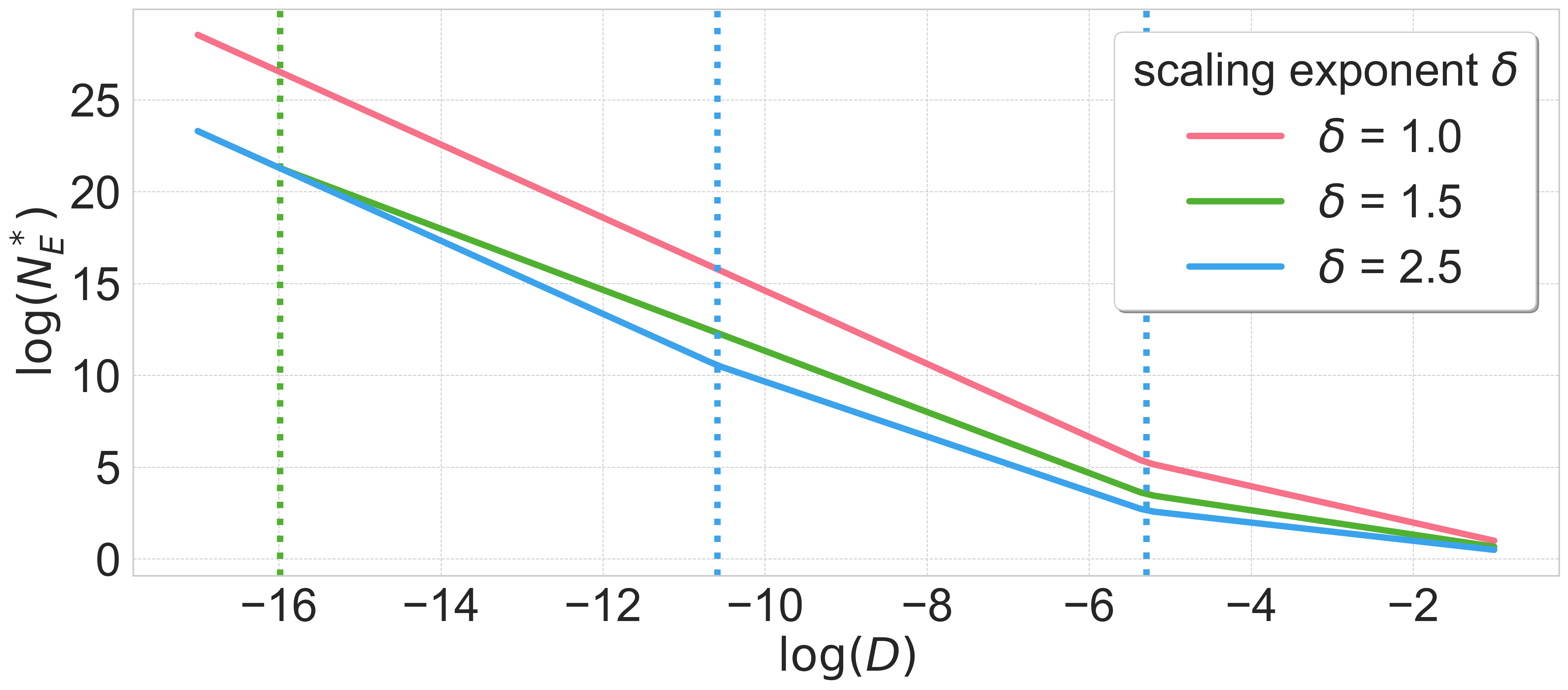}
        \label{fig:nu1}
    \end{subfigure}
    \hfill
    \begin{subfigure}[b]{0.48\textwidth}
        \centering
        \includegraphics[width=\textwidth]{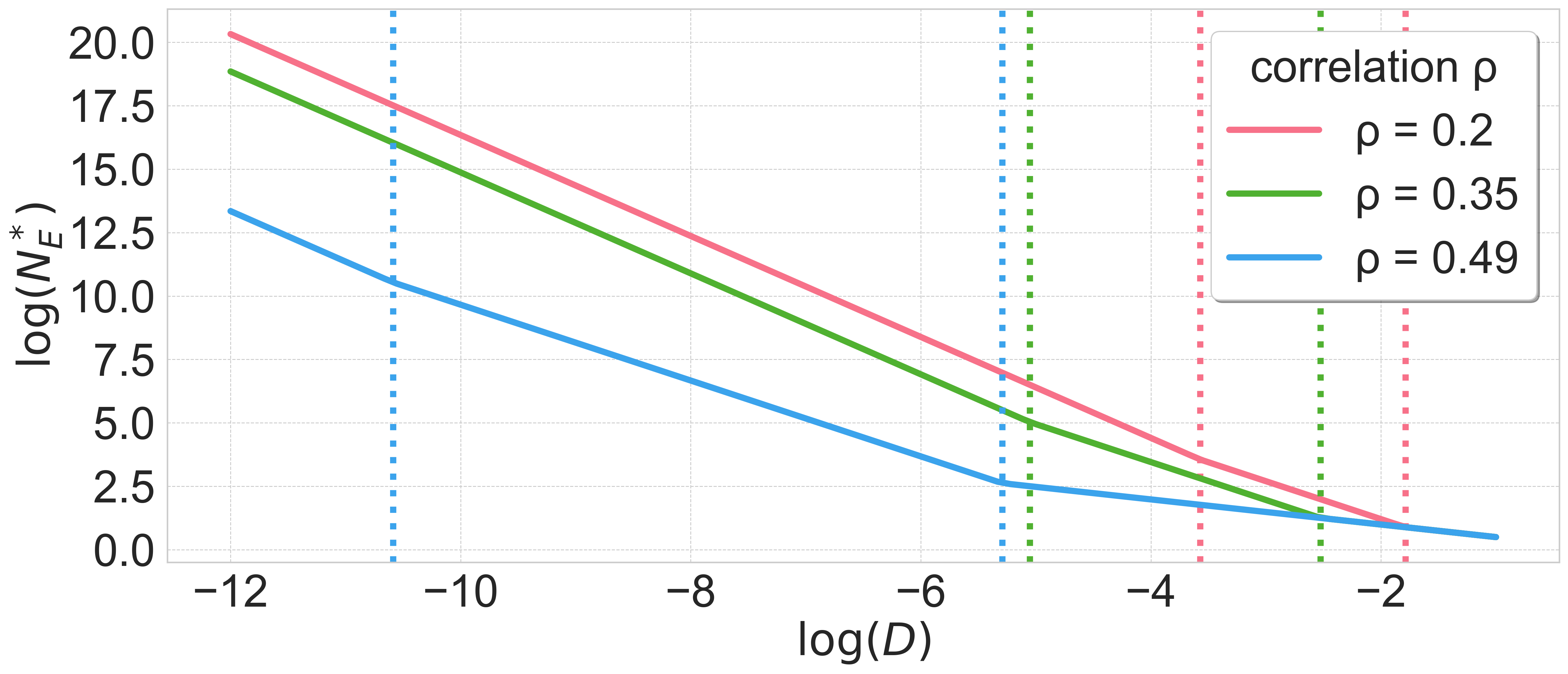}
        \label{fig:rho}
    \end{subfigure}
    \caption{The market-entry threshold $\Nentropt$  as a function of the difference $D$ between the infinite-data performance loss of the incumbent and new company, when the incumbent has infinite data (Theorem \ref{thm:alignment}). 
    The  plots show varying values of the scaling exponent $\delta$ where the correlation parameter $\rho = 0.49$ is held fixed (left) and varying values of  $\rho$ where $\delta = 2.5$ is held fixed (right). The plots are shown in log space. The market-entry threshold is finite in all cases. Each curve is the union of multiple line segments with slope increasing in magnitude as $\log D$ decreases, demonstrating that the new company needs to scale up their dataset at a faster rate as $D$ decreases.  }
    \label{fig:boundtau}
\end{figure}

\begin{theorem}
\label{thm:alignment}
Suppose that the power-law scaling holds for the eigenvalues and alignment coefficients with scaling exponents $\gamma, \delta > 0$ and correlation coefficient $\rho \in [0, 1)$, and suppose that $P = \infty$. Suppose that the safety constraints $\constrlead$ and $\constrentr$ satisfy \eqref{eq:safetythresholdnew}.
Then it holds that $\Nentropt = \Nentropt(\infty, \constrlead, \constrentr, \DC, \DF)$ satisfies: 
\[
\Nentropt := 
\begin{cases}
\Theta(D^{-\frac{1}{\scalingexp}} ) &\text{ if } D \ge G_E^{\frac{1}{2}} (1-\rho)^{\frac{1}{2}} \\
\Theta\bigg(D^{-\frac{\scalingexp+1} {\scalingexp}} G_E^{\frac{1}{2}} (1-\rho)^{\frac{1}{2}} \bigg) &\text{ if } 
G_E^{\frac{\scalingexp}{2(\scalingexp - \scalingexp')}}  (1-\rho)^{\frac{\scalingexp}{2(\scalingexp - \scalingexp')}}  \le D \le  G_E^{\frac{1}{2}} (1-\rho)^{\frac{1}{2}} 
\\
\Theta\bigg(\big(D \cdot G_E^{-\frac{1}{2}} (1-\rho)^{-\frac{1}{2}} \big)^{-\frac{\scalingexp'+1}{\scalingexp'}} \bigg) &\text{ if } D \le G_E^{\frac{\scalingexp}{2(\scalingexp - \scalingexp')}}  (1-\rho)^{\frac{\scalingexp}{2(\scalingexp - \scalingexp')}},
\end{cases}
\]
where $L^*(\rho) = \mathbb{E}_{\DC}[(\beta_1 - \beta_2)^T \Sigma (\beta_1 - \beta)] = \Theta(1 - \rho)$, where $\scalingexp = \min(2(1+\gamma), \delta + \gamma)$ and $\scalingexp' = \min(1+\gamma, \delta + \gamma)$, where 
$G_I := \left(\sqrt{L^*(\rho)} - \sqrt{\min(\constrlead, L^*(\rho))}\right)^2$ and $G_E := \left(\sqrt{L^*(\rho)} - \sqrt{\min(\constrentr, L^*(\rho))}\right)^2$, and where $D := G_I - G_E$. 
\end{theorem}

The market-entry threshold in Theorem \ref{thm:alignment} also exhibits three regimes of behavior depending on the difference $D$ in the infinite-data performance loss achievable by the incumbent and new company. In particular, the market-entry threshold takes the form $\Nentropt = \Theta(D^{-c})$ where $c$ increases from $\frac{1}{\scalingexp}$ to $\frac{\scalingexp+1}{\scalingexp}$ to $\frac{\scalingexp'+1}{\scalingexp'}$ as $D$ decreases. (The third regime only exists when $\delta > 1$.) To connect this to large-language-model marketplaces, if we take $\scalingexp = 0.34$ to be the empirically estimated scaling law exponent for language models \citep{H22}, then $c$ would range from $2.94$ to $3.94$ to potentially even larger. The fact that there are three regimes come from the scaling law derived in Theorem \ref{thm:scalinglawoptregexcessinformal}, as the following proof sketch illustrates.

\begin{proof}[Proof sketch]
The key technical tool is the scaling law for the \textit{excess loss} $\inf_{\reg \in (0, 1)} (\mathbb{E}[L_1(\hat{\beta}(\alpha, \lambda, X)) - L_1(\beta(\alpha, 0))])$ (Theorem \ref{thm:scalinglawoptregexcessinformal}),  which has three regimes of scaling behavior for different values of $N$. We apply the scaling law to analyze the performance of the new company, who faces a safety constraint and has finite data. Analyzing the performance of the incumbent---who has infinite data---is more straightforward, and the incumbent's performance loss is $G_I = D + G_E$. We compute the number of data points $\Nentropt$ needed for the new company to achieve an excess loss of $D$. The full proof is deferred to Appendix \ref{appendix:proofssubsec42}. 
\end{proof}

Theorem \ref{thm:alignment} illustrates that the new company can enter the market with finite data $\Nentropt$, as long as the safety constraint $\constrentr$ placed on the new company is strictly weaker than the constraint $\constrlead$ placed on the incumbent company (inequality (A) in \eqref{eq:safetythresholdnew}). This translates to the difference $D$ being strictly positive. The intuition is that when the new company faces a weaker safety constraint, it can train on a greater number of data points labelled with the performance objective $\beta_1$, which improves performance.

The three regimes in Theorem \ref{thm:alignment} further reveal that the market-entry threshold $\Nentropt$ scales at a faster rate as the difference $D$ between the two safety constraints decreases (Figure \ref{fig:boundB}). The intuition is  since the new company needs to achieve an excess loss of at most $D$, the new company faces a smaller multi-objective scaling exponent $\scalingexp^*$ as $D$ decreases (Theorem \ref{thm:scalinglawoptregexcessinformal}). The new company thus becomes less efficient at  using additional data to improve performance.

\section{Deriving Scaling Laws for Multi-Objective Environments}\label{sec:scalinglaws} 

We formalize and derive our multi-objective scaling laws for the loss (Theorem \ref{thm:scalinglawoptreginformal}) and excess loss (Theorem \ref{thm:scalinglawoptregexcessinformal}). Recall that the problem setting is high-dimensional ridge regression when a fraction $\alpha$ of the training data is labelled according to $\beta_1$ and the rest is labelled according to an alternate objective $\beta_2$. First, following the style of analysis of single-objective ridge regression~\citep[e.g.,][]{CLKZ21, WHS22}, we first compute a \textit{deterministic equivalent} of the loss (Section \ref{subsec:deterministic}). Then we derive the scaling law under the power scaling assumptions on the eigenvalues and alignment coefficients in Section \ref{subsec:assumptions}, both for the loss (Section \ref{subsec:scalinglaw}) and for the excess loss (Section \ref{subsec:scalingexcess}). Proofs are deferred to Appendix \ref{appendix:proofsmultiobjective}.

\subsection{Deterministic equivalent}\label{subsec:deterministic}

We show that the loss of the ridge regression estimator can be approximated as a deterministic quantity. This analysis builds on the random matrix tools in \citet{bach} (see Appendix \ref{appendix:machinery}). Note that our derivation of the deterministic equivalent does \textit{not} place the power scaling assumptions on the eigenvalues or alignment coefficients; in fact, it holds for any linear regression setup which satisfies a standard random matrix theory assumption (Assumption \ref{assumption:MP}). 

We compute the following deterministic equivalent (proof deferred to Appendix \ref{appendix:lemmasollichvariant}).\footnote{Following \citet{bach}, the asymptotic equivalence notation $u \sim v$ means that $u/v$ tends to $1$ as $N$ and $P$ go to $\infty$.}  
\begin{lemma}
\label{lemma:sollichvariant}
Suppose that $N \ge 1$, $P \ge 1$, $\DF$, $\beta_1$, and $\beta_2$ satisfy Assumption \ref{assumption:MP}. Let $\Sigma$ be the covariance matrix of $\DF$, and let $\mix \in [0,1]$ and $\reg \in (0,1)$ be general parameters. Let $\Sigma_{c} = (\Sigma + c I)$ for $c \ge 0$, let $\Matrixhom{1} = \beta_1 \beta_1^T$, let $\Matrixdiff = (\beta_1 - \beta_2)(\beta_1 - \beta_2)^T$, and let $\Matrixmixed{1} = (\beta_1 - \beta_2) \beta_1^T$. Let $\kappa = \kappa(\lambda, N, \Sigma)$ from Definition \ref{def:effectiveregularizer}. Then, it holds that 
\[L_1(\hat{\beta}(\alpha, \lambda, X)) \sim L_1^{\texttt{det}}(\beta_1, \beta_2, \DF, \reg, N, \alpha) =: \frac{T_1 + T_2 + T_3 + T_4 + T_5}{\DegreesFreedomStandard }, \]
where:
\[ T_1 := \kappa^2 \cdot \Tr(\Sigma \Sigma_{\kappa}^{-2} \Matrixhom{1}) , \;\;\; T_2 := (1-\alpha)^2 \left( \Tr\left(\Sigma_{\kappa}^{-2} \Sigma^3 \Matrixdiff \right)\right)  \;\;\; \]
\[T_3 := 2 (1-\alpha) \kappa \cdot \Tr\left(\Sigma_{\kappa}^{-2}  \Sigma^2 \Matrixmixed{1} \right) , \;\;\; T_4 := -2 (1-\alpha)  \kappa \frac{1}{N} \Tr(\Sigma^2 \Sigma_{\kappa} ^{-2}) \cdot  \Tr\left(\Sigma_{\kappa}^{-1} \Sigma \Matrixmixed{1} \right) , \]
\[ T_5 := (1-\alpha) \frac{1}{N} \Tr(\Sigma^2 \Sigma_{\kappa} ^{-2}) \cdot  \left(\Tr\left(\Sigma \Matrixdiff  \right)- 2 (1-\alpha) \Tr\left( \Sigma_{\kappa}^{-1} \Sigma^2 \Matrixdiff  \right) \right),   \;\;\; \DegreesFreedomStandard := 1 - \frac{1}{N} \Tr(\Sigma^2 \Sigma_{\kappa} ^{-2}).\]
\end{lemma}

Lemma \ref{lemma:sollichvariant} shows that the loss can be approximated by a deterministic quantity $L_1^{\texttt{det}}(\beta_1, \beta_2, \DF, \reg, N, \alpha)$ which is sum of five terms, normalized  by the standard degrees of freedom correction $\DegreesFreedomStandard^{-1}$ \citep{bach}. The sum $T_1 + T_2 +T_3$ is the loss of infinite-data ridge regression with regularizer $\kappa$. Terms $T_4$ and $T_5$ capture additional error terms. 

In more detail, term $T_1 / \DegreesFreedomStandard$ captures the standard single-objective environment error for $N$ data points \citep{bach}: i.e., the population error of the single-objective linear regression problem with regularizer $\lambda$ where all of the $N$ training data points are labelled with $\beta_i$. Term $T_2$ is similar to the infinite-data ridgeless regression error but is slightly smaller due to regularization. Term $T_3$
is a cross term which is upper bounded by the geometric mean of term $T_1$ and term $T_2$. Term $T_4$ is another cross term which is subsumed by the other terms. Term $T_5$ captures an overfitting error which increases with the regularizer $\kappa$ and decreases with the amount of data $N$. 

\paragraph{From deterministic equivalents to scaling laws.} In the following two subsections, using the deterministic equivalent from Lemma \ref{lemma:sollichvariant}, we derive \textit{scaling laws}. We make use of the the power scaling assumptions on the covariance and alignment coefficients described in Section \ref{subsec:scalinglaw}, under which the deterministic equivalent takes a cleaner form (Lemma \ref{lemma:sollichmoreprecise} in Appendix \ref{appendix:proofsmultiobjective}). 
We note that strictly speaking, deriving scaling laws requires controlling the error of the deterministic equivalent relative to the actual loss; for simplicity, we do not control errors and instead directly analyze the deterministic equivalent.

\subsection{Scaling law for the loss}\label{subsec:scalinglaw}

We derive scaling laws for the loss $L_1^{\texttt{det}} := L_1^{\texttt{det}}(\beta_1, \beta_2, \DF, \reg, N, \alpha)$. We first prove the following scaling law for a general regularizer $\reg$ (proof deferred to Appendix \ref{appendix:proofscaling}). 
\begin{theorem}
\label{thm:scalinglaw}
Suppose that the power-law scaling assumption holds for the eigenvalues and alignment coefficients with scaling exponents $\gamma, \delta > 0$ and correlation coefficient $\rho \in [0, 1)$, suppose that $P = \infty$. Assume that $\alpha \ge 0.5$ and $\lambda \in (0,1)$. 
Let $L_1^{\texttt{det}} := L_1^{\texttt{det}}(\beta_1, \beta_2, \DF, \reg, N, \alpha)$ be the deterministic equivalent from Lemma \ref{lemma:sollichvariant}. Let $\scalingexp := \min(2(1+\gamma), \delta+\gamma)$. 
Then, the expected loss satisfies: 
\[\mathbb{E}_{\DC}[L_1^{\texttt{det}}] =\Theta\left( \underbrace{\max(\lambda^{\frac{\scalingexp}{1+\gamma}}, N^{-\scalingexp})}_{\text{finite data error}} + \underbrace{(1-\alpha)^2 \cdot (1 - \rho)}_{\text{mixture error}} +   \underbrace{(1-\alpha) \left(
 \frac{\min(\lambda^{-\frac{1}{1+\gamma}}, N)}{N}
\right) 
(1-\rho)
}_{\text{overfitting error}}\right).\]
\end{theorem}

Theorem \ref{thm:scalinglaw} illustrates that the loss is the sum of a \textit{finite data error}, an \textit{overfitting error}, and a \textit{mixture error}. The finite data error for $L_1^{\texttt{det}}$ matches the loss in the single-objective environment for $N$ data points labelled with objective $\beta_1$.  The mixture error equals the loss of the infinite-data ridgeless regression predictor $\beta(\alpha, 0)$. The overfitting error for $L_1^{\texttt{det}}$ equals the error incurred when the regularizer $\reg$ is too small. This term is always at most $(1-\alpha)^{-1}$ times larger than the mixture error, and it is smaller than the mixture error when $\reg$ is sufficiently large relative to $N$.

Due to the overfitting error, the optimal loss is \textit{not} necessarily achieved by taking $\reg \rightarrow 0$ for multi-objective linear regression. In fact, if the regularizer decays too quickly as a function of $N$ (i.e., if $\reg = O(N^{-1-\gamma})$), then the error would converge to $(1-\alpha)(1-\rho)$, which is a factor of $(1-\alpha)^{-1}$ higher than the error of the infinite-data ridgeless predictor $\beta(\alpha, 0)$. 
The fact that $\reg \rightarrow 0$ is suboptimal reveals a sharp disconnect between the multi-objective setting and the single-objective setting where no explicit regularization is necessary to achieve the optimal loss~\citep[see, e.g.,][]{CLKZ21, WHS22}.\footnote{Tempered overfitting \citep{MSAPBN22} can similarly occur in single-objective settings with \textit{noisy observations}. In this sense, labelling some of the data with the alternate objective $\beta_2$ behaves qualitatively similarly to noisy observations. }

In the next result, we compute the optimal regularizer and derive a scaling law under optimal regularization as a corollary of Theorem \ref{thm:scalinglaw}.
\begin{corollary}[Formal version of Theorem \ref{thm:scalinglawoptreginformal}]
\label{cor:scalinglawoptreg}
Consider the setup of Theorem \ref{thm:scalinglaw}. 
Then, the loss under optimal regularization can be expressed as:
\[
\inf_{\reg \in (0,1)} \mathbb{E}_{\DC}[L_1^{\texttt{det}}]  = \begin{cases}
 \Theta\left(N^{-\scalingexp}\right) &\text{ if } N \le (1-\alpha)^{-\frac{1}{\scalingexp}}(1-\rho)^{-\frac{1}{\scalingexp}} \\
 \Theta\left(\left(\frac{N}{(1-\alpha)(1-\rho)}\right)^{-\frac{\scalingexp}{\scalingexp + 1}}\right) &\text{ if } (1-\alpha)^{-\frac{1}{\scalingexp}}(1-\rho)^{-\frac{1}{\scalingexp}} 
 \le N \le  (1-\alpha)^{-\frac{2+\scalingexp}{\scalingexp}} (1-\rho)^{-\frac{1}{\scalingexp}}
\\
\Theta((1-\alpha)^2(1-\rho)) &\text{ if } N \ge (1-\alpha)^{-\frac{2+\scalingexp}{\scalingexp}} (1-\rho)^{-\frac{1}{\scalingexp}},
\end{cases}
\]
where $\scalingexp := \min(2(1+\gamma), \delta+\gamma)$. 
\end{corollary}

The scaling law exponent $\scalingexp^*$ ranges from $\scalingexp$, to $\scalingexp/(\scalingexp+1)$, to $0$ (Figure \ref{fig:loss}). To better understand each regime, we provide intuition for when error term from Theorem \ref{thm:scalinglaw} dominates, the form of the optimal regularizer, and the behavior of the loss. 
\begin{itemize}[leftmargin=*]
    \item \textit{Regime 1: $N \le (1-\alpha)^{-\frac{1}{\scalingexp}}(1-\rho)^{-\frac{1}{\scalingexp}}$.} Since $N$ is small, the finite data error dominates regardless of $\reg$. As a result, like in a single-objective environment, taking $\reg = O(N^{-1-\gamma})$ recovers the optimal loss up to constants. Note that the loss thus behaves as if all $N$ data points were labelled according to $\beta_i$: the learner benefits from \textit{all} of the data, not just the data is labelled according to $\beta_i$. 
    \item \textit{Regime 2: $(1-\alpha)^{-\frac{1}{\scalingexp}}(1-\rho)^{-\frac{1}{\scalingexp}} 
 \le N \le  (1-\alpha)^{-\frac{2+\scalingexp}{\scalingexp}} (1-\rho)^{-\frac{1}{\scalingexp}}$.} In this regime, the finite error term and overfitting error dominate.  Taking $\reg = \Theta\left(\left(\frac{(1-\alpha) (1-\rho)}{N} \right)^{\frac{1+\gamma}{\scalingexp + 1}}\right)$, which equalizes the two error terms, recovers the optimal loss up to constants. The loss in this regime improves with $N$, but at a slower rate than in a single-objective environment. 
    \item \textit{Regime 3: $N \ge (1-\alpha)^{-\frac{2+\scalingexp}{\scalingexp}} (1-\rho)^{-\frac{1}{\scalingexp}}$.}
   Since $N$ is large, the mixture and the overfitting error terms dominate. Taking $\reg = \Theta((N(1-\alpha))^{-1-\gamma})$, which equalizes the two error terms, recovers the optimal loss up to constants. The loss behaves (up to the constants) as if there were \textit{infinitely many data points} from the mixture distribution with weight $\alpha$. This is the minimal possible loss and there is thus no additional benefit for data beyond improving constants.
\end{itemize}
The full proof of Corollary \ref{cor:scalinglawoptreg} is deferred to Appendix \ref{appendix:proofcor}. 

\subsection{Scaling law for the excess loss}\label{subsec:scalingexcess}

Now, we turn to scaling laws for the excess loss $\mathbb{E}_{\DC}[L_1^{\texttt{det}}(\beta_1, \beta_2, \DF, \reg, N, \alpha) - L_1(\beta(\alpha, 0))] $, , which is normalized by the loss of the infinite-data ridgeless predictor $\beta(\alpha, 0)$. We first prove the following scaling law for a general regularizer $\reg$, assuming that $\alpha \ge 0.75$ (proof deferred to Appendix \ref{appendix:proofscalingexcess}).\footnote{The assumption that $\alpha \ge 0.75$ simplifies the closed-form expression for the deterministic equivalent of the excess loss in Lemma \ref{lemma:sollichmoreprecise}. We defer a broader characterization of scaling laws for the excess loss to future work.} 
\begin{theorem}
\label{thm:scalinglawexcess}
Suppose that the power-law scaling assumption holds for the eigenvalues and alignment coefficients with scaling exponents $\gamma, \delta > 0$ and correlation coefficient $\rho \in [0, 1)$, suppose that $P = \infty$. Assume that $\alpha \ge 0.75$ and $\lambda \in (0,1)$. 
Let $L_1^{\texttt{det}} := L_1^{\texttt{det}}(\beta_1, \beta_2, \DF, \reg, N, \alpha)$ be the deterministic equivalent from Lemma \ref{lemma:sollichvariant}. Let $\scalingexp := \min(2(1+\gamma), \delta+\gamma)$ and let $\scalingexp' = \min(1+\gamma, \delta+\gamma)$ 
Then, the expected loss satisfies: 
\begin{align*}
 &\mathbb{E}_{\DC}[L_1^{\texttt{det}} - L_1(\beta(\alpha, 0))]  \\
 &=  \Theta\left( \underbrace{\max(\lambda^{\frac{\scalingexp}{1+\gamma}}, N^{-\scalingexp})}_{\text{finite data error}} + \underbrace{(1-\rho) (1-\alpha) \max(\lambda^{\frac{\scalingexp'}{1+\gamma}}, N^{-\scalingexp'})}_{\text{mixture finite data error}} +   \underbrace{(1-\alpha) \left(
 \frac{\min(\lambda^{-\frac{1}{1+\gamma}}, N)}{N}
\right) 
(1-\rho)
}_{\text{overfitting error}}\right).
\end{align*}
\end{theorem}

Theorem \ref{thm:scalinglawexcess} illustrates that the loss is the sum of a \textit{finite data error}, an \textit{overfitting error}, and a \textit{mixture finite data error}. In comparison with Theorem \ref{thm:scalinglaw}, the difference is that the mixture error is replaced by the mixture finite data error. Interestingly, the mixture finite data error exhibits a different asymptotic dependence with respect to $\lambda$ and $N$ than the finite data error: the asymptotic rate of decay scales with $\scalingexp'$ rather than $\scalingexp$. In fact, the rate is \textit{slower} for the mixture finite data error than the finite data error as long as $\delta > 1$ (since this means that $\scalingexp' < \scalingexp$). 

Since the optimal excess loss is also not necessarily achieved by taking $\reg \rightarrow 0$, we compute the optimal regularizer for the excess loss and derive a scaling law under optimal regularization as a corollary of Theorem \ref{thm:scalinglawexcess}.
\begin{corollary}[Formal version of Theorem \ref{thm:scalinglawoptregexcessinformal}]
\label{cor:scalinglawoptregexcess}
Consider the setup of Theorem \ref{thm:scalinglawexcess}. 
The excess loss under optimal regularization can be expressed as:
\begin{align*}
&\inf_{\reg \in (0,1)} (\mathbb{E}_{\DC}[L_1^{\texttt{det}} - L_1(\beta(\alpha, 0))]) \\
&= \begin{cases}
 \Theta\left(N^{-\scalingexp}\right) &\text{ if } N \le (1-\alpha)^{-\frac{1}{\scalingexp}}(1-\rho)^{-\frac{1}{\scalingexp}} \\
 \Theta\left(\left(\frac{N}{(1-\alpha)(1-\rho)}\right)^{-\frac{\scalingexp}{\scalingexp + 1}}\right) &\text{ if } 
(1-\alpha)^{-\frac{1}{\scalingexp}}(1-\rho)^{-\frac{1}{\scalingexp}} \le N \le  (1-\alpha)^{-\frac{\scalingexp'+1}{\scalingexp - \scalingexp'}}(1-\rho)^{-\frac{\scalingexp'+1}{\scalingexp - \scalingexp'}}
\\
\Theta\left((1-\alpha) (1-\rho) N^{-\frac{\scalingexp'}{\scalingexp'+1}} \right) &\text{ if } N \ge (1-\alpha)^{-\frac{\scalingexp'+1}{\scalingexp - \scalingexp'}}(1-\rho)^{-\frac{\scalingexp'+1}{\scalingexp - \scalingexp'}},
\end{cases}
\end{align*}
where $\scalingexp := \min(2(1+\gamma), \delta+\gamma)$ and $\scalingexp' = \min(1+\gamma, \delta + \gamma)$. 
\end{corollary}

The scaling law exponent $\scalingexp^*$ ranges from $\scalingexp$, to $\scalingexp/(\scalingexp+1)$, to $\scalingexp'/(\scalingexp'+1)$ (Figure \ref{fig:excessloss}). The first two regimes behave similarly to  Corollary \ref{cor:scalinglawoptreg}, and the key difference arises in the third regime (when $N$ is large). In the third regime ($N \ge (1-\alpha)^{-\frac{\scalingexp'+1}{\scalingexp - \scalingexp'}}(1-\rho)^{-\frac{\scalingexp'+1}{\scalingexp - \scalingexp'}}$), the mixture finite data error and the overfitting error terms dominate. Taking $\reg = \Theta\left(N^{-\frac{1+\gamma}{\scalingexp'+1}} \right)$---which equalizes these two error terms---recovers the optimal loss up to constants. The resulting scaling behavior captures that in this regime, additional data meaningfully improves the \textit{excess loss}, even though additional data only improves the loss in terms of constants. The full proof of Corollary \ref{cor:scalinglawoptregexcess} is deferred to Appendix \ref{appendix:proofscalingoptregexcess}. 

\section{Discussion}\label{sec:discussion}

We studied market entry in marketplaces for machine learning models, showing that pressure to satisfy safety constraints can reduce barriers to entry for new companies. We modelled the marketplace using a high-dimensional multi-objective linear regression model. Our key finding was that a new company can consistently enter the marketplace with significantly less data than the incumbent. En route to proving these results, we derive scaling laws for multi-objective regression, showing that the scaling rate becomes slower when the dataset size is large.

\paragraph{Potential implications for regulation.} Our results have nuanced design consequences for regulators, who implicitly influence the level of safety that each company needs to achieve to avoid reputational damage. On one hand, our results suggest that placing greater scrutiny on dominant companies can encourage market entry and create a more competitive marketplace of model providers. On the other hand, market entry does come at a cost to the safety objective: the smaller companies exploit that they can incur more safety violations while maintaining their reputation, which leads to a race to the bottom for safety. Examining the tradeoffs between market competitiveness and safety compliance is an important direction for future work.

\paragraph{Barriers to market entry for online platforms.} While we focused on language models, we expect that our conceptual findings about market entry also extend to recommendation and social media platforms. 

In particular, our motivation and modeling assumptions capture key aspects of these online platforms. Policymakers have raised concerns have been raised about barriers to entry for social media platforms \citep{stiger19}, motivated by the fact that social media platforms such as X and Facebook each have over a half billion users \citep{statista_facebook_users_2024, ingram2024fewer}. Incumbent companies risk reputational damage if their model violates safety-oriented objectives---many recommendation platforms have faced scrutiny for promoting hate speech \citep{DigitalServicesAct2022}, divisive content \citep{RBL21}, and excessive use by users \citep{hasan2018excessive}, even when recommendations perform well in terms of generating user engagement. This means that incumbent platforms must balance optimizing engagement with controlling negative societal impacts \citep{BST22}. Moreover, new companies face less regulatory scrutiny, given that some regulations explicitly place more stringent requirements on companies with large user bases: for example the Digital Services Act \citep{DigitalServicesAct2022} places a greater responsibility on Very Large Online Platforms (with over 45 million users per month) to identify and remove illegal or harmful content. 

Given that incumbent platforms similarly face more pressure to satisfy safety-oriented objectives, our results suggest that multi-objective learning can also reduce barriers to entry for new online platforms. 

\paragraph{Limitations.} Our model for interactions between companies and users makes several simplifying assumptions. For example, we focused entirely whether the new company can enter the market, which leaves open the question of whether the new company can survive in the long run. Moreover, we assumed that all users choose the model with the highest overall performance. However, different users often care about performance on different queries; this could create an incentive for specialization, which could also reduce barriers to entry and market concentration. Finally, we focused on direct interactions between model providers and users, but in reality, downstream providers sometimes build services on top of a foundation model. Understanding how these market complexities affect market entry as well as long-term concentration is an interesting direction for future work. 

Furthermore, our model also made the simplifying assumption that performance and safety trade off according to a multi-objective regression problem. 
However, not all safety objectives fit the mold of linear coefficients within linear regression. For some safety objectives such as privacy, we still expect that placing greater scrutiny on dominant companies could similarly reduce barriers to entry. Nonetheless, for other safety or societal considerations, we do expect that the implications for market entry might be fundamentally different. For example, if the safety objective is a multi-group performance criteria, and there is a single predictor that achieves zero accuracy on all distributions, then a dominant company with infinite data would be able to retain all users even if the company faces greater scrutiny. Extending our model to capture a broader scope of safety objectives is a natural direction for future work.

\section{Acknowledgments}
We thank Alireza Fallah, Jiahai Feng, Nika Haghtalab, Andy Haupt, Erik Jones, Jon Kleinberg, Ben Laufer, Neil Mallinar, Judy Shen, Alex Wei, Xuelin Yang, and Eric Zhao for useful feedback on this project. This work was partially supported by an Open Philanthropy AI fellowship and partially supported by the European Union (ERC-2022-SYG-OCEAN-101071601).

\bibliographystyle{plainnat}
\bibliography{ref.bib}

\begin{thebibliography}{87}
\providecommand{\natexlab}[1]{#1}
\providecommand{\url}[1]{\texttt{#1}}
\expandafter\ifx\csname urlstyle\endcsname\relax
  \providecommand{\doi}[1]{doi: #1}\else
  \providecommand{\doi}{doi: \begingroup \urlstyle{rm}\Url}\fi

\bibitem[Aridor et~al.(2020)Aridor, Mansour, Slivkins, and Wu]{AMSW20}
Guy Aridor, Yishay Mansour, Aleksandrs Slivkins, and Zhiwei~Steven Wu.
\newblock Competing bandits: The perils of exploration under competition.
\newblock \emph{CoRR}, abs/2007.10144, 2020.

\bibitem[Atanasov et~al.(2024)Atanasov, Zavatone{-}Veth, and Pehlevan]{AZP24}
Alexander~B. Atanasov, Jacob~A. Zavatone{-}Veth, and Cengiz Pehlevan.
\newblock Scaling and renormalization in high-dimensional regression.
\newblock \emph{CoRR}, abs/2405.00592, 2024.

\bibitem[Bach(2023)]{bach}
Francis~R. Bach.
\newblock High-dimensional analysis of double descent for linear regression with random projections.
\newblock \emph{CoRR}, abs/2303.01372, 2023.

\bibitem[Bahri et~al.(2021)Bahri, Dyer, Kaplan, Lee, and Sharma]{B21}
Yasaman Bahri, Ethan Dyer, Jared Kaplan, Jaehoon Lee, and Utkarsh Sharma.
\newblock Explaining neural scaling laws.
\newblock \emph{CoRR}, abs/2102.06701, 2021.

\bibitem[Bai et~al.(2022)Bai, Jones, Ndousse, Askell, Chen, DasSarma, Drain, Fort, Ganguli, Henighan, Joseph, Kadavath, Kernion, Conerly, Showk, Elhage, Hatfield{-}Dodds, Hernandez, Hume, Johnston, Kravec, Lovitt, Nanda, Olsson, Amodei, Brown, Clark, McCandlish, Olah, Mann, and Kaplan]{bairlhf2022}
Yuntao Bai, Andy Jones, Kamal Ndousse, Amanda Askell, Anna Chen, Nova DasSarma, Dawn Drain, Stanislav Fort, Deep Ganguli, Tom Henighan, Nicholas Joseph, Saurav Kadavath, Jackson Kernion, Tom Conerly, Sheer~El Showk, Nelson Elhage, Zac Hatfield{-}Dodds, Danny Hernandez, Tristan Hume, Scott Johnston, Shauna Kravec, Liane Lovitt, Neel Nanda, Catherine Olsson, Dario Amodei, Tom~B. Brown, Jack Clark, Sam McCandlish, Chris Olah, Benjamin Mann, and Jared Kaplan.
\newblock Training a helpful and harmless assistant with reinforcement learning from human feedback.
\newblock \emph{CoRR}, abs/2204.05862, 2022.

\bibitem[Baker(2007)]{B07}
Jonathan~B. Baker.
\newblock Beyond {S}chumpeter vs. {A}rrow: How antitrust fosters innovation.
\newblock \emph{Antitrust Law Journal}, 74\penalty0 (3):\penalty0 575--602, 2007.

\bibitem[Ben{-}Porat and Tennenholtz(2017)]{BT17}
Omer Ben{-}Porat and Moshe Tennenholtz.
\newblock Best response regression.
\newblock In \emph{Advances in Neural Information Processing Systems 30: Annual Conference on Neural Information Processing Systems (NIPS)}, pages 1499--1508, 2017.

\bibitem[Ben{-}Porat and Tennenholtz(2019)]{BT19}
Omer Ben{-}Porat and Moshe Tennenholtz.
\newblock Regression equilibrium.
\newblock In \emph{Proceedings of the 2019 {ACM} Conference on Economics and Computation (EC)}, pages 173--191. {ACM}, 2019.

\bibitem[Bengani et~al.(2022)Bengani, Stray, and Thorburn]{BST22}
Priyanjana Bengani, Jonathan Stray, and Luke Thorburn.
\newblock What’s right and what’s wrong with optimizing for engagement.
\newblock \emph{Understanding Recommenders}, Apr 2022.
\newblock URL \url{https://medium.com/understanding-recommenders/whats-right-and-what-s-wrong-with-optimizing-for-engagement-5abaac021851}.

\bibitem[Bennett et~al.(2023)Bennett, Stulz, and Wang]{Bennett2023}
Benjamin Bennett, Rene~M. Stulz, and Zexi Wang.
\newblock Does greater public scrutiny hurt a firm's performance?
\newblock Available at SSRN: https://ssrn.com/abstract=4321191, 2023.

\bibitem[Blum et~al.(2017)Blum, Haghtalab, Procaccia, and Qiao]{BHPQ17}
Avrim Blum, Nika Haghtalab, Ariel~D. Procaccia, and Mingda Qiao.
\newblock Collaborative {PAC} learning.
\newblock In \emph{Advances in Neural Information Processing Systems 30: Annual Conference on Neural Information Processing Systems 2017, December 4-9, 2017, Long Beach, CA, {USA}}, pages 2392--2401, 2017.

\bibitem[Bordelon et~al.(2020)Bordelon, Canatar, and Pehlevan]{BCP20}
Blake Bordelon, Abdulkadir Canatar, and Cengiz Pehlevan.
\newblock Spectrum dependent learning curves in kernel regression and wide neural networks.
\newblock In \emph{Proceedings of the 37th International Conference on Machine Learning, {ICML} 2020, 13-18 July 2020, Virtual Event}, volume 119 of \emph{Proceedings of Machine Learning Research}, pages 1024--1034. {PMLR}, 2020.

\bibitem[Bordelon et~al.(2024)Bordelon, Atanasov, and Pehlevan]{BAP24}
Blake Bordelon, Alexander~B. Atanasov, and Cengiz Pehlevan.
\newblock A dynamical model of neural scaling laws.
\newblock \emph{CoRR}, abs/2402.01092, 2024.

\bibitem[{California Legislature}(2024)]{CaliforniaLegislation2024}
{California Legislature}.
\newblock California senate bill no. 1047 (2023-2024).
\newblock \url{https://leginfo.legislature.ca.gov/faces/billTextClient.xhtml?bill_id=202320240SB1047}, 2024.

\bibitem[Calvano and Polo(2021)]{C21}
Emilio Calvano and Michele Polo.
\newblock Market power, competition and innovation in digital markets: A survey.
\newblock \emph{Information Economics and Policy}, 54:\penalty0 100853, 2021.

\bibitem[Cen et~al.(2023)Cen, Hopkins, Ilyas, Madry, Struckman, and Videgaray~Caso]{Cen2023}
Sarah~Huiyi Cen, Aspen Hopkins, Andrew Ilyas, Aleksander Madry, Isabella Struckman, and Luis Videgaray~Caso.
\newblock A{I} supply chains.
\newblock Available at SSRN: https://ssrn.com/abstract=4789403, 2023.

\bibitem[{Competition and Markets Authority}(2024)]{UK2023AIFoundationModels}
{Competition and Markets Authority}.
\newblock A{I} foundation models: Technical update report.
\newblock Technical report, UK Government, 2024.
\newblock URL \url{https://assets.publishing.service.gov.uk/media/661e5a4c7469198185bd3d62/AI_Foundation_Models_technical_update_report.pdf}.

\bibitem[Cook(2024)]{cook_2024_ai_tools}
Jodie Cook.
\newblock {ChatGPT}, {C}laude, {G}emini or another: The {AI} tool entrepreneurs prefer.
\newblock \emph{Forbes}, 2024.
\newblock URL \url{https://www.forbes.com/sites/jodiecook/2024/05/07/chatgpt-claude-gemini-or-another-the-ai-tool-entrepreneurs-prefer/}.

\bibitem[Covert et~al.(2024)Covert, Ji, Hashimoto, and Zou]{CJHZ24}
Ian Covert, Wenlong Ji, Tatsunori Hashimoto, and James Zou.
\newblock Scaling laws for the value of individual data points in machine learning.
\newblock \emph{CoRR}, abs/2405.20456, 2024.

\bibitem[Cui et~al.(2021)Cui, Loureiro, Krzakala, and Zdeborov{\'{a}}]{CLKZ21}
Hugo Cui, Bruno Loureiro, Florent Krzakala, and Lenka Zdeborov{\'{a}}.
\newblock Generalization error rates in kernel regression: The crossover from the noiseless to noisy regime.
\newblock In \emph{Advances in Neural Information Processing Systems 34: Annual Conference on Neural Information Processing Systems 2021, NeurIPS 2021, December 6-14, 2021, virtual}, pages 10131--10143, 2021.

\bibitem[Dean et~al.(2022)Dean, Curmei, Ratliff, Morgenstern, and Fazel]{DCRMF22}
Sarah Dean, Mihaela Curmei, Lillian~J. Ratliff, Jamie Morgenstern, and Maryam Fazel.
\newblock Multi-learner risk reduction under endogenous participation dynamics.
\newblock \emph{CoRR}, abs/2206.02667, 2022.

\bibitem[Dohmatob et~al.(2024)Dohmatob, Feng, Yang, Charton, and Kempe]{DFYCK24}
Elvis Dohmatob, Yunzhen Feng, Pu~Yang, Fran{\c{c}}ois Charton, and Julia Kempe.
\newblock A tale of tails: Model collapse as a change of scaling laws.
\newblock \emph{CoRR}, abs/2402.07043, 2024.

\bibitem[Dong et~al.(2019)Dong, Elzayn, Jabbari, Kearns, and Schutzman]{DEJKS19}
Jinshuo Dong, Hadi Elzayn, Shahin Jabbari, Michael~J. Kearns, and Zachary Schutzman.
\newblock Equilibrium characterization for data acquisition games.
\newblock In Sarit Kraus, editor, \emph{Proceedings of the Twenty-Eighth International Joint Conference on Artificial Intelligence, {IJCAI} 2019, Macao, China, August 10-16, 2019}, pages 252--258. ijcai.org, 2019.

\bibitem[{European Union}(2022{\natexlab{a}})]{DigitalServicesAct2022}
{European Union}.
\newblock Regulation ({EU}) 2022/2065 of the {E}uropean {P}arliament and of the {C}ouncil of 19 {O}ctober 2022 on a single market for digital services and {A}mending {D}irective 2000/31/{EC} ({D}igital {S}ervices {A}ct).
\newblock Official Journal of the European Union, 2022{\natexlab{a}}.
\newblock URL \url{https://eur-lex.europa.eu/legal-content/EN/TXT/?uri=celex%3A32022R2065}.

\bibitem[{European Union}(2022{\natexlab{b}})]{eu_digital_markets_act_2022}
{European Union}.
\newblock Regulation ({EU}) 2022/1925 of the {E}uropean parliament and of the {C}ouncil of 14 {S}eptember 2022 on contestable and fair markets in the digital sector and {A}mending {D}irectives ({EU}) 2019/1937 and ({EU}) 2020/1828 ({D}igital {M}arkets {A}ct), 2022{\natexlab{b}}.
\newblock URL \url{https://eur-lex.europa.eu/eli/reg/2022/1925/oj}.

\bibitem[Fallah and Jordan(2023)]{FJ23}
Alireza Fallah and Michael~I. Jordan.
\newblock Contract design with safety inspections.
\newblock \emph{CoRR}, abs/2311.02537, 2023.

\bibitem[Fallah et~al.(2024)Fallah, Jordan, Makhdoumi, and Malekian]{FJMM24}
Alireza Fallah, Michael~I. Jordan, Ali Makhdoumi, and Azarakhsh Malekian.
\newblock On three-layer data markets.
\newblock \emph{CoRR}, abs/2402.09697, 2024.

\bibitem[Feng et~al.(2019)Feng, Gradwohl, Hartline, Johnsen, and Nekipelov]{FGHJN19}
Yiding Feng, Ronen Gradwohl, Jason~D. Hartline, Aleck~C. Johnsen, and Denis Nekipelov.
\newblock Bias-variance games.
\newblock \emph{CoRR}, abs/1909.03618, 2019.

\bibitem[Gal and Aviv(2020)]{GA20}
Michal~S Gal and Oshrit Aviv.
\newblock The competitive effects of the gdpr.
\newblock \emph{Journal of Competition Law \& Economics}, 16\penalty0 (3):\penalty0 349--391, 05 2020.

\bibitem[Gao et~al.(2023)Gao, Schulman, and Hilton]{GSH23}
Leo Gao, John Schulman, and Jacob Hilton.
\newblock Scaling laws for reward model overoptimization.
\newblock In \emph{International Conference on Machine Learning, {ICML} 2023, 23-29 July 2023, Honolulu, Hawaii, {USA}}, volume 202 of \emph{Proceedings of Machine Learning Research}, pages 10835--10866. {PMLR}, 2023.

\bibitem[Gerstgrasser et~al.(2024)Gerstgrasser, Schaeffer, Dey, Rafailov, Sleight, Hughes, Korbak, Agrawal, Pai, Gromov, Roberts, Yang, Donoho, and Koyejo]{GSD24}
Matthias Gerstgrasser, Rylan Schaeffer, Apratim Dey, Rafael Rafailov, Henry Sleight, John Hughes, Tomasz Korbak, Rajashree Agrawal, Dhruv Pai, Andrey Gromov, Daniel~A. Roberts, Diyi Yang, David~L. Donoho, and Sanmi Koyejo.
\newblock Is model collapse inevitable? breaking the curse of recursion by accumulating real and synthetic data.
\newblock \emph{CoRR}, abs/2404.01413, 2024.

\bibitem[Ginart et~al.(2021)Ginart, Zhang, Kwon, and Zou]{GZKZ21}
Tony Ginart, Eva Zhang, Yongchan Kwon, and James Zou.
\newblock Competing {AI:} how does competition feedback affect machine learning?
\newblock In Arindam Banerjee and Kenji Fukumizu, editors, \emph{The 24th International Conference on Artificial Intelligence and Statistics (AISTATS)}, volume 130 of \emph{Proceedings of Machine Learning Research}, pages 1693--1701, 2021.

\bibitem[Goyal et~al.(2024)Goyal, Maini, Lipton, Raghunathan, and Kolter]{GMLRK24}
Sachin Goyal, Pratyush Maini, Zachary~C. Lipton, Aditi Raghunathan, and J.~Zico Kolter.
\newblock Scaling laws for data filtering---{D}ata curation cannot be compute agnostic.
\newblock \emph{CoRR}, abs/2404.07177, 2024.

\bibitem[Gradwohl and Tennenholtz(2023)]{GT23}
Ronen Gradwohl and Moshe Tennenholtz.
\newblock Coopetition against an {A}mazon.
\newblock \emph{J. Artif. Intell. Res.}, 76:\penalty0 1077--1116, 2023.

\bibitem[Hachem et~al.(2007)Hachem, Loubaton, and Najim]{HLN07}
Walid Hachem, Philippe Loubaton, and Jamal Najim.
\newblock {Deterministic equivalents for certain functionals of large random matrices}.
\newblock \emph{The Annals of Applied Probability}, 17\penalty0 (3):\penalty0 875 -- 930, 2007.

\bibitem[Haghtalab et~al.(2022)Haghtalab, Jordan, and Zhao]{HJZ22}
Nika Haghtalab, Michael~I. Jordan, and Eric Zhao.
\newblock On-demand sampling: Learning optimally from multiple distributions.
\newblock In \emph{Advances in Neural Information Processing Systems 35: Annual Conference on Neural Information Processing Systems 2022, NeurIPS 2022, New Orleans, LA, USA, November 28 - December 9, 2022}, 2022.

\bibitem[Handina and Mazumdar(2024)]{HM24}
Tinashe Handina and Eric Mazumdar.
\newblock Rethinking scaling laws for learning in strategic environments.
\newblock \emph{CoRR}, abs/2402.07588, 2024.

\bibitem[Hardt et~al.(2022)Hardt, Jagadeesan, and Mendler{-}D{\"{u}}nner]{HJM22}
Moritz Hardt, Meena Jagadeesan, and Celestine Mendler{-}D{\"{u}}nner.
\newblock Performative power.
\newblock In \emph{Advances in Neural Information Processing Systems 35: Annual Conference on Neural Information Processing Systems 2022, NeurIPS 2022, New Orleans, LA, USA, November 28 - December 9, 2022}, 2022.

\bibitem[Harrington(1988)]{H88}
Winston Harrington.
\newblock Enforcement leverage when penalties are restricted.
\newblock \emph{Journal of Public Economics}, 37\penalty0 (1):\penalty0 29--53, 1988.
\newblock ISSN 0047-2727.

\bibitem[Hasan et~al.(2018)Hasan, Jha, and Liu]{hasan2018excessive}
Md~Rajibul Hasan, Ashish~Kumar Jha, and Yi~Liu.
\newblock Excessive use of online video streaming services: Impact of recommender system use, psychological factors, and motives.
\newblock \emph{Computers in Human Behavior}, 80:\penalty0 220--228, 2018.

\bibitem[Hashimoto(2021)]{H21}
Tatsunori Hashimoto.
\newblock Model performance scaling with multiple data sources.
\newblock In \emph{Proceedings of the 38th International Conference on Machine Learning, {ICML} 2021, 18-24 July 2021, Virtual Event}, volume 139 of \emph{Proceedings of Machine Learning Research}, pages 4107--4116. {PMLR}, 2021.

\bibitem[Hashimoto et~al.(2018)Hashimoto, Srivastava, Namkoong, and Liang]{HSNL18}
Tatsunori Hashimoto, Megha Srivastava, Hongseok Namkoong, and Percy Liang.
\newblock Fairness without demographics in repeated loss minimization.
\newblock In Jennifer Dy and Andreas Krause, editors, \emph{Proceedings of the 35th International Conference on Machine Learning}, volume~80 of \emph{Proceedings of Machine Learning Research}, pages 1929--1938. PMLR, 10--15 Jul 2018.

\bibitem[Hastie et~al.(2019)Hastie, Montanari, Rosset, and Tibshirani]{HMRT19}
Trevor Hastie, Andrea Montanari, Saharon Rosset, and Ryan~J. Tibshirani.
\newblock Surprises in high-dimensional ridgeless least squares interpolation.
\newblock \emph{CoRR}, abs/1903.08560, 2019.

\bibitem[Hernandez et~al.(2021)Hernandez, Kaplan, Henighan, and McCandlish]{hernandez2021scaling}
Danny Hernandez, Jared Kaplan, Tom Henighan, and Sam McCandlish.
\newblock Scaling laws for transfer.
\newblock \emph{arXiv preprint arXiv:2102.01293}, 2021.

\bibitem[Hoffmann et~al.(2022)Hoffmann, Borgeaud, Mensch, Buchatskaya, Cai, Rutherford, de~Las~Casas, Hendricks, Welbl, Clark, Hennigan, Noland, Millican, van~den Driessche, Damoc, Guy, Osindero, Simonyan, Elsen, Rae, Vinyals, and Sifre]{H22}
Jordan Hoffmann, Sebastian Borgeaud, Arthur Mensch, Elena Buchatskaya, Trevor Cai, Eliza Rutherford, Diego de~Las~Casas, Lisa~Anne Hendricks, Johannes Welbl, Aidan Clark, Tom Hennigan, Eric Noland, Katie Millican, George van~den Driessche, Bogdan Damoc, Aurelia Guy, Simon Osindero, Karen Simonyan, Erich Elsen, Jack~W. Rae, Oriol Vinyals, and Laurent Sifre.
\newblock Training compute-optimal large language models.
\newblock \emph{CoRR}, abs/2203.15556, 2022.

\bibitem[Immorlica et~al.(2011)Immorlica, Kalai, Lucier, Moitra, Postlewaite, and Tennenholtz]{IKLMPT11}
Nicole Immorlica, Adam~Tauman Kalai, Brendan Lucier, Ankur Moitra, Andrew Postlewaite, and Moshe Tennenholtz.
\newblock Dueling algorithms.
\newblock In \emph{Proceedings of the 43rd {ACM} Symposium on Theory of Computing (STOC)}, pages 215--224, 2011.

\bibitem[Ingram(2024)]{ingram2024fewer}
David Ingram.
\newblock Fewer people using {E}lon {M}usk's {X} as platform struggles to keep users.
\newblock \emph{NBC News}, 2024.
\newblock URL \url{https://www.nbcnews.com/tech/tech-news/fewer-people-using-elon-musks-x-struggles-keep-users-rcna144115}.

\bibitem[Iyer and Ke(2022)]{IK22}
Ganesh Iyer and T.~Tony Ke.
\newblock Competitive algorithmic targeting and model selection.
\newblock Available at SSRN: https://ssrn.com/abstract=4214973, 2022.

\bibitem[Jagadeesan et~al.(2023{\natexlab{a}})Jagadeesan, Jordan, and Haghtalab]{JJH23}
Meena Jagadeesan, Michael~I. Jordan, and Nika Haghtalab.
\newblock Competition, alignment, and equilibria in digital marketplaces.
\newblock In \emph{Thirty-Seventh {AAAI} Conference on Artificial Intelligence, {AAAI} 2023, Thirty-Fifth Conference on Innovative Applications of Artificial Intelligence, {IAAI} 2023, Thirteenth Symposium on Educational Advances in Artificial Intelligence, {EAAI} 2023, Washington, DC, USA, February 7-14, 2023}, pages 5689--5696. {AAAI} Press, 2023{\natexlab{a}}.

\bibitem[Jagadeesan et~al.(2023{\natexlab{b}})Jagadeesan, Jordan, Steinhardt, and Haghtalab]{JJSH23}
Meena Jagadeesan, Michael~I. Jordan, Jacob Steinhardt, and Nika Haghtalab.
\newblock Improved {B}ayes risk can yield reduced social welfare under competition.
\newblock In \emph{Advances in Neural Information Processing Systems 36: Annual Conference on Neural Information Processing Systems 2023, NeurIPS 2023, New Orleans, LA, USA, December 10 - 16, 2023}, 2023{\natexlab{b}}.

\bibitem[Jain et~al.(2024)Jain, Montanari, and Sasoglu]{JMS24}
Ayush Jain, Andrea Montanari, and Eren Sasoglu.
\newblock Scaling laws for learning with real and surrogate data.
\newblock \emph{CoRR}, abs/2402.04376, 2024.

\bibitem[Jullien and Sand-Zantman(2021)]{J21}
Bruno Jullien and Wilfried Sand-Zantman.
\newblock The economics of platforms: A theory guide for competition policy.
\newblock \emph{Information Economics and Policy}, 54:\penalty0 100880, 2021.

\bibitem[Kaplan et~al.(2020)Kaplan, McCandlish, Henighan, Brown, Chess, Child, Gray, Radford, Wu, and Amodei]{K20}
Jared Kaplan, Sam McCandlish, Tom Henighan, Tom~B. Brown, Benjamin Chess, Rewon Child, Scott Gray, Alec Radford, Jeffrey Wu, and Dario Amodei.
\newblock Scaling laws for neural language models.
\newblock \emph{CoRR}, abs/2001.08361, 2020.

\bibitem[Kleinberg and Raghavan(2021)]{KR21}
Jon~M. Kleinberg and Manish Raghavan.
\newblock Algorithmic monoculture and social welfare.
\newblock \emph{Proc. Natl. Acad. Sci. {USA}}, 118\penalty0 (22):\penalty0 e2018340118, 2021.

\bibitem[Kwon et~al.(2022)Kwon, Ginart, and Zou]{KGZ22}
Yongchan Kwon, Tony Ginart, and James Zou.
\newblock Competition over data: how does data purchase affect users?
\newblock \emph{Trans. Mach. Learn. Res.}, 2022.

\bibitem[Laufer et~al.(2024)Laufer, Kleinberg, and Heidari]{LKH24}
Benjamin Laufer, Jon~M. Kleinberg, and Hoda Heidari.
\newblock Fine-tuning games: Bargaining and adaptation for general-purpose models.
\newblock In \emph{Proceedings of the {ACM} on Web Conference 2024, {WWW} 2024, Singapore, May 13-17, 2024}, pages 66--76. {ACM}, 2024.

\bibitem[Lin et~al.(2024)Lin, Wu, Kakade, Bartlett, and Lee]{LWKBL24}
Licong Lin, Jingfeng Wu, Sham~M. Kakade, Peter~L. Bartlett, and Jason~D. Lee.
\newblock Scaling laws in linear regression: Compute, parameters, and data.
\newblock \emph{CoRR}, abs/2406.08466, 2024.

\bibitem[Mallinar et~al.(2022)Mallinar, Simon, Abedsoltan, Pandit, Belkin, and Nakkiran]{MSAPBN22}
Neil Mallinar, James~B. Simon, Amirhesam Abedsoltan, Parthe Pandit, Misha Belkin, and Preetum Nakkiran.
\newblock Benign, tempered, or catastrophic: Toward a refined taxonomy of overfitting.
\newblock In Sanmi Koyejo, S.~Mohamed, A.~Agarwal, Danielle Belgrave, K.~Cho, and A.~Oh, editors, \emph{Advances in Neural Information Processing Systems 35: Annual Conference on Neural Information Processing Systems 2022, NeurIPS 2022, New Orleans, LA, USA, November 28 - December 9, 2022}, 2022.

\bibitem[Mallinar et~al.(2024)Mallinar, Zane, Frei, and Yu]{MZFY24}
Neil Mallinar, Austin Zane, Spencer Frei, and Bin Yu.
\newblock Minimum-norm interpolation under covariate shift.
\newblock \emph{CoRR}, abs/2404.00522, 2024.

\bibitem[Marčenko and Pastur(1967)]{marchenko}
V~A Marčenko and L~A Pastur.
\newblock Distribution of eigenvalues for some sets of random matrices.
\newblock \emph{Mathematics of the USSR-Sbornik}, 1\penalty0 (4):\penalty0 457, apr 1967.

\bibitem[Mendler-D{\"u}nner et~al.(2024)Mendler-D{\"u}nner, Carovano, and Hardt]{MCH2024}
Celestine Mendler-D{\"u}nner, Gabriele Carovano, and Moritz Hardt.
\newblock An engine not a camera: Measuring performative power of online search.
\newblock \emph{CoRR}, abs/2405.19073, 2024.

\bibitem[Mohri et~al.(2019)Mohri, Sivek, and Suresh]{MSS19}
Mehryar Mohri, Gary Sivek, and Ananda~Theertha Suresh.
\newblock Agnostic federated learning.
\newblock In \emph{Proceedings of the 36th International Conference on Machine Learning, {ICML} 2019, 9-15 June 2019, Long Beach, California, {USA}}, volume~97 of \emph{Proceedings of Machine Learning Research}, pages 4615--4625. {PMLR}, 2019.

\bibitem[Muennighoff et~al.(2023)Muennighoff, Rush, Barak, Scao, Tazi, Piktus, Pyysalo, Wolf, and Raffel]{MRBST23}
Niklas Muennighoff, Alexander~M. Rush, Boaz Barak, Teven~Le Scao, Nouamane Tazi, Aleksandra Piktus, Sampo Pyysalo, Thomas Wolf, and Colin~A. Raffel.
\newblock Scaling data-constrained language models.
\newblock In \emph{Advances in Neural Information Processing Systems 36: Annual Conference on Neural Information Processing Systems 2023, NeurIPS 2023, New Orleans, LA, USA, December 10 - 16, 2023}, 2023.

\bibitem[Ouyang et~al.(2022)Ouyang, Wu, Jiang, Almeida, Wainwright, Mishkin, Zhang, Agarwal, Slama, Ray, Schulman, Hilton, Kelton, Miller, Simens, Askell, Welinder, Christiano, Leike, and Lowe]{OJAWMZASR22}
Long Ouyang, Jeffrey Wu, Xu~Jiang, Diogo Almeida, Carroll~L. Wainwright, Pamela Mishkin, Chong Zhang, Sandhini Agarwal, Katarina Slama, Alex Ray, John Schulman, Jacob Hilton, Fraser Kelton, Luke Miller, Maddie Simens, Amanda Askell, Peter Welinder, Paul~F. Christiano, Jan Leike, and Ryan Lowe.
\newblock Training language models to follow instructions with human feedback.
\newblock In \emph{Advances in Neural Information Processing Systems 35: Annual Conference on Neural Information Processing Systems 2022, NeurIPS 2022, New Orleans, LA, USA, November 28 - December 9, 2022}, 2022.

\bibitem[Patil et~al.(2024)Patil, Du, and Tibshirani]{PDT24}
Pratik Patil, Jin{-}Hong Du, and Ryan~J. Tibshirani.
\newblock Optimal ridge regularization for out-of-distribution prediction.
\newblock \emph{CoRR}, abs/2404.01233, 2024.

\bibitem[Perdomo et~al.(2020)Perdomo, Zrnic, Mendler{-}D{\"{u}}nner, and Hardt]{PZMH20}
Juan~C. Perdomo, Tijana Zrnic, Celestine Mendler{-}D{\"{u}}nner, and Moritz Hardt.
\newblock Performative prediction.
\newblock In \emph{Proceedings of the 37th International Conference on Machine Learning, {ICML} 2020, 13-18 July 2020, Virtual Event}, volume 119 of \emph{Proceedings of Machine Learning Research}, pages 7599--7609. {PMLR}, 2020.

\bibitem[Perrigo(2023)]{timegpt}
Billy Perrigo.
\newblock The new {AI}-powered {B}ing is threatening users. that’s no laughing matter.
\newblock \emph{Time Magazine}, 2023.
\newblock URL \url{https://time.com/6256529/bing-openai-chatgpt-danger-alignment/}.

\bibitem[Prüfer and Schottmüller(2021)]{PS21}
Jens Prüfer and Christoph Schottmüller.
\newblock Competing with big data.
\newblock \emph{The Journal of Industrial Economics}, 69\penalty0 (4):\penalty0 967--1008, 2021.

\bibitem[Raghunathan et~al.(2020)Raghunathan, Xie, Yang, Duchi, and Liang]{RXYDL20}
Aditi Raghunathan, Sang~Michael Xie, Fanny Yang, John~C. Duchi, and Percy Liang.
\newblock Understanding and mitigating the tradeoff between robustness and accuracy.
\newblock In \emph{Proceedings of the 37th International Conference on Machine Learning, {ICML} 2020, 13-18 July 2020, Virtual Event}, volume 119 of \emph{Proceedings of Machine Learning Research}, pages 7909--7919. {PMLR}, 2020.

\bibitem[Rathje et~al.(2021)Rathje, Bavel, and van~der Linden]{RBL21}
Steve Rathje, Jay J.~Van Bavel, and Sander van~der Linden.
\newblock Out-group animosity drives engagement on social media.
\newblock \emph{Proceedings of the National Academy of Sciences}, 118\penalty0 (26):\penalty0 e2024292118, 2021.

\bibitem[Rolf et~al.(2021)Rolf, Worledge, Recht, and Jordan]{RWRJ21}
Esther Rolf, Theodora~T. Worledge, Benjamin Recht, and Michael~I. Jordan.
\newblock Representation matters: Assessing the importance of subgroup allocations in training data.
\newblock In Marina Meila and Tong Zhang, editors, \emph{Proceedings of the 38th International Conference on Machine Learning, {ICML} 2021, 18-24 July 2021, Virtual Event}, volume 139 of \emph{Proceedings of Machine Learning Research}, pages 9040--9051. {PMLR}, 2021.

\bibitem[Sagawa et~al.(2020)Sagawa, Koh, Hashimoto, and Liang]{SKHL20}
Shiori Sagawa, Pang~Wei Koh, Tatsunori~B. Hashimoto, and Percy Liang.
\newblock Distributionally robust neural networks for group shifts: On the importance of regularization for worst-case generalization.
\newblock In \emph{8th International Conference on Learning Representations, {ICLR} 2020, Addis Ababa, Ethiopia, April 26-30, 2020}. OpenReview.net, 2020.

\bibitem[Segal and Whinston(2007)]{SW07}
Ilya Segal and Michael~D. Whinston.
\newblock Antitrust in innovative industries.
\newblock \emph{American Economic Review}, 97\penalty0 (5):\penalty0 1703–1730, December 2007.

\bibitem[Shekhtman and Dean(2024)]{SD24}
Eliot Shekhtman and Sarah Dean.
\newblock Strategic usage in a multi-learner setting.
\newblock In \emph{International Conference on Artificial Intelligence and Statistics, 2-4 May 2024, Palau de Congressos, Valencia, Spain}, volume 238 of \emph{Proceedings of Machine Learning Research}, pages 2665--2673. {PMLR}, 2024.

\bibitem[Shen et~al.(2024)Shen, Raji, and Chen]{shen2024data}
Judy~Hanwen Shen, Inioluwa~Deborah Raji, and Irene~Y Chen.
\newblock The data addition dilemma.
\newblock \emph{arXiv preprint arXiv:2408.04154}, 2024.

\bibitem[Song et~al.(2024)Song, Bhattacharya, and Sur]{SBS24}
Yanke Song, Sohom Bhattacharya, and Pragya Sur.
\newblock Generalization error of min-norm interpolators in transfer learning.
\newblock \emph{CoRR}, abs/2406.13944, 2024.

\bibitem[{Statista}(2024)]{statista_facebook_users_2024}
{Statista}.
\newblock Leading countries based on facebook audience size as of january 2024, 2024.
\newblock URL \url{https://www.statista.com/statistics/268136/top-15-countries-based-on-number-of-facebook-users/#:~:text=With%20around%202.9%20billion%20monthly,most%20popular%20social%20media%20worldwide.}

\bibitem[{Stigler Committee}(2019)]{stiger19}
{Stigler Committee}.
\newblock Final report: Stigler committee on digital platforms.
\newblock available at \url{ https://www.chicagobooth.edu/-/media/research/stigler/pdfs/digital-platforms---committee-report---stigler-center.pdf,}, September 2019.

\bibitem[Su and Dean(2024)]{SuD24}
Jinyan Su and Sarah Dean.
\newblock Learning from streaming data when users choose.
\newblock \emph{CoRR}, abs/2406.01481, 2024.

\bibitem[{The White House}(2023)]{whitehouse_2023_aG_Exec_order}
{The White House}.
\newblock Executive order on the safe, secure, and trustworthy development and use of {A}rtificial {I}ntelligence, 2023.

\bibitem[Tirole(1988)]{T88}
Jean Tirole.
\newblock \emph{{The Theory of Industrial Organization}}, volume~1 of \emph{MIT Press Books}.
\newblock The MIT Press, December 1988.

\bibitem[Vipra and Korinek(2023)]{VK23}
Jai Vipra and Anton Korinek.
\newblock Market concentration implications of foundation models.
\newblock \emph{CoRR}, abs/2311.01550, 2023.

\bibitem[Wei(2024)]{weithesis}
Alexander Wei.
\newblock \emph{Learning and Decision-Making in Complex Environments}.
\newblock PhD thesis, EECS Department, University of California, Berkeley, May 2024.

\bibitem[Wei et~al.(2022)Wei, Hu, and Steinhardt]{WHS22}
Alexander Wei, Wei Hu, and Jacob Steinhardt.
\newblock More than a toy: Random matrix models predict how real-world neural representations generalize.
\newblock In \emph{International Conference on Machine Learning, {ICML} 2022, 17-23 July 2022, Baltimore, Maryland, {USA}}, volume 162 of \emph{Proceedings of Machine Learning Research}, pages 23549--23588. {PMLR}, 2022.

\bibitem[Wei et~al.(2023)Wei, Haghtalab, and Steinhardt]{WHS23}
Alexander Wei, Nika Haghtalab, and Jacob Steinhardt.
\newblock Jailbroken: How does {LLM} safety training fail?
\newblock In \emph{Advances in Neural Information Processing Systems 36: Annual Conference on Neural Information Processing Systems 2023, NeurIPS 2023, New Orleans, LA, USA, December 10 - 16, 2023}, 2023.

\bibitem[Xie et~al.(2023)Xie, Pham, Dong, Du, Liu, Lu, Liang, Le, Ma, and Yu]{XY23}
Sang~Michael Xie, Hieu Pham, Xuanyi Dong, Nan Du, Hanxiao Liu, Yifeng Lu, Percy Liang, Quoc~V. Le, Tengyu Ma, and Adams~Wei Yu.
\newblock Doremi: Optimizing data mixtures speeds up language model pretraining.
\newblock In \emph{Advances in Neural Information Processing Systems 36: Annual Conference on Neural Information Processing Systems 2023, NeurIPS 2023, New Orleans, LA, USA, December 10 - 16, 2023}, 2023.

\bibitem[Yang et~al.(2019)Yang, Liu, Chen, and Tong]{YLCT19}
Qiang Yang, Yang Liu, Tianjian Chen, and Yongxin Tong.
\newblock Federated machine learning: Concept and applications.
\newblock \emph{{ACM} Trans. Intell. Syst. Technol.}, 10\penalty0 (2):\penalty0 12:1--12:19, 2019.

\end{thebibliography}
\appendix

\newpage

\section{Proofs for Section \ref{sec:warmup}}\label{appendix:proofssec3}

In this section, we prove Theorem \ref{thm:tradeoffwarmup}. First, we state relevant facts (Appendix \ref{appendix:facts}) and prove intermediate lemmas (Appendix \ref{appendix:lemmas}), and then we use these ingredients to prove Theorem \ref{thm:tradeoffwarmup} (Appendix \ref{appendix:mainproof}). Throughout this section, we let 
\[L^*(\rho) = \mathbb{E}_{\DC}[(\beta_1 - \beta_2) \Sigma (\beta_1 - \beta_2)^T].\]
Moreover, let  
\[\beta(\alpha, \lambda) = \argmin_{\beta} \left(\alpha \cdot \mathbb{E}_{X \sim \DF}[(\langle \beta - \beta_1, X \rangle)^2] + (1-\alpha) \cdot \mathbb{E}_{X \sim \DF}[(\langle \beta - \beta_2, X \rangle)^2] + \lambda \|\beta\|_2^2 \right)\]  be the infinite-data ridge regression predictor.

\subsection{Facts}\label{appendix:facts}

We can explicitly solve for the infinite-data ridge regression predictor 
\begin{align*}
\beta(\alpha, \lambda) &=  \argmin_{\beta} \left(\alpha \cdot \mathbb{E}_{x \sim \DF}[\langle \beta - \beta_1, x\rangle^2] + (1-\alpha)\cdot  \mathbb{E}_{x \sim \DF}[\langle \beta - \beta_2, x\rangle^2] + \reg ||\beta||^2_2 \right) \\
&= \Sigma (\Sigma + \lambda I)^{-1} (\alpha \beta_1 + (1-\alpha) \beta_2).
\end{align*}
A simple calculation shows that $\mathbb{E}_{\DC}[L_1(\beta(\alpha, 0))]  = (1-\alpha)^2 L^*(\rho)$ and $\mathbb{E}_{\DC}[L_2(\beta(\alpha, 0))]  = \alpha^2 L^*(\rho)$. Thus, it holds that: 
\[\alpha \mathbb{E}_{\DC}[L_1(\beta(\alpha, 0))] + (1-\alpha) \mathbb{E}_{\DC}[L_2(\beta(\alpha, 0))] = \alpha (1-\alpha) L^*(\rho).\] Moreover, by the definition of the ridge regression objective, we see that:
\[
 \alpha \mathbb{E}_{\DC}[L_1(\beta(\alpha, \lambda))] + (1-\alpha) \mathbb{E}_{\DC}[L_2(\beta(\alpha, \lambda))] \ge \alpha \mathbb{E}_{\DC}[L_1(\beta(\alpha, 0))] + (1-\alpha) \mathbb{E}_{\DC}[L_2(\beta(\alpha, 0))].
\]

\subsection{Lemmas}\label{appendix:lemmas}

The first lemma upper bounds the performance loss when there is regularization.  
\begin{lemma}
\label{lemma:popparetofrontier}
Suppose that power-law scaling holds for the eigenvalues and alignment coefficients with scaling exponents $\gamma, \delta > 0$ and correlation coefficient $\rho \in [0, 1)$ and suppose that $P = \infty$. Let $L^*(\rho) = (\beta_1 - \beta_2)^T \Sigma (\beta_1 - \beta_2)^T$.  Let 
\[\beta(\alpha, \lambda) = \argmin_{\beta} \left(\alpha \cdot \mathbb{E}_{X \sim \DF}[(\langle \beta - \beta_1, X \rangle)^2] + (1-\alpha) \cdot \mathbb{E}_{X \sim \DF}[(\langle \beta - \beta_2, X \rangle)^2] + \lambda \|\beta\|_2^2\right)\]  be the infinite-data ridge regression predictor. Assume that $\alpha \ge 1/2$. Then it holds that
\[\mathbb{E}_{\DC}[L_1(\beta(\alpha, \lambda))] \ge (1-\alpha)^2 L^*(\rho) \]
and 
\[\frac{\mathbb{E}_{\DC}[L_1(\beta(\alpha, \lambda))]}{\mathbb{E}_{\DC}[L_2(\beta(\alpha, \lambda))]} \ge  \frac{(1-\alpha)^2}{\alpha^2}. \]
\end{lemma}
\begin{proof}
We define the quantities: 
\begin{align*}
  A &:= \lambda^2 \sum_{i=1}^P \frac{\lambda_i}{(\lambda_i + \lambda)^2} i^{-\delta} \\
  B &:= (1-\alpha)^2 (1-\rho)^2 \sum_{i=1}^P \frac{\lambda^3_i}{(\lambda_i + \lambda)^2} i^{-\delta} \\
  C &:= \lambda (1-\rho) \sum_{i=1}^P \frac{\lambda^2_i}{(\lambda_i + \lambda)^2} i^{-\delta}.
\end{align*}
We compute the performance loss as follows: 
\begin{align*}
&\mathbb{E}_{\DC}[L_1(\beta(\alpha, \lambda))] \\
&= \mathbb{E}_{\DC}[\Tr(\Sigma (\beta_1 - \beta(\alpha, \lambda)) (\beta_1 - \beta(\alpha, \lambda))^T) ] \\
&= \mathbb{E}_{\DC}\left[\Tr\left((\Sigma + \lambda I)^{-2} \Sigma \left( \lambda \beta_1 + \Sigma \cdot (1-\alpha) (\beta_1 - \beta_{2}) \right) \left( \lambda \beta_1 + \Sigma \cdot (1-\alpha)(\beta_1 - \beta_{2}) \right)^T\right)\right]  \\
&= \mathbb{E}_{\DC}\left[\Tr\left((\Sigma + \lambda I)^{-2} \Sigma \cdot \left( \lambda \beta_1 + \Sigma \cdot (1-\alpha) (\beta_1 - \beta_{2}) \right) \left( \lambda \beta_1 + \Sigma \cdot (1-\alpha) (\beta_1 - \beta_{2}) \right)^T\right)\right] \\
&= \lambda^2 \mathbb{E}_{\DC}\left[\Tr\left((\Sigma + \lambda I)^{-2} \Sigma \cdot \beta_1 \beta_1^T\right)\right] + (1-\alpha)^2 \mathbb{E}_{\DC}\left[\Tr\left((\Sigma + \lambda I)^{-2} \Sigma^3 \cdot (\beta_1 - \beta_{2}) (\beta_1 - \beta_{2})^T\right) \right] \\
&\ + \lambda (1-\alpha) \mathbb{E}_{\DC}\left[\Tr\left((\Sigma + \lambda I)^{-2} \Sigma^2 \cdot \beta_1 (\beta_1 - \beta_{2})^T \right)\right] \\
&= \lambda^2 \sum_{i=1}^P \frac{\lambda_i}{(\lambda_i + \lambda)^2} \mathbb{E}_{\DC}[\langle \beta_1, v_i \rangle^2] + (1-\alpha)^2 \sum_{i=1}^P \frac{\lambda^3_i}{(\lambda_i + \lambda)^2} \mathbb{E}_{\DC}[\langle \beta_1 - \beta_2, v_i \rangle^2]\\
&\ + \lambda (1-\alpha) \sum_{i=1}^P \frac{\lambda^2_i}{(\lambda_i + \lambda)^2} \mathbb{E}_{\DC}\left[\langle \beta_1, v_i \rangle \langle \beta_1 - \beta_{2}, v_i\rangle \right] \\
&= \lambda^2 \sum_{i=1}^P \frac{\lambda_i}{(\lambda_i + \lambda)^2} i^{-\delta} + (1-\alpha)^2 (1-\rho)^2 \sum_{i=1}^P \frac{\lambda^3_i}{(\lambda_i + \lambda)^2} i^{-\delta} + \lambda (1-\alpha) (1-\rho) \sum_{i=1}^P \frac{\lambda^2_i}{(\lambda_i + \lambda)^2} i^{-\delta} \\
&= A + (1-\alpha)^2 B + (1-\alpha) C.
\end{align*}
An analogous calculation shows that the safety violation can be written as:
\[ \mathbb{E}_{\DC}[L_2(\beta(\alpha, \lambda))] = A + \alpha^2 B + \alpha C\]
Since $\alpha \ge 1/2$, then it holds that:
\[\frac{\mathbb{E}_{\DC}[L_1(\beta(\alpha, \lambda))]}{\mathbb{E}_{\DC}[L_2(\beta(\alpha, \lambda))]} =  \frac{A + (1-\alpha)^2 B + (1-\alpha) C}{A + \alpha B + \alpha C} \ge  \frac{(1-\alpha)^2}{\alpha^2}. \]
Combining this with the facts from Appendix \ref{appendix:facts}---which imply that $\alpha \mathbb{E}_{\DC}[L_1(\beta(\alpha, \lambda))] + (1-\alpha) \mathbb{E}_{\DC}[L_2(\beta(\alpha, \lambda))] \ge \alpha (1-\alpha) L^*(\rho)$---we have that $\mathbb{E}_{\DC}[L_1(\beta(\alpha, \lambda))] \ge (1-\alpha)^2 L^*(\rho)$ as desired.  
\end{proof}

The following lemma computes the optimal values of $\mix$ and $\reg$ for the incumbent. 
\begin{lemma}
\label{lemma:ridgelessoptimalinfinitedata}
Suppose that power-law scaling holds for the eigenvalues and alignment coefficients with scaling exponents $\gamma, \delta > 0$ and correlation coefficient $\rho \in [0, 1)$ and suppose that $P = \infty$. Let $L^*(\rho) = \mathbb{E}_{\DC}[(\beta_1 - \beta_2)^T \Sigma (\beta_1 - \beta_2)^T]$. Suppose that $\Nlead = \infty$, and suppose that the safety constraint $\constrlead$ satisfies \eqref{eq:safetythreshold}. Then it holds that $\alpha_I =  \sqrt{\frac{\min(\constrlead, L^*(\rho))}{L^*(\rho)}}$, and $\reg_I = 0$ is optimal for the incumbent. Moreover, it holds that:
\[\mathbb{E}_{\DC}[L^*_1(\beta_1, \beta_2, \DF, \reg_I, \infty, \alpha_O)] = (\sqrt{L^*(\rho)} - \sqrt{\min(\constrlead, L^*(\rho)})^2. \]
\end{lemma}
\begin{proof}
First, we apply Lemma \ref{lemma:sollichmoreprecise} with $N = \infty$ to see that:
\[\mathbb{E}_{\DC}[L^*_1(\beta_1, \beta_2, \DF, \reg, \infty, \alpha)] = \mathbb{E}_{\DC}[L_1(\beta(\alpha, \lambda))] \]
and apply the definition of $L_2^*$ to see that:
\[\mathbb{E}_{\DC}[L^*_2(\beta_1, \beta_2, \DF, \alpha)] = \mathbb{E}_{\DC}[L_2(\beta(\alpha, 0))]. \]

Let $\alpha^* =\sqrt{\frac{\min(\constrlead, L^*(\rho))}{L^*(\rho)}}$. By the assumption in the lemma statement, we know that:
\[\alpha^* \ge \sqrt{\frac{\mathbb{E}_{\DC}[\LossAlign^*(\beta_1, \beta_2, \DF, 0.5)]}{L^*(\rho)}} = 0.5. \]

We show that $(\alpha_I, \reg_I) = (\alpha^*, 0)$. Assume for sake of contradiction that $(\alpha, \reg) \neq (\alpha^*, 0)$ satisfies the safety constraint $\mathbb{E}_{\DC}[\LossAlign^*(\beta_1, \beta_2, \DF, \alpha)] \le \constrlead$ and achieves strictly better performance loss:
\[\mathbb{E}_{\DC}[\LossPerf^*(\beta_1, \beta_2, \DF, \reg, \infty, \mix)] <  \mathbb{E}_{\DC}[\LossPerf^*(\beta_1, \beta_2, \DF, 0, \infty, \mix^*)].\]
We split into two cases: $\alpha^* = \alpha, \reg \neq 0$ and $\alpha^* \neq \alpha$. 

\paragraph{Case 1: $\alpha^* = \alpha$, $\reg \neq 0$.} 
By Lemma \ref{lemma:popparetofrontier}, we know that 
\[\mathbb{E}[L^*_1(\beta_1, \beta_2, \DF, \reg, \infty, \alpha^*)] = \mathbb{E}_{\DC}[L_1(\beta(\alpha^*, \lambda))] \ge (1-\alpha^*)^2 L^*(\rho).\]
Equality is obtained at $\reg = 0$, which is a contradiction. 

\paragraph{Case 2: $\alpha \neq \alpha^*$.} 
By Lemma \ref{lemma:popparetofrontier}, it must hold that $\alpha > \alpha^*$ in order for the performance to beat that of $(\alpha^*, 0)$. However, this means that the safety constraint 
\[\mathbb{E}_{\DC}[\LossAlign^*(\beta_1, \beta_2, \DF, \alpha)] = \alpha^2 L^*(\rho) > (\alpha^*)^2 L^*(\rho) = \constrlead \]
is violated, which is a contradiction. 

\paragraph{Concluding the statement.} This means that$(\alpha_I, \reg_I) = (\alpha^*, 0)$, which also means that:
\begin{align*}
    \mathbb{E}_{\DC}[L^*_1(\beta_1, \beta_2, \DF, \reg_I, \infty, \alpha_I)] &= \mathbb{E}_{\DC}[L_1(\beta(\alpha_I, \lambda_I))]\\
    &= (1-\alpha_I)^2  \mathbb{E}_{\DC}[(\beta_1 - \beta_2)^T \Sigma (\beta_1 - \beta_2)] \\
    &= \left(\sqrt{L^*(\rho)} - \sqrt{\min(\constrlead, L^*(\rho)}\right)^2.
\end{align*}
    
\end{proof}

The following claim calculates $\mathbb{E}_{\DC}[(\beta_1 - \beta_2)^T \Sigma (\beta_1 - \beta_2)]$. 
\begin{claim}
\label{claim:boundLstar}
Suppose that the power-law scaling assumption holds for the eigenvalues and alignment coefficients with scaling exponents $\gamma, \delta > 0$ and correlation coefficient $\rho \in [0, 1)$, suppose that $P = \infty$. Then it holds that: 
\[\mathbb{E}_{\DC}[(\beta_1 - \beta_2)^T \Sigma (\beta_1 - \beta_2)]  = 2 (1-\rho) \left(\sum_{i=1}^P i^{-\delta-1-\gamma}\right) = \Theta(1-\rho). \]
\end{claim}
\begin{proof}
Let $\Sigma = V \Lambda V^T$ be the eigendecomposition of $\Sigma$, where $\Lambda$ is a diagonal matrix consisting of the eigenvalues. We observe that 
\[\mathbb{E}_{\DC}[\langle \beta_1 - \beta_2, v_i\rangle^2] = \mathbb{E}_{\DC}[\langle \beta_1 , v_i\rangle^2] + \mathbb{E}_{\DC}[\langle \beta_2, v_i\rangle^2] - 2 \mathbb{E}_{\DC}[\langle \beta_1 , v_i\rangle \langle \beta_2, v_i\rangle]  = i^{-\delta} + i^{-\delta} - 2 \rho i^{-\delta} = 2(1-\rho) i^{-\delta}.\]

This means that: 
\begin{align*}
 \mathbb{E}_{\DC}[(\beta_1 - \beta_2)^T \Sigma (\beta_1 - \beta_2)] &= \Tr(\Sigma \mathbb{E}_{\DC}[(\beta_1 - \beta_2) (\beta_1 - \beta_2)^T]) \\
 &= \Tr(\Lambda\mathbb{E}_{\DC}[V^T (\beta_1 - \beta_2) (\beta_1 - \beta_2)^T V] ) \\
 &= \sum_{i=1}^P  i^{-1-\gamma} \mathbb{E}_{\DC}[\langle \beta_1 - \beta_2, v_i\rangle^2]\\
 &= 2 (1 - \rho) \sum_{i=1}^P i^{-\delta-1-\gamma} \\
 &= \Theta(1-\rho).
\end{align*}
\end{proof}

\subsection{Proof of Theorem \ref{thm:tradeoffwarmup}}\label{appendix:mainproof}

We prove Theorem \ref{thm:tradeoffwarmup} using the above lemmas along with Corollary \ref{cor:scalinglawoptreg} (the proof of which we defer to Appendix \ref{appendix:proofsmultiobjective}). 

\begin{proof}[Proof of Theorem \ref{thm:tradeoffwarmup}]

We analyze $(\alpha_C, \reg_C)$ first for the incumbent $C = I$ and then for the entrant $C = E$. 

\paragraph{Analysis of the incumbent $C = I$.}
To compute $\alpha_I$ and $\reg_I$, we apply Lemma \ref{lemma:ridgelessoptimalinfinitedata}. 
By Lemma \ref{lemma:ridgelessoptimalinfinitedata}, we see that:
\[\mathbb{E}_{\DC}[L^*_1(\beta_1, \beta_2, \DF, \reg_I, \infty, \alpha_I)] = \left(\sqrt{L^*(\rho)} - \sqrt{\min(\constrlead, L^*(\rho)}\right)^2. \]

\paragraph{Analysis of the entrant $C = E$.} Since the entrant faces no safety constraint, the entrant can choose any $\alpha \in [0.5, 1]$. We apply Corollary \ref{cor:scalinglawoptreg} to see that:
\[ \mathbb{E}_{\DC}[L^*_1(\beta_1, \beta_2, \DF, \reg_E, N, \alpha_E)] = \inf_{\alpha \in [0.5, 1]} \inf_{\reg > 0} \mathbb{E}_{\DC}[L^*_1(\beta_1, \beta_2, \DF, \reg, N, \alpha)] = \Theta\left(
N^{-\scalingexp} \right), \]
which means that:
\[ \Nentr^*(\infty, \constrlead, \infty, \DC, \DF) = \Theta \left(\left(\sqrt{L^*(\rho)} - \sqrt{\min(\constrlead, L^*(\rho)}\right)^{-2/\scalingexp} \right)\]
as desired. 
We can further apply Claim \ref{claim:boundLstar} to see that $L^*(\rho) = \Theta(1-\rho)$. 
\end{proof}

\section{Proofs for Section \ref{sec:general}}\label{appendix:proofssecgeneral}

\subsection{Proofs for Section \ref{subsec:finitedata}}\label{appendix:proofssubsec41}

We prove Theorem \ref{thm:finitedata}. The main technical tool is Theorem \ref{thm:scalinglaw}, the proof of which we defer to Appendix \ref{appendix:proofsmultiobjective}. 

\begin{proof}[Proof of Theorem \ref{thm:finitedata}]

We analyze $(\alpha_C, \reg_C)$ first for the incumbent $C = I$ and then for the entrant $C = E$. Like in the theorem statement, let $L^*(\rho) = \mathbb{E}_{\DC}[(\beta_1 - \beta_2)^T \Sigma (\beta_1 - \beta_2)] = \Theta(1 - \rho)$ (Claim \ref{claim:boundLstar}) and $G_I := (\sqrt{L^*(\rho)} - \sqrt{\min(\constrlead, L^*(\rho))})^2$, and $\scalingexp = \min(2(1+\gamma), \delta + \gamma)$.

\paragraph{Analysis of the incumbent $C = I$.} Recall from the facts in Appendix \ref{appendix:facts} that:
\[L_1^*(\beta_1, \beta_2, \DF, \alpha) = \alpha^2 L^*(\rho). \]
This means that the safety constraint is satisfied if and only if $\mix_I \le \sqrt{\frac{\min(\constrlead, L^*(\rho))}{L^*(\rho)}} =: \alpha^*$. The bound in Corollary \ref{cor:scalinglawoptreg} implies that:
\begin{align*}
  &\mathbb{E}_{\DC}[L^*_1(\beta_1, \beta_2, \DF, \reg_I, \Nlead, \alpha_I)] \\
  &= \inf_{\alpha \in \left[0.5, \alpha^* \right]} \inf_{\reg > 0} \mathbb{E}_{\DC}[L^*_1(\beta_1, \beta_2, \DF, \reg,  \Nlead, \alpha)]  \\
  &= \Theta\left(
\inf_{\reg > 0} \mathbb{E}_{\DC}\left[L^*_1\left(\beta_1, \beta_2, \Sigma, \reg,  \Nlead, \alpha^* \right)\right] \right) \\
&= \begin{cases}
 \Theta\left(\Nlead^{-\scalingexp}\right) &\text{ if } \Nlead \le (1-\alpha^* )^{-\frac{1}{\scalingexp}}(1-\rho)^{-\frac{1}{\scalingexp}} \\
 \Theta\left(\left(\frac{\Nlead}{(1-\alpha^* )(1-\rho)}\right)^{-\frac{\scalingexp}{\scalingexp + 1}}\right) &\text{ if } (1-\alpha^* )^{-\frac{1}{\scalingexp}}(1-\rho)^{-\frac{1}{\scalingexp}} 
 \le \Nlead \le  (1-\alpha^* )^{-\frac{2+\scalingexp}{\scalingexp}} (1-\rho)^{-\frac{1}{\scalingexp}}
\\
\Theta((1-\alpha^* )^2(1-\rho)) &\text{ if } \Nlead \ge (1-\alpha^* )^{-\frac{2+\scalingexp}{\scalingexp}} (1-\rho)^{-\frac{1}{\scalingexp}},
\end{cases}\\
&= \begin{cases}
\Theta\left( \Nlead^{-\scalingexp} \right) &\text{ if }  \Nlead \le G_I^{-\frac{1}{2\scalingexp}} (1-\rho)^{-\frac{1}{2\scalingexp}}  \\
\Theta\left(\Nlead^{-\frac{\scalingexp}{\scalingexp+1}} \cdot G_I^{\frac{\scalingexp}{2(\scalingexp+1)}} (1-\rho)^{\frac{\scalingexp}{2(\scalingexp+1)}}\right)  &\text{ if } G_I^{-\frac{1}{2\scalingexp}} (1-\rho)^{-\frac{1}{2\scalingexp}} \le \Nlead  \le G_I^{-\frac{1}{2} - \frac{1}{\scalingexp}}(1-\rho)^{\frac{1}{2}}  \\
\Theta\left(G_I\right)  &\text{ if }  \Nlead \ge G_I^{-\frac{1}{2} - \frac{1}{\scalingexp}}(1-\rho)^{\frac{1}{2}}
\end{cases}.
\end{align*}

\paragraph{Analysis of the entrant $C = E$.} Since the entrant faces no safety constraint, the entrant can choose any $\alpha \in [0.5, 1]$.  We apply Corollary \ref{thm:scalinglaw} to see that:
\[\mathbb{E}_{\DC}[L^*_1(\beta_1, \beta_2, \DF, \reg_E, N, \alpha_E)] = \inf_{\alpha \in [0.5, 1]} \inf_{\reg > 0} \mathbb{E}_{\DC}[L^*_1(\beta_1, \beta_2, \DF, \reg, N, \alpha)] = \Theta\left(
N^{-\scalingexp} \right), \]
which means that:
\[ \Nentr^*(\Nlead, \constrlead, \infty, \DC, \DF) = \begin{cases}
\Theta\left(\Nlead\right) &\text{ if }  \Nlead \le G_I^{-\frac{1}{2\scalingexp}} (1-\rho)^{-\frac{1}{2\scalingexp}}  \\
\Theta\left(\Nlead^{\frac{1}{\scalingexp+1}} \cdot G_I^{-\frac{1}{2(\scalingexp+1)}} (1-\rho)^{-\frac{1}{2(\scalingexp+1)}}\right)  &\text{ if } G_I^{-\frac{1}{2\scalingexp}} (1-\rho)^{-\frac{1}{2\scalingexp}} \le \Nlead  \le G_I^{-\frac{1}{2} - \frac{1}{\scalingexp}}(1-\rho)^{\frac{1}{2}}  \\
\Theta\left(G_I^{-\frac{1}{\scalingexp}}\right)  &\text{ if }  \Nlead \ge G_I^{-\frac{1}{2} - \frac{1}{\scalingexp}}(1-\rho)^{\frac{1}{2}}.
\end{cases}\]
as desired.

\end{proof}

\subsection{Proofs for Section \ref{subsec:alignment}}\label{appendix:proofssubsec42}

We prove Theorem \ref{thm:alignment}. When the the safety constraints of the two firms are sufficiently close, it no longer suffices to analyze the loss up to constants for the entrant, and we require a more fine-grained analysis of the error terms than is provided in the scaling laws in Corollary \ref{cor:scalinglawoptreg}. In this case, we turn to scaling laws for the \textit{excess loss} as given by Corollary \ref{cor:scalinglawoptregexcess}.

\begin{proof}[Proof of Theorem \ref{thm:alignment}]

We analyze $(\alpha_C, \reg_C)$ first for the incumbent $C = I$ and then for the entrant $C = E$. Like in the theorem statement, let $L^*(\rho) = \mathbb{E}_{\DC}[(\beta_1 - \beta_2)^T \Sigma (\beta_1 - \beta_2)] = \Theta(1 - \rho)$ (Claim \ref{claim:boundLstar}), $G_I = (\sqrt{L^*(\rho)} - \sqrt{\min(\constrlead, L^*(\rho))})^2$, $G_E = (\sqrt{L^*(\rho)} - \sqrt{\min(\constrentr, L^*(\rho))})^2$,
$D = G_I - G_E$, and $\scalingexp = \min(2(1+\gamma), \delta + \gamma)$.

\paragraph{Analysis of the incumbent $C = I$.} Since the incumbent has infinite data, we apply Lemma \ref{lemma:ridgelessoptimalinfinitedata} to see that: 
\begin{align*}
\mathbb{E}_{\DC}[L^*_1(\beta_1, \beta_2, \DF, \reg_I, \infty, \mix_I)] &= \left(\sqrt{L^*(\rho)} - \sqrt{\min(\constrlead, L^*(\rho))} \right)^2 \\
&= D + G_E.
\end{align*}

\paragraph{Analysis of the entrant $C = E$.} Recall from the facts in Appendix \ref{appendix:facts} that:
\[L_1^*(\beta_1, \beta_2, \DF, \alpha) = \alpha^2 L^*(\rho). \]
This means that the safety constraint is satisfied if and only if $\mixentr \le \sqrt{\frac{\min(\constrentr, L^*(\rho))}{L^*(\rho)}} =: \alpha^*$. The bound in Corollary \ref{cor:scalinglawoptregexcess} implies that:
\begin{align*}
  &\mathbb{E}_{\DC}[L^*_1(\beta_1, \beta_2, \DF, \reg_E, N, \alpha_E)] \\
  &= \inf_{\alpha \in \left[0.5, \alpha^* \right]} \inf_{\reg > 0} \mathbb{E}_{\DC}[L^*_1(\beta_1, \beta_2, \DF, \reg,  N, \alpha)]  \\
    &= \inf_{\alpha \in \left[0.5, \alpha^* \right]} \left(\inf_{\reg > 0} \left(\mathbb{E}_{\DC}[L^*_1(\beta_1, \beta_2, \DF, \reg,  N, \alpha) - L_1(\beta(\alpha, 0))]\right) + \mathbb{E}_{\DC}[L_1(\beta(\alpha, 0))] \right) \\
    &= \inf_{\alpha \in \left[0.5, \alpha^* \right]} \left(\inf_{\reg > 0} \left(\mathbb{E}_{\DC}[L^*_1(\beta_1, \beta_2, \DF, \reg,  N, \alpha) - L_1(\beta(\alpha, 0))]\right) + (1-\alpha)^2 L^*(\rho) \right) \\
  &= \Theta\left(
\inf_{\reg > 0} \left(\mathbb{E}_{\DC}[L^*_1(\beta_1, \beta_2, \DF, \reg,  N, \alpha) - L_1(\beta(\alpha^*, 0))]\right) \right) +  (1-\alpha^*)^2 L^*(\rho)  \\
&= \begin{cases}
 (1-\alpha^*)^2 L^*(\rho)  + \Theta\left(N^{-\scalingexp}\right) &\text{ if } N \le (1-\alpha^*)^{-\frac{1}{\scalingexp}}(1-\rho)^{-\frac{1}{\scalingexp}} \\
 (1-\alpha^*)^2 L^*(\rho)  + \Theta\left(\left(\frac{N}{(1-\alpha^*)(1-\rho)}\right)^{-\frac{\scalingexp}{\scalingexp + 1}}\right) &\text{ if } 
(1-\alpha^*)^{-\frac{1}{\scalingexp}}(1-\rho)^{-\frac{1}{\scalingexp}} \le N \le  (1-\alpha^*)^{-\frac{\scalingexp'+1}{\scalingexp - \scalingexp'}}(1-\rho)^{-\frac{\scalingexp'+1}{\scalingexp - \scalingexp'}}
\\
(1-\alpha^*)^2 L^*(\rho)  + \Theta\left((1-\alpha^*) (1-\rho) N^{-\frac{\scalingexp'}{\scalingexp'+1}} \right) &\text{ if } N \ge (1-\alpha^*)^{-\frac{\scalingexp'+1}{\scalingexp - \scalingexp'}}(1-\rho)^{-\frac{\scalingexp'+1}{\scalingexp - \scalingexp'}},
\end{cases} \\
&= \begin{cases}
G_E + \Theta\left(N^{-\scalingexp}\right) &\text{ if } N \le (1-\alpha^*)^{-\frac{1}{\scalingexp}}(1-\rho)^{-\frac{1}{\scalingexp}} \\
G_E  + \Theta\left(\left(\frac{N}{(1-\alpha^*)(1-\rho)}\right)^{-\frac{\scalingexp}{\scalingexp + 1}}\right) &\text{ if } 
(1-\alpha^*)^{-\frac{1}{\scalingexp}}(1-\rho)^{-\frac{1}{\scalingexp}} \le N \le  (1-\alpha^*)^{-\frac{\scalingexp'+1}{\scalingexp - \scalingexp'}}(1-\rho)^{-\frac{\scalingexp'+1}{\scalingexp - \scalingexp'}}
\\
G_E + \Theta\left((1-\alpha^*) (1-\rho) N^{-\frac{\scalingexp'}{\scalingexp'+1}} \right) &\text{ if } N \ge (1-\alpha^*)^{-\frac{\scalingexp'+1}{\scalingexp - \scalingexp'}}(1-\rho)^{-\frac{\scalingexp'+1}{\scalingexp - \scalingexp'}},
\end{cases}.
\end{align*}

Using this, we can compute the market-entry threshold as follows: 
\begin{align*}
&\Nentropt(\infty, \constrlead, \constrentr, \DC, \DF) \\
&= \begin{cases}
\Theta(D^{-\frac{1}{\scalingexp}} ) &\text{ if } D \ge (1-\alpha^*) (1-\rho) \\
\Theta\left(D^{-\frac{\scalingexp+1} {\scalingexp}} (1-\alpha^*) (1-\rho) \right) &\text{ if } 
(1-\alpha^*)^{\frac{\scalingexp}{\scalingexp - \scalingexp'}}(1-\rho)^{\frac{\scalingexp}{\scalingexp - \scalingexp'}} \le D \le  (1-\alpha^*) (1-\rho)  
\\
\Theta\left(\left(\frac{D}{(1 - \alpha^*)(1-\rho)} \right)^{-\frac{\scalingexp'+1}{\scalingexp'}} \right) &\text{ if } D \le (1-\alpha^*)^{\frac{\scalingexp}{\scalingexp - \scalingexp'}}(1-\rho)^{\frac{\scalingexp}{\scalingexp - \scalingexp'}} 
\end{cases} \\
&= \begin{cases}
\Theta(D^{-\frac{1}{\scalingexp}} ) &\text{ if } D \ge G_E^{\frac{1}{2}} (1-\rho)^{\frac{1}{2}} \\
\Theta\left(D^{-\frac{\scalingexp+1} {\scalingexp}} G_E^{\frac{1}{2}} (1-\rho)^{\frac{1}{2}}  \right) &\text{ if } 
G_E^{\frac{\scalingexp}{2(\scalingexp - \scalingexp')}}  (1-\rho)^{\frac{\scalingexp}{2(\scalingexp - \scalingexp')}}  \le D \le  G_E^{\frac{1}{2}} (1-\rho)^{\frac{1}{2}} 
\\
\Theta\left(\left(\frac{D}{G_E^{\frac{1}{2}} (1-\rho)^{\frac{1}{2}}} \right)^{-\frac{\scalingexp'+1}{\scalingexp'}} \right) &\text{ if } D \le G_E^{\frac{\scalingexp}{2(\scalingexp - \scalingexp')}}  (1-\rho)^{\frac{\scalingexp}{2(\scalingexp - \scalingexp')}} 
\end{cases}
\end{align*} 
\end{proof}

\section{Proofs for Section \ref{sec:scalinglaws}}\label{appendix:proofsmultiobjective}

In this section, we derive a deterministic equivalent and scaling laws for high-dimensional multi-objective linear regression. Before diving into this, we introduce notation, derive a basic decomposition, and give an outline for the remainder of the section.

\paragraph{Notation.} Recall that $(X_i, Y_i)$ denotes the labelled training dataset. Let the sample covariance be: 
\[\hat{\Sigma} = \frac{1}{N} \sum_{i=1}^N X_i X_i^T.\] 
We also consider the following reparameterization where we group together inputs according to how they are labelled. For $j \in \left\{1,2\right\}$, we let $X_{1,j}, \ldots, X_{N_j,j}$ be the inputs labelled by $\param_j$. 
We let 
\begin{align*}
 \hat{\Sigma}_1 &= \frac{1}{N_1} \sum_{i=1}^{N_1} X_{i, 1} X_{i,1}^T \\
 \hat{\Sigma}_2 &= \frac{1}{N_2} \sum_{i=1}^{N_2} X_{i, 2} X_{i,2}^T. 
\end{align*}
It is easy to see that $\Sigma = \alpha \hat{\Sigma}_1 + (1-\alpha) \hat{\Sigma}_2$. Moreover, $\mathbb{E}[\hat{\Sigma}] = \mathbb{E}[\hat{\Sigma}_1] = \mathbb{E}[\hat{\Sigma}_2] = \Sigma$. Furthermore, $\hat{\Sigma}_1$ and $\hat{\Sigma}_2$ are fully independent. We let $\sim$ denote asymptotic equivalence following \citet{bach}. 

\paragraph{Basic decomposition.} A simple calculation shows that the solution and population-level loss of ridge regression takes the following form.  
\begin{claim}
\label{claim:finitesample}
Assume the notation above. Let $\Matrixhom{1} = \beta_1 \beta_1^T$, let $\Matrixdiff = (\beta_1 - \beta_2)(\beta_1 - \beta_2)^T$, and let $\Matrixmixed{1} = (\beta_1 - \beta_2) \beta_1^T$. 
The learned predictor takes the form:
\[\hat{\beta}(\alpha, \lambda, X) = (\hat{\Sigma} + \lambda I)^{-1} (\alpha \hat{\Sigma}_1 \beta_1 + (1-\alpha) \hat{\Sigma}_2 \beta_2). \]
Moreover, it holds that: 
\begin{align*}
L_1(\hat{\beta}(\alpha, \lambda, X)) &= \underbrace{\lambda^2 \Tr((\hat{\Sigma} + \lambda I)^{-1} \Sigma (\hat{\Sigma} + \lambda I)^{-1} \Matrixhom{1})}_{(T1)} + \underbrace{(1-\alpha)^2 \Tr(\hat{\Sigma}_2 (\hat{\Sigma} + \lambda I)^{-1} \Sigma (\hat{\Sigma} + \lambda I)^{-1} \hat{\Sigma}_2 \Matrixdiff)}_{(T2)} \\
&+ \underbrace{2 \lambda (1-\alpha) \cdot \Tr((\hat{\Sigma} + \lambda I)^{-1} \Sigma (\hat{\Sigma} + \lambda I)^{-1}  \hat{\Sigma}_2 \Matrixmixed{1})}_{(T3)}.
\end{align*}
\end{claim}
\begin{proof}
For $1 \le i \le N$, let $Y_i$ be the label for input $X_i$ in the training dataset. For $i \in \left\{1, 2\right\}$ and $1 \le i \le N_i$, let $Y_{i,j} := \langle \beta_i, X_{i, j}\rangle$ be the label for the input $X_{i, j}$ according to $\beta_i$. 

For the first part, it follows from standard analyses of ridge regression that the learned predictor takes the form:
\begin{align*}
\hat{\beta}(\alpha, \lambda, X) &= (\hat{\Sigma} + \lambda I)^{-1} \left(\frac{1}{N} \sum_{i=1}^N X_i Y_i \right) \\
&= (\hat{\Sigma} + \lambda I)^{-1} \left(\frac{1}{N}  \sum_{i=1}^N X_{i, 1} Y_{i,1} + \frac{1}{N}  \sum_{i=1}^N X_{i, 2} Y_{i,2} \right) \\ 
&= (\hat{\Sigma} + \lambda I)^{-1} \left(\alpha \hat{\Sigma}_1 \beta_1 + (1-\alpha)\hat{\Sigma}_2 \beta_2\right) 
\end{align*}
as desired. 

For the second part, we first observe that the difference $\beta_1 - \hat{\beta}(\alpha, \lambda, X)$ takes the form:
\begin{align*}
 \beta_1 - \hat{\beta}(\alpha, \lambda, X) &= \beta_1 -  (\hat{\Sigma} + \lambda I)^{-1} \left(\alpha \hat{\Sigma}_1 \beta_1 + (1-\alpha)\hat{\Sigma}_2 \beta_2\right)  \\
 &= (\hat{\Sigma} + \lambda I)^{-1} \left(\lambda \beta_1 + (1-\alpha) \hat{\Sigma}_2 (\beta_1 - \beta_2) \right). 
\end{align*}
This means that:
\begin{align*}
&L_1(\hat{\beta}(\alpha, \lambda, X))\\
&= (\beta_1 - \hat{\beta}(\alpha, \lambda, X)^T \Sigma  (\beta_1 - \hat{\beta}(\alpha, \lambda, X) \\
&=  \left(\lambda \beta_1 + (1-\alpha) \hat{\Sigma}_2 (\beta_1 - \beta_2) \right)^T (\hat{\Sigma} + \lambda I)^{-1} \Sigma (\hat{\Sigma} + \lambda I)^{-1} \left(\lambda \beta_1 + (1-\alpha) \hat{\Sigma}_2 (\beta_1 - \beta_2) \right) \\
&= \lambda^2 \cdot \beta_1^T \hat{\Sigma} + \lambda I)^{-1} \Sigma (\hat{\Sigma} + \lambda I)^{-1}\beta_1 + (1-\alpha)^2 \cdot (\beta_1 - \beta_2)^T \hat{\Sigma}_2 \hat{\Sigma} + \lambda I)^{-1} \Sigma (\hat{\Sigma} + \lambda I)^{-1}  \hat{\Sigma}_2 (\beta_1 - \beta_{2})  \\
&\ + 2 \lambda (1-\alpha) \cdot \beta_1^T \hat{\Sigma} + \lambda I)^{-1} \Sigma (\hat{\Sigma} + \lambda I)^{-1} \hat{\Sigma}_2 (\beta_1 - \beta_2) \\
&= \lambda^2 \Tr((\hat{\Sigma} + \lambda I)^{-1} \Sigma (\hat{\Sigma} + \lambda I)^{-1} \Matrixhom{1})+ (1-\alpha)^2 \Tr(\hat{\Sigma}_2 (\hat{\Sigma} + \lambda I)^{-1} \Sigma (\hat{\Sigma} + \lambda I)^{-1} \hat{\Sigma}_2 \Matrixdiff) \\
&\ + 2 \lambda (1-\alpha) \cdot \Tr((\hat{\Sigma} + \lambda I)^{-1} \Sigma (\hat{\Sigma} + \lambda I)^{-1}  \hat{\Sigma}_2 \Matrixmixed{1}).
\end{align*}
as desired. 
\end{proof}

\paragraph{Outline for the rest of this Appendix.} The bulk of our analysis in this section boils down to analyzing Term 1 (T1), Term 2 (T2), and Term 3 (T3) in Claim \ref{claim:finitesample}. Our main technical tool is the random matrix machinery from Appendix \ref{appendix:machinery}. In Appendix \ref{appendix:usefulsublemmas}, we provide useful sublemmas about intermediate deterministic equivalents that we apply to analyze Terms 2 and 3. We then analyze Term 1 (Appendix \ref{appendix:term1}), Term 2 (Appendix \ref{appendix:term2}), and Term 3 (Appendix \ref{appendix:term3}), and use this to prove Lemma \ref{lemma:sollichvariant} (Appendix \ref{appendix:lemmasollichvariant}). 

We apply the power scaling assumptions to derive a simpler expression for the deterministic equivalent  (Lemma \ref{lemma:sollichmoreprecise} in Appendix \ref{appendix:analysisterms}). We then apply Lemma \ref{lemma:sollichmoreprecise} to prove Theorem \ref{thm:scalinglaw} (Appendix \ref{appendix:proofscaling}), and we prove Corollary \ref{cor:scalinglawoptreg} (Appendix \ref{appendix:proofcor}). We also apply Lemma \ref{lemma:sollichmoreprecise} to prove Theorem \ref{thm:scalinglawexcess} (Appendix \ref{appendix:proofscalingexcess}), and we prove Corollary \ref{cor:scalinglawoptregexcess} (Appendix \ref{appendix:proofscalingoptregexcess}). We defer auxiliary calculations to Appendix \ref{appendix:auxiliarycalculations}.

\subsection{Useful lemmas about intermediate deterministic equivalents}\label{appendix:usefulsublemmas}

The results in this section consider $Z_1 := \frac{\alpha}{1-\alpha} \hat{\Sigma}_1 + \frac{\reg}{1-\alpha} I$, which we introduce when conditioning on the randomness of $\hat{\Sigma}_1$ when analyzing (T2) and (T3). We derive several properties of $Z_1$ and the effective regularizer $\kappa_1 = \kappa(1, N(1-\alpha), Z_1^{-1/2} \Sigma Z_1^{-1/2})$ below.

The first set of lemmas relate the trace of various matrices involving $\kappa_1$ and $Z_1$ to deterministic quantities. A subtlety is that $\kappa_1$ and $Z_1$ are correlated, so we cannot directly apply Marčenko-Pastur, and instead we must indirectly analyze this quantity. 
\begin{lemma}
\label{lemma:equivalencelinear}
Consider the setup of Lemma \ref{lemma:sollichvariant}, and assume the notation above. Assume $\alpha < 1$. Let $Z_1 = \frac{\alpha}{1-\alpha} \hat{\Sigma}_1 + \frac{\reg}{1-\alpha} I$, and let $\kappa_1 = \kappa(1, N(1-\alpha), Z_1^{-1/2} \Sigma Z_1^{-1/2})$. Suppose that  $B$ has bounded operator norm. 
\[\kappa_1 \Tr\left((\Sigma + \kappa_1 Z_1)^{-1}  B   \right) \sim \frac{(1-\alpha)  \kappa}{\lambda} \Tr\left((\Sigma + \kappa I)^{-1} B \right) \]  
\end{lemma}
\begin{proof} 
By Claim \ref{claim:z1rearrange}, we know that: 
\begin{align*}
(1-\alpha) \Tr\left(\left(\hat{\Sigma} + \lambda I\right)^{-1} B \right) &= \Tr\left(\left(\hat{\Sigma}_2  + Z_1 \right)^{-1} B \right) \\
&\sim_{(A)} \kappa_1 \Tr\left(\left(\Sigma  + \kappa_1 I \right)^{-1} B \right).
\end{align*}
where (A) applies Lemma \ref{lemma:bach} and Claim \ref{claim:boundedoperatornorm}. 

Furthermore, by Lemma \ref{lemma:bach}, it holds that: 
 \[
\lambda\Tr\left(\left(\hat{\Sigma} + \lambda I\right)^{-1} B \right) \sim \kappa \Tr\left(\left(\Sigma + \kappa I \right)^{-1} B \right) . 
 \]

 Putting this all together yields the desired result.
\end{proof}

\begin{lemma}
\label{lemma:equivalencequadratic}
Consider the setup of Lemma \ref{lemma:sollichvariant}, and assume the notation above.  Assume $\alpha < 1$. Let $Z_1 = \frac{\alpha}{1-\alpha} \hat{\Sigma}_1 + \frac{\reg}{1-\alpha} I$, and let $\kappa_1 = \kappa(1, N(1-\alpha), Z_1^{-1/2} \Sigma Z_1^{-1/2})$. Suppose that $A$ and $B$ have bounded operator norm. Then it holds that:
 \begin{align*}
 &(\kappa_1)^2 \left(\Tr\left(\left(\Sigma + \kappa_1 Z_1 \right)^{-1} A \left(\Sigma + \kappa_1 Z_1\right)^{-1} B \right) + E_1  \right)
 &\sim  \frac{(1-\alpha)^2 \kappa^2}{ \lambda^2}   \left(\Tr\left(\left(\Sigma + \kappa I \right)^{-1} A \left(\Sigma + \lambda I \right)^{-1} B \right) + E_2 \right)   
 \end{align*}
 where 
 \begin{align*}
 \kappa &= \kappa(\lambda, N, \Sigma) \\
  E_1 &=    \frac{\frac{1}{N (1-\alpha)} \Tr(A  (\Sigma + \kappa_1 Z_1)^{-1} \Sigma (\Sigma + \kappa_1 Z_1)^{-1})}{1 - \frac{1}{N (1-\alpha)} \Tr( (\Sigma + \kappa_1 Z_1)^{-1}) \Sigma (\Sigma + \kappa_1 Z_1)^{-1} \Sigma} \cdot  \Tr\left((\Sigma + \kappa_1Z_1)^{-1} \Sigma (\Sigma + \kappa_1Z_1)^{-1}  B \right) \\
  E_2 &=  \frac{\frac{1}{N} \Tr(A \Sigma (\Sigma + \kappa I)^{-2})}{1 - \frac{1}{N} \Tr(\Sigma^2 (\Sigma + \kappa I)^{-2})} \cdot  \Tr\left((\Sigma + \kappa I)^{-1} \Sigma (\Sigma + \kappa I)^{-1}  B \right) 
 \end{align*}
\end{lemma}
\begin{proof}

By Claim \ref{claim:z1rearrange}, we know that: 
\begin{align*}
(1-\alpha)^2 \Tr\left(\left(\hat{\Sigma} + \lambda I\right)^{-1} A \left(\hat{\Sigma} + \lambda I\right)^{-1} B \right) &= \Tr\left(\left(\hat{\Sigma}_2  + Z_1 \right)^{-1} A \left(\hat{\Sigma}_2  + Z_1 \right)^{-1} B \right) \\
&\sim_{(A)} \kappa_1^2 \left( \Tr\left(\left(\Sigma  + \kappa_1 Z_1 \right)^{-1} A \left(\Sigma  + \kappa_1 Z_1  \right)^{-1} B \right) + E_1\right).
\end{align*}
where (A) applies Lemma \ref{lemma:bach} and Claim \ref{claim:boundedoperatornorm}.  

Furthermore, by Lemma \ref{lemma:bach}, it holds that: 
 \[
\lambda^2\Tr\left(\left(\hat{\Sigma} + \lambda I\right)^{-1} A \left(\hat{\Sigma} + \lambda I\right)^{-1} B \right) \sim \kappa^2 \left(\Tr\left(\left(\Sigma + \kappa I \right)^{-1} A \left(\Sigma + \kappa I \right)^{-1} B \right) + E_2 \right). 
 \]

 Putting this all together yields the desired result.
    
\end{proof}

\begin{lemma}
\label{lemma:equivalenceswitchtraceorder}
Consider the setup of Lemma \ref{lemma:sollichvariant}, and assume the notation above.  Assume $\alpha < 1$. Let $Z_1 = \frac{\alpha}{1-\alpha} \hat{\Sigma}_1 + \frac{\reg}{1-\alpha} I$, and let $\kappa_1 = \kappa(1, N(1-\alpha), Z_1^{-1/2} \Sigma Z_1^{-1/2})$. Then it holds that: 
\[\kappa_1^2 \frac{\Tr((\Sigma + \kappa_1 Z_1)^{-1} \Sigma (\Sigma + \kappa_1 Z_1)^{-1} \Sigma)}{1 - \frac{1}{N(1-\alpha)} \Tr((\Sigma + \kappa_1 Z_1)^{-1} \Sigma (\Sigma + \kappa_1 Z_1)^{-1} \Sigma)} \sim \frac{(1-\alpha)^2 \kappa^2}{\lambda^2 } \frac{\Tr(\Sigma^2 (\Sigma + \kappa I)^{-2})}{1 - \frac{1}{N} \Tr(\Sigma^2 (\Sigma + \kappa I)^{-2})}  \]  
\end{lemma}
\begin{proof}
  
By Claim \ref{claim:z1rearrange}, we know that: 
\begin{align*}
&(1-\alpha)^2 \Tr\left(\left(\hat{\Sigma} + \lambda I\right)^{-1} \Sigma \left(\hat{\Sigma} + \lambda I\right)^{-1} \Sigma \right) \\
&=  \Tr\left(\left(\hat{\Sigma}_2  + Z_1 \right)^{-1} \Sigma \left(\hat{\Sigma}_2  + Z_1 \right)^{-1} \Sigma \right) \\
&\sim_{(A)} \kappa_1^2 \left(1 + \frac{\frac{1}{N(1-\alpha)} \Tr((\Sigma + \kappa_1 Z_1)^{-1} \Sigma (\Sigma + \kappa_1 Z_1)^{-1} \Sigma)}{1-\frac{1}{N(1-\alpha)} \Tr((\Sigma + \kappa_1 Z_1)^{-1} \Sigma (\Sigma + \kappa_1 Z_1)^{-1} \Sigma)} \right) \Tr\left(\left(\Sigma  + \kappa_1 Z_1 \right)^{-1} \Sigma \left(\Sigma  + \kappa_1 Z_1  \right)^{-1} \Sigma \right) \\
&= \kappa_1^2 \frac{ \Tr((\Sigma + \kappa_1 Z_1)^{-1} \Sigma (\Sigma + \kappa_1 Z_1)^{-1} \Sigma)}{1 - \frac{1}{N(1-\alpha)} \Tr((\Sigma + \kappa_1 Z_1)^{-1} \Sigma (\Sigma + \kappa_1 Z_1)^{-1} \Sigma)} 
\end{align*}
where (A) applies Lemma \ref{lemma:bach} and Claim \ref{claim:boundedoperatornorm}.

Furthermore, by Lemma \ref{lemma:bach}, it holds that: 
\begin{align*}
 &\lambda^2 \Tr\left(\left(\hat{\Sigma} + \lambda I\right)^{-1} \Sigma \left(\hat{\Sigma} + \lambda I\right)^{-1} \Sigma \right) \\
 &\sim_{(A)} \kappa^2 \left(1 + \frac{\frac{1}{N} \Tr(\Sigma^2(\Sigma + \kappa I)^{-2})}{1-\frac{1}{N} \Tr(\Sigma^2(\Sigma + \kappa I)^{-2}} \right) \Tr\left(\left(\Sigma + \kappa I \right)^{-1} \Sigma \left(\Sigma + \kappa I \right)^{-1} \Sigma \right) \\
 &=  \kappa^2 \left(\frac{\Tr\left(\Sigma^2 \left(\Sigma + \kappa I \right)^{-2}  \right)}{1-\frac{1}{N} \Tr(\Sigma^2(\Sigma + \kappa I)^{-2}} \right) .   
\end{align*}
where (A) applies Lemma \ref{lemma:bach}. 

 Putting this all together yields the desired result.  
\end{proof}

Next, we relate the random effective regularizer $\kappa_1$ to the deterministic effective regularizer $\kappa(\reg, N, \Sigma)$. 
\begin{lemma}
\label{lemma:kappaequal}
 Consider the setup of Lemma \ref{lemma:sollichvariant}, and assume the notation above.  Assume $\alpha < 1$. Let $Z_1 = \frac{\alpha}{1-\alpha} \hat{\Sigma}_1 + \frac{\reg}{1-\alpha} I$, and let $\kappa_1 = \kappa(1, N(1-\alpha), Z_1^{-1/2} \Sigma Z_1^{-1/2})$. Let $\kappa = \kappa(\reg, N, \Sigma)$. Then, it holds that $\reg \kappa_1 \sim \kappa$.   
\end{lemma}
\begin{proof}

Recall that $\kappa_1 = \kappa(1,  N (1-\alpha), Z_1^{-1/2} \Sigma Z_1^{-1/2})$ is the unique value such that:
\[\frac{1}{\kappa_1} +  \frac{1}{ N (1-\alpha)} \Tr((Z_1^{-1/2} \Sigma Z_1^{-1/2} + \kappa_1 I)^{-1} Z_1^{-1/2} \Sigma Z_1^{-1/2}) = 1.\]
We can write this as:
\[1 +  \frac{\kappa_1}{ N (1-\alpha)} \Tr((\Sigma + \kappa_1 Z_1)^{-1} \Sigma) = \kappa_1.\]
Now we apply Lemma \ref{lemma:equivalencelinear} to see that:
\[\kappa_1 = 1 +  \frac{\kappa_1}{ N (1-\alpha)} \Tr((\Sigma + \kappa_1 Z_1)^{-1} \Sigma) \sim 1 +  \frac{1}{ N (1-\alpha)} \frac{(1-\alpha) \kappa}{\lambda} \Tr((\Sigma + \kappa I)^{-1} \Sigma).\]
We can write this to see that:
\[ \kappa_1 \sim \frac{\kappa}{\lambda} \left(\frac{\lambda}{\kappa} + \frac{1}{N} \Tr((\Sigma + \kappa I)^{-1} \Sigma) \right) = \frac{\kappa}{\lambda}. \]
This implies that $\lambda \kappa_1 \sim \kappa$ as desired. 

\end{proof}

The proofs of these results relied on the following facts. 
\begin{claim}
\label{claim:z1rearrange}
Consider the setup of Lemma \ref{lemma:sollichvariant}, and assume the notation above.  Assume $\alpha < 1$. Let $Z_1 = \frac{\alpha}{1-\alpha} \hat{\Sigma}_1 + \frac{\reg}{1-\alpha} I$. Then it holds that:
\[ (\hat{\Sigma} + \lambda I)^{-1} = (1-\alpha)^{-1} (\hat{\Sigma}_2 + Z_1)^{-1}.\]
\end{claim}
\begin{proof}
We observe that:
\begin{align*}
(1-\alpha)  (\hat{\Sigma} + \reg I)^{-1} &= (1-\alpha) (\alpha \hat{\Sigma}_1 + (1-\alpha) \hat{\Sigma}_2  + \reg I)^{-1} \\
&= (1-\alpha) (1-\alpha)^{-1} \left(\hat{\Sigma}_2  + \frac{\alpha}{1-\alpha} \hat{\Sigma}_1 + \frac{\reg}{1-\alpha} I\right)^{-1} \\
&= \left(\hat{\Sigma}_2  + Z_1 \right)^{-1},
\end{align*}
where $Z_1 = \frac{\alpha}{1-\alpha} \hat{\Sigma}_1 + \frac{\reg}{1-\alpha} I$.    
\end{proof}
\begin{claim}
\label{claim:boundedoperatornorm}
Consider the setup of Lemma \ref{lemma:sollichvariant}, and assume the notation above.  Assume $\alpha < 1$. Let $Z_1 = \frac{\alpha}{1-\alpha} \hat{\Sigma}_1 + \frac{\reg}{1-\alpha} I$. Then it holds that $Z_1$ and $Z_1^{-1}$ both have bounded operator norm. 
\end{claim}
\begin{proof}
 Since $\hat{\Sigma}_1$ is PSD, we observe that: 
\[ \|Z_1\|_{op} = \frac{\alpha}{1-\alpha} \|\hat{\Sigma}_1\|_{op} + \frac{\reg}{1-\alpha}. \]
The fact that $\|\hat{\Sigma}_1\|_{op}$ is bounded follows from the boundedness requirements from Assumption \ref{assumption:MP}. This proves that $\|Z_1\|_{op}$ is bounded. 

To see that $\|Z_1^{-1}\|$ is also bounded, note that:
\[\|Z_1^{-1}\|_{op} \ge \frac{1-\alpha}{\reg}\]
\end{proof}

\subsection{Analysis of Term 1 (T1)}\label{appendix:term1}

We show the following deterministic equivalent for term 1. This analysis is identical to the analysis of the deterministic equivalent for single-objective linear regression \citep{bach,WHS22}, and we include it for completeness. 
\begin{lemma}
\label{lemma:mainterm1}
Consider the setup of Lemma \ref{lemma:sollichvariant}, and assume the notation above. Then it holds that: 
 \begin{align*}
 \lambda^2 \Tr((\hat{\Sigma} + \reg I)^{-1} \Sigma (\hat{\Sigma} + \reg I)^{-1} \Matrixhom{1})  &\sim \frac{\kappa^2}{1 - \frac{1}{N} \Tr(\Sigma^2 (\Sigma + \kappa I)^{-2})} \cdot \Tr(\Sigma (\Sigma + \kappa I)^{-2} \Matrixhom{1})
 \end{align*} 
\end{lemma}
\begin{proof}
We apply Lemma \ref{lemma:bach} to see that:
\begin{align*}
 &\lambda^2 \Tr((\hat{\Sigma} + \reg I)^{-1} \Sigma (\hat{\Sigma} + \reg I)^{-1} \Matrixhom{1}) \\
 &\sim \kappa^2 \Tr((\Sigma + \kappa I)^{-1} \Sigma (\Sigma  + \kappa I)^{-1} \Matrixhom{1}) + \frac{\frac{1}{N} \Tr(\Sigma^2 (\Sigma + \kappa I)^{-2})}{1 - \frac{1}{N} \Tr(\Sigma^2 (\Sigma + \kappa I)^{-2})} \\
 &= \frac{\kappa^2}{1 - \frac{1}{N} \Tr(\Sigma^2 (\Sigma + \kappa I)^{-2})} \cdot \Tr(\Sigma (\Sigma + \kappa I)^{-2} \Matrixhom{1}),
\end{align*}
as desired. 
\end{proof}

\subsection{Analysis of Term 2 (T2)}\label{appendix:term2}

We show the following deterministic equivalent for term 2. 
\begin{lemma}
\label{lemma:mainterm2}
Consider the setup of Lemma \ref{lemma:sollichvariant}, and assume the notation above. Then it holds that: 
 \begin{align*}
 &(1-\alpha)^2 \Tr\left(\hat{\Sigma}_2 (\hat{\Sigma} + \reg I)^{-1} \Sigma (\hat{\Sigma} + \reg I)^{-1} \hat{\Sigma}_2 \Matrixdiff \right)  \\
 &\sim \frac{(1-\alpha)^2}{1 - \frac{1}{N} \Tr(\Sigma^2 (\Sigma + \kappa I)^{-2})} \left( \Tr\left(\left( \Sigma + \kappa I \right)^{-1} \Sigma \left( \Sigma + \kappa I  \right)^{-1} \Sigma \Matrixdiff \Sigma \right)\right) \\
 &+ \frac{(1-\alpha)\frac{1}{N} \Tr(\Sigma^2 (\Sigma + \kappa I)^{-2})}{1 - \frac{1}{N} \Tr(\Sigma^2 (\Sigma + \kappa I)^{-2})}  \cdot  \left(\Tr\left(\Sigma \Matrixdiff  \right)- 2 (1-\alpha) \Tr\left( (\Sigma+ \kappa I)^{-1} \Sigma \Matrixdiff \Sigma \right)\right)
 \end{align*} 
\end{lemma}

The key idea of the proof is to unwrap the randomness in layers. First, we condition on $\hat{\Sigma}_1$ and replace the randomness $\hat{\Sigma}_2$ with a deterministic equivalent where the effective regularizer $\kappa_1$ depends on $\hat{\Sigma}_1$ (Lemma \ref{lemma:deterministicapprox}). At this stage, we unfortunately cannot directly deal with the randomness $\hat{\Sigma}_1$ with deterministic equivalence due to the presence of terms $\kappa_1$ which depend on $\hat{\Sigma}_1$, and we instead apply the sublemmas from the previous section. 

The following lemma replaces the randomness $\hat{\Sigma}_2$ with a deterministic equivalent. 
\begin{lemma}
\label{lemma:deterministicapprox}
Consider the setup of Lemma \ref{lemma:sollichvariant}, and assume the notation above. Assume that $\mix < 1$. Let $Z_1 = \frac{\alpha}{1-\alpha} \hat{\Sigma}_1 + \frac{\reg}{1-\alpha} I$, and let $\kappa_1 = \kappa(1, N(1-\alpha), Z_1^{-1/2} \Sigma Z_1^{-1/2})$. Then it holds that:
 \begin{align*}
 &(1-\alpha)^2 \Tr\left( \hat{\Sigma}_2 (\hat{\Sigma} + \reg I)^{-1} \Sigma (\hat{\Sigma} + \reg I)^{-1} \hat{\Sigma}_2 \Matrixdiff \right) \\
 &\sim \frac{\Tr\left(\left( \Sigma + \kappa_1 Z_1 \right)^{-1} \Sigma \left( \Sigma + \kappa_1 Z_1  \right)^{-1} \Sigma \Matrixdiff \Sigma \right)}{1 - \frac{1}{N (1-\alpha)}\Tr((\Sigma + \kappa_1 Z_1)^{-1} \Sigma (\Sigma + \kappa_1 Z_1)^{-1} \Sigma)} \\
 &+ \frac{\frac{1}{N (1-\alpha)}\Tr((\Sigma + \kappa_1 Z_1)^{-1} \Sigma (\Sigma + \kappa_1 Z_1)^{-1} \Sigma)}{1 - \frac{1}{N (1-\alpha)}\Tr((\Sigma + \kappa_1 Z_1)^{-1} \Sigma (\Sigma + \kappa_1 Z_1)^{-1} \Sigma)} \cdot  \left(\Tr\left(\Sigma \Matrixdiff\right) - 2 \Tr\left(\Sigma (\Sigma + \kappa_1 Z_1)^{-1} \Sigma \Matrixdiff \right) \right).
 \end{align*}
\end{lemma}
\begin{proof}
By Claim \ref{claim:z1rearrange} we have that:
\begin{align*}
(1-\alpha)^2 \Tr\left( \hat{\Sigma}_2 (\hat{\Sigma} + \reg I)^{-1} \Sigma (\hat{\Sigma} + \reg I)^{-1} \hat{\Sigma}_2 \Matrixdiff \right) &= \Tr\left(\hat{\Sigma}_2 \left(\hat{\Sigma}_2  + Z_1 \right)^{-1} \Sigma   \left(\hat{\Sigma}_2  + Z_1 \right)^{-1} \hat{\Sigma}_2 \Matrixdiff \right) \\
&\sim_{(A)} \Tr\left(\Sigma (\Sigma + \kappa_1 Z_1)^{-1}  \Sigma (\Sigma + \kappa_1 Z_1)^{-1} \Sigma \Matrixdiff \right) + E\\
&= \Tr\left((\Sigma + \kappa_1 Z_1)^{-1}  \Sigma (\Sigma + \kappa_1 Z_1)^{-1} \Sigma \Matrixdiff \Sigma\right) + E
\end{align*}
where (A) follows from Lemma \ref{lemma:doublequadraticformMP} and Claim \ref{claim:boundedoperatornorm}, and $E$ is defined such that 
\begin{align*}
E &:=\frac{\frac{1}{N (1-\alpha)} \Tr((\Sigma + \kappa_1 Z_1)^{-1} \Sigma (\Sigma + \kappa_1 Z_1)^{-1} \Sigma)}{1 - \frac{1}{N (1-\alpha)} \Tr(\Sigma + \kappa_1 Z_1)^{-1} \Sigma (\Sigma + \kappa_1 Z_1)^{-1} \Sigma)} \cdot (\kappa_1)^2 \Tr\left(Z_1 \left( \Sigma  + \kappa_1 Z_1 \right)^{-1} \Sigma \left( \Sigma + \kappa_1 Z_1 \right)^{-1} Z_1 \Matrixdiff \right).
\end{align*}
and $\kappa_1 = \kappa(\lambda, N(1-\alpha), Z_1^{-1/2} \Sigma Z_1^{-1/2})$.

Note that:
\begin{align*}
 &(\kappa_1)^2 \Tr\left(Z_1 \left( \Sigma  + \kappa_1 Z_1 \right)^{-1} \Sigma \left( \Sigma + \kappa_1 Z_1 \right)^{-1} Z_1 \Matrixdiff \right) \\
 &= \Tr\left((\kappa_1 Z_1) \left( \Sigma  + \kappa_1 Z_1 \right)^{-1} \Sigma \left( \Sigma + \kappa_1 Z_1 \right)^{-1} (\kappa_1 Z_1) \Matrixdiff \right) \\
 &=  \Tr\left(\left( I - \Sigma (\Sigma + \kappa_1 Z_1)^{-1} \right) \Sigma \left( I - \Sigma  (\Sigma + \kappa_1 Z_1)^{-1} \right)^T \Matrixdiff \right) \\
 &= \Tr\left(\Sigma \Matrixdiff\right) - 2 \Tr\left((\Sigma + \kappa_1 Z_1)^{-1} \Sigma \Matrixdiff \Sigma\right) + \Tr\left((\Sigma + \kappa_1Z_1)^{-1} \Sigma (\Sigma + \kappa_1Z_1)^{-1} \Sigma \Matrixdiff \Sigma \right).
\end{align*}

Note that:
\begin{align*}
&\Tr\left((\Sigma + \kappa_1Z_1)^{-1} \Sigma (\Sigma + \kappa_1Z_1)^{-1} \Sigma \Matrixdiff \Sigma \right) \\
&+ \Tr\left((\Sigma + \kappa_1Z_1)^{-1} \Sigma (\Sigma + \kappa_1Z_1)^{-1} \Sigma \Matrixdiff \Sigma \right) \cdot \frac{\frac{1}{N (1-\alpha)} \Tr((\Sigma + \kappa_1 Z_1)^{-1} \Sigma (\Sigma + \kappa_1 Z_1)^{-1} \Sigma)}{1 - \frac{1}{N (1-\alpha)} \Tr(\Sigma + \kappa_1 Z_1)^{-1} \Sigma (\Sigma + \kappa_1 Z_1)^{-1} \Sigma)} \\
&=   \frac{\Tr\left((\Sigma + \kappa_1Z_1)^{-1} \Sigma (\Sigma + \kappa_1Z_1)^{-1} \Sigma \Matrixdiff \Sigma \right)}{1 - \frac{1}{N (1-\alpha)} \Tr(\Sigma + \kappa_1 Z_1)^{-1} \Sigma (\Sigma + \kappa_1 Z_1)^{-1} \Sigma)} 
\end{align*}
\end{proof}

Now we are ready to prove Lemma \ref{lemma:mainterm2}.
\begin{proof}[Proof of Lemma \ref{lemma:mainterm2}]
The statement follows trivially if $\alpha = 1$. By Lemma \ref{lemma:deterministicapprox}, it holds that:
 \begin{align*}
 &(1-\alpha)^2 \Tr\left( \hat{\Sigma}_2 (\hat{\Sigma} + \reg I)^{-1} \Sigma (\hat{\Sigma} + \reg I)^{-1} \hat{\Sigma}_2 \Matrixdiff \right) \\
 &\sim \frac{\Tr\left(\left( \Sigma + \kappa_1 Z_1 \right)^{-1} \Sigma \left( \Sigma + \kappa_1 Z_1  \right)^{-1} \Sigma \Matrixdiff \Sigma \right)}{1 - \frac{1}{N (1-\alpha)} \Tr((\Sigma + \kappa_1 Z_1)^{-1} \Sigma (\Sigma + \kappa_1 Z_1)^{-1} \Sigma )}  \\
 &+ \frac{\frac{1}{N (1-\alpha)} \Tr((\Sigma + \kappa_1 Z_1)^{-1} \Sigma (\Sigma + \kappa_1 Z_1)^{-1} \Sigma )}{1 - \frac{1}{N (1-\alpha)} \Tr((\Sigma + \kappa_1 Z_1)^{-1} \Sigma (\Sigma + \kappa_1 Z_1)^{-1} \Sigma )} \cdot  \left(\Tr\left(\Sigma \Matrixdiff  \right)- 2 \Tr\left((\Sigma + \kappa_1 Z_1)^{-1} \Sigma \Matrixdiff  \Sigma  \right)\right) \\
 &\sim_{(A)} (1-\alpha)^2 \left( \Tr\left(\left( \Sigma + \kappa I \right)^{-1} \Sigma \left( \Sigma + \kappa I  \right)^{-1} \Sigma \Matrixdiff \Sigma \right)\right) \\
 &+  \frac{\frac{1}{N} \Tr(\Sigma^2 (\Sigma + \kappa I)^{-2})}{1 - \frac{1}{N} \Tr(\Sigma^2 (\Sigma + \kappa I)^{-2})}  \cdot (1-\alpha)^2 \cdot  \Tr\left((\Sigma + \kappa I)^{-1} \Sigma (\Sigma + \kappa I)^{-1}  \Sigma \Matrixdiff \Sigma \right) \\
 &+ \frac{\frac{1}{N (1-\alpha)} \Tr((\Sigma + \kappa_1 Z_1)^{-1} \Sigma (\Sigma + \kappa_1 Z_1)^{-1} \Sigma )}{1 - \frac{1}{N (1-\alpha)} \Tr((\Sigma + \kappa_1 Z_1)^{-1} \Sigma (\Sigma + \kappa_1 Z_1)^{-1} \Sigma )} \cdot  \left(\Tr\left(\Sigma \Matrixdiff  \right)- 2(1-\alpha) \Tr\left( (\Sigma+ \kappa I)^{-1} \Sigma \Matrixdiff \Sigma \right)\right) \\
   &= \frac{(1-\alpha)^2}{1 - \frac{1}{N} \Tr(\Sigma^2 (\Sigma + \kappa I)^{-2})} \left( \Tr\left(\left( \Sigma + \kappa I \right)^{-1} \Sigma \left( \Sigma + \kappa I  \right)^{-1} \Sigma \Matrixdiff \Sigma \right)\right) \\
 &+ \frac{\frac{1}{N (1-\alpha)} \Tr(\Sigma^2 (\Sigma + \kappa_1 Z_1)^{-2})}{1 - \frac{1}{N (1-\alpha)} \Tr(\Sigma^2 (\Sigma + \kappa_1 Z_1)^{-2})}  \cdot  \left(\Tr\left(\Sigma \Matrixdiff  \right)- 2(1-\alpha) \Tr\left( (\Sigma+ \kappa I)^{-1} \Sigma \Matrixdiff \Sigma \right)\right) \\
    &\sim _{(B)} \frac{(1-\alpha)^2}{1 - \frac{1}{N} \Tr(\Sigma^2 (\Sigma + \kappa I)^{-2})} \left( \Tr\left(\left( \Sigma + \kappa I \right)^{-1} \Sigma \left( \Sigma + \kappa I  \right)^{-1} \Sigma \Matrixdiff \Sigma \right)\right) \\
 &+ (1-\alpha) \frac{\frac{1}{N} \Tr(\Sigma^2 (\Sigma + \kappa I)^{-2})}{1 - \frac{1}{N} \Tr(\Sigma^2 (\Sigma + \kappa I)^{-2})}  \cdot  \left(\Tr\left(\Sigma \Matrixdiff  \right)- 2(1-\alpha) \Tr\left( (\Sigma+ \kappa I)^{-1} \Sigma \Matrixdiff \Sigma \right)\right)
 \end{align*} 
 where (A) applies Lemma \ref{lemma:equivalencequadratic}, Lemma \ref{lemma:equivalencelinear}, and (B) uses Lemma \ref{lemma:equivalenceswitchtraceorder} and Lemma \ref{lemma:kappaequal}. 

\end{proof}

\subsection{Analysis of Term 3 (T3)}\label{appendix:term3}

We show the following deterministic equivalent for term 3. 

\begin{lemma}
\label{lemma:mainterm3}
Consider the setup of Lemma \ref{lemma:sollichvariant} and assume the notation above. Let $\Matrixmixed{1} =  (\beta_1 - \beta_2) \beta_1^T$, and let $\kappa = \kappa(\lambda, N, \Sigma)$. Then it holds that: 
 \begin{align*}
 &2 \lambda (1-\alpha) \Tr\left( (\hat{\Sigma} + \reg I)^{-1} \Sigma (\hat{\Sigma} + \reg I)^{-1} \hat{\Sigma}_2\Matrixmixed{1} \right)  \\
 &\sim \frac{2 (1-\alpha) \kappa}{1 - \frac{1}{N} \Tr(\Sigma^2 (\Sigma + \kappa I)^{-2})} \Tr\left((\Sigma + \kappa I)^{-1}  \Sigma (\Sigma + \kappa I)^{-1} \Sigma  \Matrixmixed{1} \right) \\
 &- 2 \frac{(1-\alpha) \frac{1}{N} \Tr(\Sigma^2 (\Sigma + \kappa I)^{-2})}{1 - \frac{1}{N} \Tr(\Sigma^2 (\Sigma + \kappa I)^{-2})} \cdot  \kappa \Tr\left(  (\Sigma + \kappa I)^{-1} \Sigma \Matrixmixed{1} \right) 
 \end{align*} 
\end{lemma}

The analysis follows a similar structure to the analysis of (T2); we similarly unwrap the randomness in layers.  

\begin{lemma}
\label{lemma:deterministicapproxterm3}
Consider the setup of Lemma \ref{lemma:sollichvariant} and assume the notation above. Assume $\mix < 1$. Let $Z_1 = \frac{\alpha}{1-\alpha} \hat{\Sigma}_1 + \frac{\reg}{1-\alpha} I$, and let $\kappa_1 = \kappa(1, N(1-\alpha), Z_1^{-1/2} \Sigma Z_1^{-1/2})$. Then it holds that:
 \begin{align*}
 &2 \lambda (1-\alpha)^2 \Tr\left( (\hat{\Sigma} + \reg I)^{-1} \Sigma (\hat{\Sigma} + \reg I)^{-1} \hat{\Sigma}_2 \Matrixmixed{1} \right) \\
 &\sim 2 \frac{\lambda \kappa_1}{(1-\alpha)} \frac{\Tr\left((\Sigma + \kappa_1 Z_1)^{-1}  \Sigma (\Sigma + \kappa_1 Z_1)^{-1}  \Sigma \Matrixmixed{1} \right)}{1 - \frac{1}{N (1-\alpha)} \Tr((\Sigma + \kappa_1 Z_1)^{-1} \Sigma(\Sigma + \kappa_1 Z_1)^{-1} \Sigma )} \\
 &- 2 \frac{\lambda \kappa_1}{(1-\alpha)} 
 \cdot \frac{\frac{1}{N (1-\alpha)} \Tr(\Sigma^2 (\Sigma + \kappa_1 Z_1)^{-2})}{1 - \frac{1}{N (1-\alpha)} \Tr((\Sigma + \kappa_1 Z_1)^{-1} \Sigma(\Sigma + \kappa_1 Z_1)^{-1} \Sigma )} \cdot  \Tr\left(\Sigma(\Sigma + \kappa_1 Z_1)^{-1} \Sigma \Matrixmixed{1} \right).
 \end{align*}
\end{lemma}
\begin{proof}
By Claim \ref{claim:z1rearrange} we have that:
\begin{align*}
2 \lambda (1-\alpha) \Tr\left( (\hat{\Sigma} + \reg I)^{-1} \Sigma (\hat{\Sigma} + \reg I)^{-1} \hat{\Sigma}_2 \Matrixmixed{1} \right) &= 2 \frac{\lambda}{(1-\alpha)} \Tr\left( \left(\hat{\Sigma}_2  + Z_1 \right)^{-1} \Sigma   \left(\hat{\Sigma}_2  + Z_1 \right)^{-1}  \hat{\Sigma}_2 \Matrixmixed{1} \right) \\
&\sim_{(A)} 2 \frac{\lambda }{(1-\alpha)}  \left(\kappa_1 \Tr\left( (\Sigma + \kappa_1 Z_1)^{-1}  \Sigma (\Sigma + \kappa_1 Z_1)^{-1} \Sigma \Matrixmixed{1} \right) - E \right)
\end{align*}
where (A) follows from Lemma \ref{lemma:singlequadraticformMP} and Claim \ref{claim:boundedoperatornorm}, and $E$ is defined such that 
\begin{align*}
E &:=\frac{\frac{1}{N (1-\alpha)} \Tr((\Sigma + \kappa_1 Z_1)^{-1} \Sigma (\Sigma + \kappa_1 Z_1)^{-1}  \Sigma )}{1 - \frac{1}{N (1-\alpha)} \Tr((\Sigma + \kappa_1 Z_1)^{-1} \Sigma (\Sigma + \kappa_1 Z_1)^{-1}  \Sigma )} \cdot (\kappa_1)^2 \Tr\left(\left( \Sigma  + \kappa_1 Z_1 \right)^{-1} \Sigma \left( \Sigma + \kappa_1 Z_1 \right)^{-1} Z_1  \Matrixmixed{1} \right).
\end{align*}
and $\kappa_1 = \kappa(\lambda, N(1-\alpha), Z_1^{-1/2} \Sigma Z_1^{-1/2})$. 

Note that:
\begin{align*}
 &(\kappa_1)^2 \Tr\left( \left( \Sigma  + \kappa_1 Z_1 \right)^{-1} \Sigma \left( \Sigma + \kappa_1 Z_1 \right)^{-1} Z_1 \Matrixmixed{1} \right) \\
 &= \kappa_1 \Tr\left(\left( \Sigma  + \kappa_1 Z_1 \right)^{-1} \Sigma \left( \Sigma + \kappa_1 Z_1 \right)^{-1} (\kappa_1 Z_1)  \Matrixmixed{1} \right) \\
 &=  \kappa_1 \Tr\left( (\Sigma + \kappa_1 Z_1)^{-1}  \Sigma \left( I - (\Sigma + \kappa_1 Z_1)^{-1}  \Sigma \right) \Matrixmixed{1} \right) \\
 &=  \kappa_1 \Tr\left(  (\Sigma + \kappa_1 Z_1)^{-1} \Sigma \Matrixmixed{1}\right) - \kappa_1 \Tr\left((\Sigma + \kappa_1Z_1)^{-1} \Sigma (\Sigma + \kappa_1Z_1)^{-1}  \Sigma \Matrixmixed{1} \right) \\
\end{align*}

Moreover, note that:
\begin{align*}
& 2 \frac{\lambda \kappa_1 }{(1-\alpha)} \Tr\left( (\Sigma + \kappa_1 Z_1)^{-1}  \Sigma (\Sigma + \kappa_1 Z_1)^{-1} \Sigma \Matrixmixed{1} \right) \\
&+ 2 \frac{\lambda }{(1-\alpha)} \cdot \frac{\frac{1}{N (1-\alpha)} \Tr((\Sigma + \kappa_1 Z_1)^{-1} \Sigma (\Sigma + \kappa_1 Z_1)^{-1}  \Sigma )}{1 - \frac{1}{N (1-\alpha)} \Tr((\Sigma + \kappa_1 Z_1)^{-1} \Sigma (\Sigma + \kappa_1 Z_1)^{-1}  \Sigma )} \cdot \kappa_1 \Tr\left((\Sigma + \kappa_1Z_1)^{-1} \Sigma (\Sigma + \kappa_1Z_1)^{-1}  \Sigma \Matrixmixed{1} \right) \\
&= 2 \frac{\lambda }{(1-\alpha)} \frac{\Tr\left((\Sigma + \kappa_1Z_1)^{-1} \Sigma (\Sigma + \kappa_1Z_1)^{-1}  \Sigma \Matrixmixed{1} \right)}{1 - \frac{1}{N (1-\alpha)} \Tr((\Sigma + \kappa_1 Z_1)^{-1} \Sigma (\Sigma + \kappa_1 Z_1)^{-1}  \Sigma )} \cdot \kappa_1.
\end{align*}

\end{proof}

Now we are ready to prove Lemma \ref{lemma:mainterm2}.
\begin{proof}[Proof of Lemma \ref{lemma:mainterm2}]

The statement follows trivially if $\alpha = 1$. 
By Lemma \ref{lemma:deterministicapprox}, it holds that:
 \begin{align*}
 &2 \lambda (1-\alpha)^2 \Tr\left( (\hat{\Sigma} + \reg I)^{-1} \Sigma (\hat{\Sigma} + \reg I)^{-1} \hat{\Sigma}_2 \Matrixmixed{1} \right) \\
 &\sim 2 \frac{\lambda \kappa_1}{(1-\alpha)} \frac{\Tr\left((\Sigma + \kappa_1 Z_1)^{-1}  \Sigma (\Sigma + \kappa_1 Z_1)^{-1} \Sigma \Matrixmixed{1} \right)}{1 - \frac{1}{N (1-\alpha)} \Tr((\Sigma + \kappa_1 Z_1)^{-1} \Sigma (\Sigma + \kappa_1 Z_1)^{-1}  \Sigma )}  \\
 &- \frac{\frac{1}{N (1-\alpha)} \Tr((\Sigma + \kappa_1 Z_1)^{-1} \Sigma (\Sigma + \kappa_1 Z_1)^{-1}  \Sigma )}{1 - \frac{1}{N (1-\alpha)} \Tr((\Sigma + \kappa_1 Z_1)^{-1} \Sigma (\Sigma + \kappa_1 Z_1)^{-1}  \Sigma )}  \cdot  2 \frac{\lambda \kappa_1}{(1-\alpha)} \Tr\left((\Sigma + \kappa_1 Z_1)^{-1} \Sigma \Matrixmixed{1}\right) \\
  &\sim_{(A)} 2 (1-\alpha) \kappa \Tr\left((\Sigma + \kappa I)^{-1}  \Sigma (\Sigma + \kappa I)^{-1} \Sigma \Matrixmixed{1} \right) \\
  &+ 2 (1-\alpha) \kappa \frac{\frac{1}{N} \Tr(\Sigma^2 (\Sigma + \kappa I)^{-2})}{1 - \frac{1}{N} \Tr(\Sigma^2 (\Sigma + \kappa I)^{-2})} \Tr\left((\Sigma + \kappa I)^{-1}  \Sigma (\Sigma + \kappa I)^{-1} \Sigma \Matrixmixed{1} \right) \\
 &- 2 \frac{\frac{1}{N (1-\alpha)} \Tr((\Sigma + \kappa_1 Z_1)^{-1} \Sigma (\Sigma + \kappa_1 Z_1)^{-1}  \Sigma )}{1 - \frac{1}{N (1-\alpha)} \Tr((\Sigma + \kappa_1 Z_1)^{-1} \Sigma (\Sigma + \kappa_1 Z_1)^{-1}  \Sigma )}  \cdot  \kappa \Tr\left(  (\Sigma + \kappa I)^{-1} \Sigma \Matrixmixed{1} \right) \\
   &= 2 \frac{(1-\alpha) \kappa}{1 - \frac{1}{N} \Tr(\Sigma^2 (\Sigma + \kappa I)^{-2})} \Tr\left((\Sigma + \kappa I)^{-1}  \Sigma (\Sigma + \kappa I)^{-1} \Sigma \Matrixmixed{1} \right) \\
 &- 2 \frac{\frac{1}{N (1-\alpha)} \Tr((\Sigma + \kappa_1 Z_1)^{-1} \Sigma (\Sigma + \kappa_1 Z_1)^{-1}  \Sigma )}{1 - \frac{1}{N (1-\alpha)} \Tr((\Sigma + \kappa_1 Z_1)^{-1} \Sigma (\Sigma + \kappa_1 Z_1)^{-1}  \Sigma )}  \cdot  \kappa \Tr\left(  (\Sigma + \kappa I)^{-1} \Sigma \Matrixmixed{1} \right) \\
   &\sim_{(B)} 2 \frac{(1-\alpha) \kappa}{1 - \frac{1}{N} \Tr(\Sigma^2 (\Sigma + \kappa I)^{-2})} \Tr\left((\Sigma + \kappa I)^{-1}  \Sigma (\Sigma + \kappa I)^{-1} \Sigma \Matrixmixed{1} \right) \\
 &- 2 (1-\alpha) \frac{\frac{1}{N} \Tr(\Sigma^2 (\Sigma + \kappa I)^{-2})}{1 - \frac{1}{N} \Tr(\Sigma^2 (\Sigma + \kappa I)^{-2})} \cdot  \kappa \Tr\left(  (\Sigma + \kappa I)^{-1} \Sigma \Matrixmixed{1} \right) \\
 \end{align*} 
 where (A) applies Lemma \ref{lemma:equivalencequadratic}, Lemma \ref{lemma:equivalencelinear}, and Lemma \ref{lemma:kappaequal}, and (B) uses Lemma \ref{lemma:equivalenceswitchtraceorder} and Lemma \ref{lemma:kappaequal}. 

\end{proof}

\subsection{Proof of Lemma \ref{lemma:sollichvariant}}\label{appendix:lemmasollichvariant}

Lemma \ref{lemma:sollichvariant} follows from the sublemmas in this section.
\begin{proof}
We apply Claim \ref{claim:finitesample} to decompose the error in terms (T1), (T2), and (T3). We replace these terms with deterministic equivalents using Lemma \ref{lemma:mainterm1}, Lemma \ref{lemma:mainterm2}, and Lemma \ref{lemma:mainterm3}. The statement follows from adding these terms. 
\end{proof}

\subsection{Reformulation of Lemma \ref{lemma:sollichvariant} using assumptions from Section \ref{subsec:assumptions}}\label{appendix:analysisterms}

Under the assumptions from Section \ref{subsec:assumptions}, we show the following: 
\begin{lemma}
\label{lemma:sollichmoreprecise}
Suppose that power scaling holds for the eigenvalues and alignment coefficients with scaling $\gamma, \delta > 0$ and correlation coefficient $\rho \in [0, 1)$, and suppose that $P = \infty$. Suppose that $\reg \in (0,1)$, and $N \ge 1$. Let $L_1^{\texttt{det}} := L_1^{\texttt{det}}(\beta_1, \beta_2, \DF, \reg, N, \alpha)$ be the deterministic equivalent from Lemma \ref{lemma:sollichvariant}. Let $\kappa = \kappa(\lambda, N, \Sigma)$ from Definition \ref{def:effectiveregularizer}. Let $L^*(\rho) = \mathbb{E}_{\DC}[(\beta_1 - \beta_2)^T \Sigma (\beta_1 - \beta_2)]$. Then it holds that:
  \begin{align*}
  Q \cdot \mathbb{E}_{\DC}[L_1^{\texttt{det}}] &= \kappa^2 (1 - 2(1-\alpha)^2 (1-\rho)) \sum_{i=1}^P \frac{i^{-\delta -1- \gamma}}{(i^{-1-\gamma} + \kappa)^2} + (1-\alpha)^2 L^*(\rho) \\
  &+ 2\kappa (1-\rho) (1-\alpha) (1 - 2 (1-\alpha)) \sum_{i=1}^P \frac{i^{-\delta - 2(1+\gamma)}}{(i^{-1-\gamma}+\kappa)^2} \\
  &+ 2 (1-\alpha) (1-\rho) \frac{1}{N} \left(\sum_{i=1}^P \frac{i^{-2-2\gamma}}{(i^{-1-\gamma} + \kappa)^2} \right) \cdot (1 - 2(1-\alpha))  \sum_{i=1}^P  \frac{i^{-\delta - 2-2\gamma}}{i^{-1-\gamma} + \kappa},
  \end{align*}
  where $Q = 1 - \frac{1}{N} \sum_{i=1}^P \frac{i^{-2-2\gamma}}{(i^{-1-\gamma} + \kappa)^2}$. 
\end{lemma}

Before proving Lemma \ref{lemma:sollichmoreprecise}, we prove a number of sublemmas where we analyze each of the terms in Lemma \ref{lemma:sollichvariant} using the assumptions from Section \ref{subsec:assumptions}. In the proofs in this section, we use the notation $F \approx F'$ to denote that $F = \Theta(F')$ where the $\Theta$ is allowed to hide dependence on the scaling exponents $\gamma$ and $\delta$. Moreover let $\Sigma = V \Lambda V^T$ be the eigendecomposition of $\Sigma$, where $\Lambda$ is a diagonal matrix consisting of the eigenvalues.

\begin{lemma}
\label{lemma:scalingt1}
Suppose that the power-law scaling holds for the eigenvalues and alignment coefficients with scaling exponents $\gamma, \delta > 0$ and correlation coefficient $\rho \in [0, 1)$, suppose that $P = \infty$. Assume the notation from Lemma \ref{lemma:sollichvariant}. Let $\nu = \min(2(1+\gamma), \gamma + \delta)$. Then it holds that:
\[\mathbb{E}_{\DC}[T_1] := \kappa^2 \cdot \Tr(\Sigma \Sigma_{\kappa}^{-2} \mathbb{E}_{\DC}[\Matrixhom{1}])= \kappa^2  \sum_{i=1}^P \frac{i^{-\delta-1-\gamma}}{(i^{-1-\gamma} + \kappa)^2} .\]
\end{lemma}
\begin{proof}
 Observe that:
\begin{align*}
\Tr(\Sigma \Sigma_{\kappa}^{-2} \mathbb{E}_{\DC}[\Matrixhom{1}]) &= \Tr(\Lambda (\Lambda + \kappa I)^{-2} \mathbb{E}_{\DC}[V^T\beta_1 \beta_1^T V])  \\
 &= \sum_{i=1}^P \frac{i^{-1-\gamma}}{(i^{-1-\gamma} + \kappa)^2} \cdot \mathbb{E}_{\DC}[\langle \beta_1, v_i\rangle^2]   \\
&= \sum_{i=1}^P \frac{i^{-\delta-1-\gamma}}{(i^{-1-\gamma} + \kappa)^2}
\end{align*}
\end{proof}

\begin{lemma}
\label{lemma:scalingt2}
Suppose that the power-law scaling holds for the eigenvalues and alignment coefficients with scaling exponents $\gamma, \delta > 0$ and correlation coefficient $\rho \in [0, 1)$, suppose that $P = \infty$. Assume the notation from Lemma \ref{lemma:sollichvariant}. Then it holds that:
\[\mathbb{E}_{\DC}[T_2] := (1-\alpha)^2 \left( \Tr\left(\Sigma_{\kappa}^{-2} \Sigma^3 \mathbb{E}_{\DC}[\Matrixdiff] \right)\right) = 2 (1-\alpha)^2 (1-\rho) \sum_{i=1}^P \frac{i^{-\delta - 3(1+\gamma)}}{(i^{-1-\gamma} + \kappa)^2}.\]
\end{lemma}
\begin{proof}
First, we observe that 
\[\mathbb{E}_{\DC}[\langle \beta_1 - \beta_2, v_i\rangle^2] = \mathbb{E}_{\DC}[\langle \beta_1 , v_i\rangle^2] + \mathbb{E}_{\DC}[\langle \beta_2, v_i\rangle^2] - 2 \mathbb{E}_{\DC}[\langle \beta_1 , v_i\rangle \langle \beta_2, v_i\rangle]  = i^{-\delta} + i^{-\delta} - 2 \rho i^{-\delta} = 2(1-\rho) i^{-\delta}.\] It is easy to see that:
\begin{align*}
  \Tr\left(\Sigma_{\kappa}^{-2} \Sigma^3 \mathbb{E}_{\DC}[\Matrixdiff] \right) &= \Tr(\Lambda^3 (\Lambda + \kappa I)^{-2} \mathbb{E}_{\DC}[V^T (\beta_1 - \beta_2) (\beta_1 - \beta_2)^T V])  \\
 &= \sum_{i=1}^P \frac{i^{-3(1+\gamma)}}{(i^{-1-\gamma} + \kappa)^2} \cdot \mathbb{E}_{\DC}[\langle \beta_1 - \beta_2, v_i\rangle^2]   \\
  &= 2 (1-\rho) \sum_{i=1}^P \frac{i^{-\delta - 3(1+\gamma)}}{(i^{-1-\gamma} + \kappa)^2}.
\end{align*} 
\end{proof}

\begin{lemma}
\label{lemma:scalingt3}
Suppose that the power-law scaling holds for the eigenvalues and alignment coefficients with scaling exponents $\gamma, \delta > 0$ and correlation coefficient $\rho \in [0, 1)$, suppose that $P = \infty$. Assume the notation from Lemma \ref{lemma:sollichvariant}. Then it holds that:
\[\mathbb{E}_{\DC}[T_3] :=  2 (1-\alpha) \kappa \cdot \Tr\left(\Sigma_{\kappa}^{-2}  \Sigma^2 \Matrixmixed{1} \right) =  2(1-\alpha) \kappa (1-\rho) \sum_{i=1}^P \frac{i^{-\delta-2-2\gamma}}{(i^{-1-\gamma} + \kappa)^2}.\]
\end{lemma}
\begin{proof}
First, we observe that 
\[\mathbb{E}_{\DC}[\langle \beta_1 - \beta_2, v_i\rangle\langle \beta_1, v_i\rangle  ] = \mathbb{E}_{\DC}[\langle \beta_1 , v_i\rangle^2] - \mathbb{E}_{\DC}[\langle \beta_1 , v_i\rangle \langle \beta_2, v_i\rangle]  = i^{-\delta} - \rho i^{-\delta} = (1-\rho) i^{-\delta}.\]
Observe that:
\begin{align*}
\Tr\left(\Sigma_{\kappa}^{-2}  \Sigma^2 \Matrixmixed{1} \right) &= \Tr(\Lambda^2 (\Lambda + \kappa I)^{-2} \mathbb{E}_{\DC}[V^T (\beta_1 - \beta_2) \beta_1^T V])  \\
 &= \sum_{i=1}^P \frac{i^{-2(1+\gamma)}}{(i^{-1-\gamma} + \kappa)^2} \cdot \mathbb{E}_{\DC}[\langle \beta_1 - \beta_2, v_i\rangle\langle \beta_1, v_i\rangle ]   \\
&= (1-\rho) \sum_{i=1}^P \frac{i^{-\delta-2-2\gamma}}{(i^{-1-\gamma} + \kappa)^2}.
\end{align*}
This means that:
\begin{align*}
\mathbb{E}_{\DC}[T_3] &= 2(1-\alpha) \kappa (1-\rho) \sum_{i=1}^P \frac{i^{-\delta-2-2\gamma}}{(i^{-1-\gamma} + \kappa)^2}. 
\end{align*}
\end{proof}

\begin{lemma}
\label{lemma:scalingt4}
Suppose that the power-law scaling holds for the eigenvalues and alignment coefficients with scaling exponents $\gamma, \delta > 0$ and correlation coefficient $\rho \in [0, 1)$, suppose that $P = \infty$. Assume the notation from Lemma \ref{lemma:sollichvariant}. Then it holds that:
\begin{align*}
 |\mathbb{E}_{\DC}[T_4]| &:= 2 \kappa (1-\alpha) \frac{1}{N} \Tr(\Sigma^2 \Sigma_{\kappa} ^{-2}) \cdot  \Tr\left(\Sigma_{\kappa}^{-1} \Sigma \mathbb{E}_{\DC}[\Matrixmixed{1}] \right) \\
 &= 2 \kappa (1-\alpha) (1-\rho) \frac{1}{N} \left(\sum_{i=1}^P \frac{i^{-2-2\gamma}}{(i^{-1-\gamma} + \kappa)^2}\right) \left(\sum_{i=1}^P \frac{i^{-\delta-1-\gamma}}{i^{-1-\gamma} + \kappa}\right)   
\end{align*}
\end{lemma}
\begin{proof}

First, we observe that 
\[\mathbb{E}_{\DC}[\langle \beta_1 - \beta_2, v_i\rangle\langle \beta_1, v_i\rangle ] = \mathbb{E}_{\DC}[\langle \beta_1 , v_i\rangle^2] - \mathbb{E}_{\DC}[\langle \beta_1 , v_i\rangle \langle \beta_2, v_i\rangle]  = i^{-\delta} + - \rho i^{-\delta} = (1-\rho) i^{-\delta}.\]
Observe that:
\begin{align*}
\Tr\left(\Sigma_{\kappa}^{-1}  \Sigma \mathbb{E}_{\DC}[\Matrixmixed{1}] \right) &= \Tr(\Lambda (\Lambda + \kappa I)^{-1} \mathbb{E}_{\DC}[V^T (\beta_1 - \beta_2) \beta_1^T V])  \\
 &= \sum_{i=1}^P \frac{i^{-1-\gamma}}{i^{-1-\gamma} + \kappa} \cdot \mathbb{E}_{\DC}[\langle \beta_1 - \beta_2, v_i\rangle\langle \beta_1, v_i\rangle]   \\
&= (1-\rho) \sum_{i=1}^P \frac{i^{-\delta-1-\gamma}}{i^{-1-\gamma} + \kappa}.
\end{align*}

Now, apply Lemma \ref{lemma:traceterm}, we see that:
\begin{align*}
 |\mathbb{E}_{\DC}[T_4]| &:= 2 \kappa (1-\alpha) \frac{1}{N} \Tr(\Sigma^2 \Sigma_{\kappa} ^{-2}) \cdot  \Tr\left(\Sigma_{\kappa}^{-1} \Sigma \Matrixmixed{1} \right) \\
 &=_{(A)} 2 \kappa (1-\alpha) (1-\rho) \frac{1}{N} \left(\sum_{i=1}^P \frac{i^{-2-2\gamma}}{(i^{-1-\gamma} + \kappa)^2}\right) \left(\sum_{i=1}^P \frac{i^{-\delta-1-\gamma}}{i^{-1-\gamma} + \kappa}\right)
\end{align*}
where (A) follows from Lemma \ref{lemma:traceterm}.
\end{proof}

\begin{lemma}
\label{lemma:scalingt5}
Suppose that the power-law scaling holds for the eigenvalues and alignment coefficients with scaling exponents $\gamma, \delta > 0$ and correlation coefficient $\rho \in [0, 1)$, suppose that $P = \infty$. Assume the notation from Lemma \ref{lemma:sollichvariant}, and similarly let
\begin{align*}
\mathbb{E}_{\DC}[T_5] &:= (1-\alpha) \frac{1}{N} \Tr(\Sigma^2 \Sigma_{\kappa} ^{-2}) \cdot  \left(\Tr\left(\Sigma \mathbb{E}_{\DC}[\Matrixdiff]  \right)- 2 (1-\alpha) \Tr\left( \Sigma_{\kappa}^{-1} \Sigma^2 \mathbb{E}_{\DC}[\Matrixdiff]   \right) \right) \\
    &= 2 (1-\alpha) (1-\rho) \frac{1}{N} \left( \sum_{i=1}^P \frac{i^{-2-2\gamma}}{(i^{-1-\gamma} + \kappa)^2}\right) \cdot  \left(\sum_{i=1}^P i^{-\delta - 1-\gamma} - 2 (1-\alpha) \cdot \sum_{i=1}^P \frac{i^{-\delta-2-2\gamma}}{(i^{-1-\gamma} + \kappa)} \right).
\end{align*}
\end{lemma}
\begin{proof}

First, we observe that 
\[\mathbb{E}_{\DC}[\langle \beta_1 - \beta_2, v_i\rangle^2] = \mathbb{E}_{\DC}[\langle \beta_1 , v_i\rangle^2] + \mathbb{E}_{\DC}[\langle \beta_2, v_i\rangle^2] - 2 \mathbb{E}_{\DC}[\langle \beta_1 , v_i\rangle \langle \beta_2, v_i\rangle]  = i^{-\delta} + i^{-\delta} - 2 \rho i^{-\delta} = 2(1-\rho) i^{-\delta}.\]

Now, observe that:
\begin{align*}
 \mathbb{E}_{\DC}[T_5] &:= (1-\alpha) \frac{1}{N} \Tr(\Sigma^2 \Sigma_{\kappa} ^{-2}) \cdot  \left(\Tr\left(\Sigma \mathbb{E}_{\DC}[\Matrixdiff]  \right)- 2 (1-\alpha) \Tr\left( \Sigma_{\kappa}^{-1} \Sigma^2 \mathbb{E}_{\DC}[\Matrixdiff]  \right) \right) \\
 &= (1-\alpha) \frac{1}{N} \Tr(\Sigma^2 \Sigma_{\kappa} ^{-2}) \cdot  \left(\Tr\left(\Lambda \mathbb{E}_{\DC}[V^T (\beta_1 - \beta_2) (\beta_1 - \beta_2)^T V]  \right) \right) \\
 &- (1-\alpha) \frac{1}{N} \Tr(\Sigma^2 \Sigma_{\kappa} ^{-2}) \cdot \left(2 (1-\alpha) \Tr\left( (\Lambda + \kappa I)^{-1} \Lambda^2 \mathbb{E}_{\DC}[V^T (\beta_1 - \beta_2) (\beta_1 - \beta_2)^T V]  \right) \right) \\
  &= (1-\alpha) \frac{1}{N} \Tr(\Sigma^2 \Sigma_{\kappa} ^{-2}) \cdot  \left(\sum_{i=1}^P i^{-1-\gamma} \langle \beta_1 - \beta_2, v_i\rangle^2 - 2 (1-\alpha) \cdot \sum_{i=1}^P \frac{i^{-2-2\gamma}}{(i^{-1-\gamma} + \kappa)} \langle \beta_1 - \beta_2, v_i\rangle^2 \right) \\
    &= 2 (1-\alpha) (1-\rho) \frac{1}{N} \Tr(\Sigma^2 \Sigma_{\kappa} ^{-2}) \cdot  \left(\sum_{i=1}^P i^{-\delta - 1-\gamma} - 2 (1-\alpha) \cdot \sum_{i=1}^P \frac{i^{-\delta-2-2\gamma}}{(i^{-1-\gamma} + \kappa)} \right) \\
    &=_{(A)} 2 (1-\alpha) (1-\rho) \frac{1}{N} \left( \sum_{i=1}^P \frac{i^{-2-2\gamma}}{(i^{-1-\gamma} + \kappa)^2}\right) \cdot  \left(\sum_{i=1}^P i^{-\delta - 1-\gamma} - 2 (1-\alpha) \cdot \sum_{i=1}^P \frac{i^{-\delta-2-2\gamma}}{(i^{-1-\gamma} + \kappa)} \right).
\end{align*}
where (A) uses Lemma \ref{lemma:traceterm}.
\end{proof}

The proofs of these sublemmas use the following fact.
\begin{lemma}
\label{lemma:traceterm}
  Suppose that the power-law scaling holds for the eigenvalues and alignment coefficients with scaling exponents $\gamma, \delta > 0$ and correlation coefficient $\rho \in [0, 1)$, suppose that $P = \infty$. Assume the notation from Lemma \ref{lemma:sollichvariant}. Then it holds that:
  \[ \Tr\left(\Sigma^2 (\Sigma + \kappa I)^{-2} \right) = \sum_{i=1}^P \frac{i^{-2-2\gamma}}{(i^{-1-\gamma} + \kappa)^2}. \]
\end{lemma}
\begin{proof}
We see that:
\begin{align*}
(1-\alpha) \frac{1}{N} \Tr(\Sigma^2 \Sigma_{\kappa} ^{-2}) &= (1-\alpha) \frac{1}{N} \Tr(V \Lambda^2 (\Lambda + \kappa I)^{-2} V^T) \\
&= (1-\alpha) \frac{1}{N} \Tr(\Lambda^2 (\Lambda + \kappa I)^{-2}) \\
&= \sum_{i=1}^P \frac{i^{-2-2\gamma}}{(i^{-1-\gamma} + \kappa)^2}.
\end{align*}
\end{proof}

Now, we are ready to prove Lemma \ref{lemma:sollichmoreprecise}.
\begin{proof}[Proof of Lemma \ref{lemma:sollichmoreprecise}]
By Lemma \ref{lemma:traceterm}, we know: 
\[Q = 1 - \frac{1}{N}\Tr(\Sigma^2 (\Sigma + \kappa I)^{-2})= 1 - \frac{1}{N} \sum_{i=1}^P \frac{i^{-2-2\gamma}}{(i^{-1-\gamma} + \kappa)^2}.\]
Moreover, we have that: 
     \begin{align*}
  Q \cdot \mathbb{E}_{\DC}[L_1^{\texttt{det}}] &=_{(A)} \mathbb{E}_{\DC}[T_1 + T_2 + T_3 + T_4 + T_5] \\
  &=_{(B)} \kappa^2  \sum_{i=1}^P \frac{i^{-\delta - 1- \gamma}}{(i^{-1-\gamma} + \kappa)^2} + 2 (1-\alpha)^2 (1-\rho) \sum_{i=1}^P \frac{i^{-\delta - 3 (1+\gamma)}}{(i^{-1-\gamma} + \kappa)^2} \\
  &+ 2\kappa (1-\rho) (1-\alpha) \sum_{i=1}^P \frac{i^{-\delta - 2(1+\gamma)}}{(i^{-1-\gamma}+\kappa)^2} \\
  &- 2\kappa (1-\rho)  (1-\alpha) \frac{1}{N} \left(\sum_{i=1}^P \frac{i^{-2-2\gamma}}{(i^{-1-\gamma} + \kappa)^2} \right)   \left(\sum_{i=1}^P \frac{i^{-\delta - 1-\gamma}}{i^{-1-\gamma} + \kappa} \right) \\
  &+ 2 (1-\alpha) (1-\rho) \frac{1}{N} \left( \sum_{i=1}^P \frac{i^{-2-2\gamma}}{(i^{-1-\gamma} + \kappa)^2}\right) \cdot  \left(\sum_{i=1}^P i^{-\delta - 1-\gamma} - 2 (1-\alpha) \cdot \sum_{i=1}^P \frac{i^{-\delta-2-2\gamma}}{(i^{-1-\gamma} + \kappa)} \right).
  \end{align*}
  where (A) follows from Lemma \ref{lemma:sollichvariant}, and (B) follows from Lemmas \ref{lemma:scalingt1}-\ref{lemma:scalingt5}. 

By Claim \ref{claim:boundLstar}, we know that:
\[L^*(\rho) = 2 (1-\rho) \sum_{i=1}^P i^{-\delta-1-\gamma}= 2 (1-\rho) \sum_{i=1}^P \frac{i^{-\delta-3(1+\gamma)}}{(i^{-1-\gamma)^2}}.   \]
This means that:
\begin{align*}
  &L^*(\rho) - 2 (1-\rho) \sum_{i=1}^P \frac{i^{-\delta - 3 (1+\gamma)}}{(i^{-1-\gamma} + \kappa)^2} \\
  &=  2 (1-\rho) \sum_{i=1}^P \left(\frac{i^{-\delta-3(1+\gamma)}}{(i^{-1-\gamma})^2} - \frac{i^{-\delta - 3 (1+\gamma)}}{(i^{-1-\gamma} + \kappa)^2} \right) \\
  &= 2 (1-\rho) \sum_{i=1}^P \left(\frac{i^{-\delta-3(1+\gamma)} \cdot ((i^{-1-\gamma} + \kappa)^2 - (i^{-1-\gamma})^2)}{(i^{-1-\gamma})^2 \cdot (i^{-1-\gamma} + \kappa)^2} \right)  \\
  &= 2 \kappa^2 (1-\rho) \sum_{i=1}^P \left(\frac{i^{-\delta-3(1+\gamma)}}{(i^{-1-\gamma})^2 \cdot (i^{-1-\gamma} + \kappa)^2} \right)  + 4 \kappa (1-\rho) \sum_{i=1}^P \left(\frac{i^{-\delta-3(1+\gamma)} \cdot i^{-1-\gamma}}{(i^{-1-\gamma})^2 \cdot (i^{-1-\gamma} + \kappa)^2} \right) \\
  &= 2 \kappa^2 (1-\rho) \sum_{i=1}^P \left(\frac{i^{-\delta-1-\gamma}}{(i^{-1-\gamma} + \kappa)^2} \right)  + 4 \kappa (1-\rho) \sum_{i=1}^P \left(\frac{i^{-\delta-2(1+\gamma)}}{(i^{-1-\gamma} + \kappa)^2} \right) \\
\end{align*}

Applying this and some other algebraic manipulations, we obtain that:     
  \begin{align*}
  Q \cdot L_1^{\texttt{det}} &= \kappa^2 (1 - 2(1-\alpha)^2 (1-\rho)) \sum_{i=1}^P \frac{i^{-\delta - 1-\gamma}}{(i^{-1-\gamma} + \kappa)^2} + (1-\alpha)^2 L^*(\rho) \\
  &+ 2\kappa (1-\rho) (1-\alpha) (1 - 2 (1-\alpha)) \sum_{i=1}^P \frac{i^{-\delta - 2(1+\gamma)}}{(i^{-1-\gamma}+\kappa)^2} \\
  &- 2 (1-\alpha) (1-\rho) \frac{1}{N} \left(\sum_{i=1}^P \frac{i^{-2-2\gamma}}{(i^{-1-\gamma} + \kappa)^2} \right)   \left(\sum_{i=1}^P i^{-\delta - 1-\gamma} - \sum_{i=1}^P  \frac{i^{-\delta - 2-2\gamma}}{i^{-1-\gamma} + \kappa} \right) \\
  &+ 2 (1-\alpha) (1-\rho) \frac{1}{N} \left( \sum_{i=1}^P \frac{i^{-2-2\gamma}}{(i^{-1-\gamma} + \kappa)^2}\right) \cdot  \left(\sum_{i=1}^P i^{-\delta - 1-\gamma} - 2 (1-\alpha) \cdot \sum_{i=1}^P \frac{i^{-\delta-2-2\gamma}}{(i^{-1-\gamma} + \kappa)} \right) \\
&= \kappa^2 (1 - 2(1-\alpha)^2 (1-\rho)) \sum_{i=1}^P \frac{i^{-\delta - \gamma}}{(i^{-1-\gamma} + \kappa)^2} + (1-\alpha)^2 L^*(\rho) \\
  &+ 2\kappa (1-\rho) (1-\alpha) (1 - 2 (1-\alpha)) \sum_{i=1}^P \frac{i^{-\delta - 2(1+\gamma)}}{(i^{-1-\gamma}+\kappa)^2} \\
  &+ 2 (1-\alpha) (1-\rho) \frac{1}{N} \left(\sum_{i=1}^P \frac{i^{-2-2\gamma}}{(i^{-1-\gamma} + \kappa)^2} \right) \cdot (1 - 2(1-\alpha))  \sum_{i=1}^P  \frac{i^{-\delta - 2-2\gamma}}{i^{-1-\gamma} + \kappa}.
  \end{align*}
\end{proof}

\subsection{Proof of Theorem \ref{thm:scalinglaw}}\label{appendix:proofscaling}

We now prove Theorem \ref{thm:scalinglaw}. In the proof, we again use the notation $F \approx F'$ to denote $F = \Theta(F')$. The main ingredient is Lemma \ref{lemma:sollichmoreprecise}, coupled with the auxiliary calculations in Appendix \ref{appendix:auxiliarycalculations}. 

\begin{proof}
The proof boils down to three steps: (1) obtaining an exact expression, (2) obtaining an up-to-constants asymptotic expression in terms of $\kappa$ and $Q$, and (3) substituting in $\kappa$ and $Q$.

\paragraph{Step 1: Exact expression.}
We apply Lemma \ref{lemma:sollichmoreprecise} to see that: 
  \begin{align*}
  Q \cdot L_1^{\texttt{det}} &= \kappa^2 (1 - 2(1-\alpha)^2 (1-\rho)) \sum_{i=1}^P \frac{i^{-\delta - 1-\gamma}}{(i^{-1-\gamma} + \kappa)^2} + (1-\alpha)^2 L^*(\rho) \\
  &+ 2\kappa (1-\rho) (1-\alpha) (1 - 2 (1-\alpha)) \sum_{i=1}^P \frac{i^{-\delta - 2(1+\gamma)}}{(i^{-1-\gamma}+\kappa)^2} \\
  &+ 2 (1-\alpha) (1-\rho) \frac{1}{N} \left(\sum_{i=1}^P \frac{i^{-2-2\gamma}}{(i^{-1-\gamma} + \kappa)^2} \right) \cdot (1 - 2(1-\alpha))  \sum_{i=1}^P  \frac{i^{-\delta - 2-2\gamma}}{i^{-1-\gamma} + \kappa},
  \end{align*}
where $Q = 1 - \frac{1}{N} \sum_{i=1}^P \frac{i^{-2-2\gamma}}{(i^{-1-\gamma} + \kappa)^2}$, where $L^*(\rho) = \mathbb{E}_{\DC}[(\beta_1 - \beta_2)^T \Sigma (\beta_1 - \beta_2)]$, and where $\kappa = \kappa(\Sigma, N, \lambda)$ as defined in Definition \ref{def:effectiveregularizer}.

\paragraph{Step 2: Asymptotic expression in terms of $\kappa$ and $Q$.} We show that \[ Q \cdot L_1^{\texttt{det}} \approx  \kappa^{\frac{\scalingexp}{1+\gamma}} + (1-\alpha)^2 (1-\rho) + 
(1-\alpha) (1-\rho) \frac{\kappa^{-\frac{1}{1+\gamma}}}{N}. \]

We analyze this expression term-by-term and repeatedly apply Lemma \ref{lemma:splitintegralbounds}. We see that:
\[\kappa^2 (1 - 2(1-\alpha)^2 (1-\rho)) \sum_{i=1}^P \frac{i^{-\delta - 1- \gamma}}{(i^{-1-\gamma} + \kappa)^2} \approx_{(A)} \kappa^{\frac{\scalingexp}{1+\gamma}} (1 - 2(1-\alpha)^2 (1-\rho)) \approx_{(B)} \kappa^{\frac{\scalingexp}{1+\gamma}}, \]
where (A) uses Lemma \ref{lemma:splitintegralbounds} and (B) uses that $\alpha \ge 0.5$. Moreover, we observe that:
\[ (1-\alpha)^2 L^*(\rho) \approx_{(C)} (1-\alpha)^2 (1-\rho),\]
where (C) uses Claim \ref{claim:boundLstar}. Moreover, we see that:
\begin{align*}
2\kappa (1-\rho) (1-\alpha) (1 - 2 (1-\alpha)) \sum_{i=1}^P \frac{i^{-\delta - 2(1+\gamma)}}{(i^{-1-\gamma}+\kappa)^2} 
&\approx_{(D)} (1-\alpha) (1-\rho) (1 - 2 (1-\alpha))  \max\left( \kappa, \kappa^{\frac{\delta + \gamma}{1+\gamma}} \right) \\
 &=_{(E)} O\left((1-\alpha) \sqrt{1-\rho} \max\left( \kappa, \kappa^{\frac{\delta + \gamma}{2(1+\gamma)}} \right)\right) \\
 &= O\left( \sqrt{(1-\alpha)^2 (1-\rho) \cdot  \kappa^{\frac{\min(2(1+\gamma), \gamma + \delta)}{1+\gamma}}} \right) \\ 
 &=_{(F)} O\left( \kappa^{\frac{\min(2(1+\gamma), \gamma + \delta)}{1+\gamma}} + (1-\alpha)^2 (1-\rho) \right) \\
 &= O\left( \kappa^{\frac{\scalingexp}{1+\gamma}} + (1-\alpha)^2 (1-\rho) \right)
\end{align*}
where (D) uses Lemma \ref{lemma:splitintegralbounds}, (E) uses that $1-\rho \le 1$ and that $\kappa = O(1)$ (which follows from Lemma \ref{lemma:kappabasic} and the assumption that $\lambda \in (0,1)$) and (F) follows from AM-GM. Finally, observe that: 
\begin{align*}
&2 (1-\alpha) (1-\rho) \frac{1}{N} \left(\sum_{i=1}^P \frac{i^{-2-2\gamma}}{(i^{-1-\gamma} + \kappa)^2} \right) \cdot (1 - 2(1-\alpha))  \sum_{i=1}^P  \frac{i^{-\delta - 2-2\gamma}}{i^{-1-\gamma} + \kappa}\\
&\approx (1 - 2(1-\alpha)) \cdot (1-\alpha) (1-\rho) \frac{1}{N} \left(\sum_{i=1}^P \frac{i^{-2-2\gamma}}{(i^{-1-\gamma} + \kappa)^2} \right) \sum_{i=1}^P  \frac{i^{-\delta - 2-2\gamma}}{i^{-1-\gamma} + \kappa} \\
&\approx_{(G)} (1 - 2(1-\alpha)) \cdot 
(1-\alpha) (1-\rho) \frac{\kappa^{-\frac{1}{1+\gamma}}}{N} 
\end{align*}
where (G) uses Lemma \ref{lemma:splitintegralbounds} twice. 

Putting this all together, we see that:
\[ Q \cdot L_1^{\texttt{det}} \approx  \kappa^{\frac{\scalingexp}{1+\gamma}} + (1-\alpha)^2 (1-\rho) + (1 - 2(1-\alpha)) \cdot 
(1-\alpha) (1-\rho) \frac{\kappa^{-\frac{1}{1+\gamma}}}{N}. \]

We split into two cases based on $\alpha$. When $\alpha \ge 0.75$, we observe that 
\[(1 - 2(1-\alpha)) \cdot 
(1-\alpha) (1-\rho) \frac{\kappa^{-\frac{1}{1+\gamma}}}{N} \approx (1-\alpha) (1-\rho) \frac{\kappa^{-\frac{1}{1+\gamma}}}{N},\] and when $\alpha \in [0.5, 0.75]$, we observe that 
\[(1 - 2(1-\alpha)) \cdot 
(1-\alpha) (1-\rho) \frac{\kappa^{-\frac{1}{1+\gamma}}}{N} = O\left( (1-\alpha) (1-\rho) \frac{\kappa^{-\frac{1}{1+\gamma}}}{N}\right)\] and 
\[(1-\alpha)^2 (1-\rho) \approx_{(H)} (1-\alpha) (1-\rho) \frac{\kappa^{-\frac{1}{1+\gamma}}}{N}\]
where (H) follows from the fact that $\kappa = \Omega(N^{-1-\gamma})$ by Lemma \ref{lemma:kappabasic}. Altogether, this implies that:
\[ Q \cdot L_1^{\texttt{det}} \approx  \kappa^{\frac{\scalingexp}{1+\gamma}} + (1-\alpha)^2 (1-\rho) + 
(1-\alpha) (1-\rho) \frac{\kappa^{-\frac{1}{1+\gamma}}}{N}, \]
as desired.

\paragraph{Step 2: Substitute in $\kappa$ and $Q$.} Finally, we apply Lemma \ref{lemma:degreesoffreedom} to see that:
\[Q^{-1} = \left(1 - \frac{1}{N} \sum_{i=1}^P \frac{i^{-2-2\gamma}}{(i^{-1-\gamma} + \kappa)^2}\right)^{-1} = \Theta(1).\]
We apply Lemma \ref{lemma:kappabasic} to see that 
\[\kappa = \kappa(\Sigma, N, \Sigma)  = \max(N^{-1-\gamma}, \lambda).\]
Plugging this into the expression derived in Step 2, we obtain the desired expression. 
\end{proof}

\subsection{ Proof of Corollary \ref{cor:scalinglawoptreg}}\label{appendix:proofcor}

We prove Corollary \ref{cor:scalinglawoptreg} using Theorem \ref{thm:scalinglaw}.
\begin{proof}
We apply Theorem \ref{thm:scalinglaw} to see that:
\[\mathbb{E}_{\DC}[L_1^{\texttt{det}}] = \Theta\left( \underbrace{\max(\lambda^{\frac{\scalingexp}{1+\gamma}}, N^{-\scalingexp})}_{\text{finite data error}} + \underbrace{(1-\alpha)^2 \cdot (1 - \rho)}_{\text{mixture error}} +   \underbrace{(1-\alpha) \left(
 \frac{\min(\lambda^{-\frac{1}{1+\gamma}}, N)}{N}
\right) 
(1-\rho)
}_{\text{overfitting error}}\right).\]
We split into three cases: $N \le (1-\alpha)^{-\frac{1}{\scalingexp}}(1-\rho)^{-\frac{1}{\scalingexp}}$, $(1-\alpha)^{-\frac{1}{\scalingexp}}(1-\rho)^{-\frac{1}{\scalingexp}} 
 \le N \le  (1-\alpha)^{-\frac{2+\scalingexp}{\scalingexp}} (1-\rho)^{-\frac{1}{\scalingexp}}$, and $N \ge (1-\alpha)^{-\frac{2+\scalingexp}{\scalingexp}} (1-\rho)^{-\frac{1}{\scalingexp}}$.

 \paragraph{Case 1: $N \le (1-\alpha)^{-\frac{1}{\scalingexp}}(1-\rho)^{-\frac{1}{\scalingexp}}$.} We observe that the finite data error dominates regardless of $\reg$. This is because the condition implies that 
 \[\max(\lambda^{\frac{\scalingexp}{1+\gamma}}, N^{-\scalingexp}) \ge (1-\alpha)(1-\rho),\]
 which dominates both the mixture error and the overfitting error. 

\paragraph{Case 2: $(1-\alpha)^{-\frac{1}{\scalingexp}}(1-\rho)^{-\frac{1}{\scalingexp}} 
 \le N \le  (1-\alpha)^{-\frac{2+\scalingexp}{\scalingexp}} (1-\rho)^{-\frac{1}{\scalingexp}}$.}  We show that the finite error term and overfitting error dominate. Let $\tilde{N} = \min(\lambda^{-\frac{1}{1+\gamma}}, N)$. We can bound the sum of the finite data error and the overfitting error as:
 \begin{align*}
 \max(\lambda^{\frac{\scalingexp}{1+\gamma}}, N^{-\scalingexp}) + (1-\alpha) \left(
 \frac{\min(\lambda^{-\frac{1}{1+\gamma}}, N)}{N}
\right) 
(1-\rho) &=  \tilde{N}^{-\scalingexp} + (1-\alpha) (1-\rho) \frac{\tilde{N}}{N}.
 \end{align*}
 Taking a derivative (and verifying the second order condition), we see that this expression is minimized when:
 \[ \scalingexp \cdot \tilde{N}^{-\scalingexp - 1} = \frac{(1-\alpha) (1-\rho)}{N} \]
 which solves to:
 \[ \tilde{N} = \Theta\left(\left(\frac{(1-\alpha) (1-\rho)}{N}\right)^{-\frac{1}{1+\scalingexp}}\right). \]
 The lower bound on $N$ guarantees that:
 \[ \tilde{N} = \Theta\left(\left(\frac{(1-\alpha) (1-\rho)}{N}\right)^{-\frac{1}{1+\scalingexp}}\right) = O \left(\left((1-\alpha)^{1 + \frac{1}{\scalingexp}} (1-\rho)^{1 + \frac{1}{\scalingexp}}\right)^{-\frac{1}{1+\scalingexp}} \right)= O \left((1-\alpha)^{-\frac{1}{\scalingexp}} (1-\rho)^{-\frac{1}{\scalingexp}} \right) = O(N) \]
 which ensures that $\tilde{N}$ can be achieved by some choice of $\lambda$. In particular, we can take $\reg = \Theta\left(\left(\frac{(1-\alpha) (1-\rho)}{N} \right)^{\frac{1+\gamma}{\scalingexp + 1}}\right)$. 

The resulting sum of the finite error and the overfitting error is:
 \[\max(\lambda^{\frac{\scalingexp}{1+\gamma}}, N^{-\scalingexp}) + (1-\alpha) \left(
 \frac{\min(\lambda^{-\frac{1}{1+\gamma}}, N)}{N}
\right) = \Theta\left(\left(\frac{(1-\alpha) (1-\rho)}{N} \right)^{\frac{\scalingexp}{\scalingexp + 1}} \right). \]

The upper bound on $N$ guarantees that this dominates the mixture error:
\[\Theta\left(\left(\frac{(1-\alpha) (1-\rho)}{N} \right)^{\frac{\scalingexp}{\scalingexp + 1}} \right) = \Omega\left(\left((1-\alpha)^{1 + \frac{2+\scalingexp}{\scalingexp}} (1-\rho)^{1 + \frac{1}{\scalingexp}} \right)^{\frac{\scalingexp}{\scalingexp + 1}} \right) = \Omega((1-\alpha)^2 (1-\rho)) \]
as desired. 

\paragraph{Case 3: $N \ge (1-\alpha)^{-\frac{2+\scalingexp}{\scalingexp}} (1-\rho)^{-\frac{1}{\scalingexp}}$.}
 We show that the mixture and the overfitting error terms dominate. First, we observe that the sum of the mixture error and the finite data error is:
 \[ (1-\alpha)^2 (1-\rho) + (1-\alpha) \left(
 \frac{\min(\lambda^{-\frac{1}{1+\gamma}}, N)}{N}
\right) 
(1-\rho) = \Theta\left( (1-\alpha)(1-\rho) \left(1 - \alpha + \frac{\min(\lambda^{-\frac{1}{1+\gamma}}, N)}{N} \right) \right).\]
This is minimized by taking $\reg = \Theta((N(1-\alpha))^{-1-\gamma})$, which yields $\Theta((1-\alpha)^2 (1-\rho))$.

The upper bound on $N$ and the setting of $\reg$ guarantees that this term dominates the finite data error: 
 \[\max(\lambda^{\frac{\scalingexp}{1+\gamma}}, N^{-\scalingexp}) = O((N(1-\alpha))^{-\scalingexp}) \le O\left((1-\alpha)^{-\scalingexp} (1-\alpha)^{2 + \scalingexp} (1-\rho) \right) = O((1-\alpha)^2 (1-\rho),\]
 as desired.

\end{proof}

\subsection{Proof of Theorem \ref{thm:scalinglawexcess}}\label{appendix:proofscalingexcess}

We prove Theorem \ref{thm:scalinglawexcess}. 

\begin{proof}[Proof of Theorem \ref{thm:scalinglawexcess}]
Like the proof of Theorem \ref{thm:scalinglaw}, the proof boils down to three steps: (1) obtaining an exact expression, (2) obtaining an up-to-constants asymptotic expression in terms of $\kappa$, and (3) substituting in $\kappa$.

\paragraph{Step 1: Exact expression.}
We first apply Lemma \ref{lemma:sollichmoreprecise} to obtain the precise loss:
\begin{align*}
 Q \cdot \mathbb{E}_{\DC}[L_1^*(\beta_1, \beta_2, \DF, \lambda_E, N, \alpha_E)]  &=\kappa^2 (1 - 2(1-\alpha)^2 (1-\rho)) \sum_{i=1}^P \frac{i^{-\delta -1- \gamma}}{(i^{-1-\gamma} + \kappa)^2} + (1-\alpha)^2 L^*(\rho) \\
  &+  2\kappa (1-\rho) (1-\alpha) (1 - 2 (1-\alpha)) \sum_{i=1}^P \frac{i^{-\delta - 2(1+\gamma)}}{(i^{-1-\gamma}+\kappa)^2} \\
  &+ 2 (1-\alpha) (1-\rho) \frac{1}{N} \left(\sum_{i=1}^P \frac{i^{-2-2\gamma}}{(i^{-1-\gamma} + \kappa)^2} \right) \cdot (1 - 2(1-\alpha))  \sum_{i=1}^P  \frac{i^{-\delta - 2-2\gamma}}{i^{-1-\gamma} + \kappa},
  \end{align*}
  where $Q = 1 - \frac{1}{N} \sum_{i=1}^P \frac{i^{-2-2\gamma}}{(i^{-1-\gamma} + \kappa)^2}$ and where $\kappa = \kappa(\Sigma, N, \lambda)$ as defined in Definition \ref{def:effectiveregularizer}. 
  This can be written as: 
\begin{align*}
 &\mathbb{E}_{\DC}[L_1^*(\beta_1, \beta_2, \DF, \lambda_E, N, \alpha_E)] - (1-\alpha)^2 L^*(\rho)   \\
 &=Q^{-1} \cdot \kappa^2 (1 - 2(1-\alpha)^2 (1-\rho))  \sum_{i=1}^P \frac{i^{-\delta - 1- \gamma}}{(i^{-1-\gamma}   + \kappa)^2}  \\
  &+ Q^{-1} \cdot 2\kappa (1-\rho) (1-\alpha) (1 - 2 (1-\alpha)) \sum_{i=1}^P \frac{i^{-\delta - 2(1+\gamma)}}{(i^{-1-\gamma}+\kappa)^2} \\
  &+ Q^{-1} \cdot 2 (1-\alpha) (1-\rho) \frac{1}{N} \left(\sum_{i=1}^P \frac{i^{-2-2\gamma}}{(i^{-1-\gamma} + \kappa)^2} \right) \cdot (1 - 2(1-\alpha))  \sum_{i=1}^P  \frac{i^{-\delta - 2-2\gamma}}{i^{-1-\gamma} + \kappa} \\
   &+ \frac{1-Q}{Q} (1-\alpha)^2 L^*(\rho).
  \end{align*}  

\paragraph{Step 2: Asymptotic expression in terms of $\kappa$.}
 We use the notation $F \approx F'$ to denote that $F = \Theta(F')$. We obtain:
 \begin{align*}
 &\mathbb{E}_{\DC}[L_1^*(\beta_1, \beta_2, \DF, \lambda_E, N, \alpha_E)] - (1-\alpha)^2 L^*(\rho)   \\
 &\approx_{(A)} \kappa^2 (1 - 2(1-\alpha)^2 (1-\rho))  \sum_{i=1}^P \frac{i^{-\delta -1- \gamma}}{(i^{-1-\gamma}   + \kappa)^2}  \\
  &+ \kappa (1-\rho) (1-\alpha) (1 - 2 (1-\alpha)) \sum_{i=1}^P \frac{i^{-\delta - 2(1+\gamma)}}{(i^{-1-\gamma}+\kappa)^2} \\
  &+ (1-\alpha) (1-\rho) \frac{1}{N} \left(\sum_{i=1}^P \frac{i^{-2-2\gamma}}{(i^{-1-\gamma} + \kappa)^2} \right) \cdot (1 - 2(1-\alpha))  \sum_{i=1}^P  \frac{i^{-\delta - 2-2\gamma}}{i^{-1-\gamma} + \kappa} \\
   &+ (1-Q) (1-\alpha)^2 L^*(\rho) \\
   &\approx_{(B)} \kappa^2 \sum_{i=1}^P \frac{i^{-\delta - 1- \gamma}}{(i^{-1-\gamma}   + \kappa)^2} + \kappa (1-\rho) (1-\alpha) \sum_{i=1}^P \frac{i^{-\delta - 2(1+\gamma)}}{(i^{-1-\gamma}+\kappa)^2} \\
  &+ (1-\alpha) (1-\rho) \frac{1}{N} \left(\sum_{i=1}^P \frac{i^{-2-2\gamma}}{(i^{-1-\gamma} + \kappa)^2} \right) \sum_{i=1}^P  \frac{i^{-\delta - 2-2\gamma}}{i^{-1-\gamma} + \kappa} \\
   &+ (1-\alpha)^2 L^*(\rho) \cdot \frac{1}{N} \left(\sum_{i=1}^P \frac{i^{-2-2\gamma}}{(i^{-1-\gamma} + \kappa)^2} \right).
  \end{align*}  
  where (A) uses that $Q^{-1}$ is a constant by Lemma \ref{lemma:degreesoffreedom} and (B) uses that $\alpha \ge 0.75$ and the definition of $Q$. Now, using the bounds from Lemma \ref{lemma:splitintegralbounds}, and the bound from Claim \ref{claim:boundLstar}, we obtain: 
\begin{align*}
&\mathbb{E}_{\DC}[L_1^*(\beta_1, \beta_2, \DF, \lambda_E, N, \alpha_E)] - (1-\alpha)^2 L^*(\rho)  \\
&\approx \kappa^{\frac{\min(2(1+\gamma), \gamma + \delta)}{1+\gamma}}  + (1-\rho) (1-\alpha) \max \left( \kappa, \kappa^{\frac{\gamma + \delta}{1+\gamma}} \right) + (1-\alpha) (1-\rho) \frac{\kappa^{-\frac{1}{1+\gamma}}}{N}  + \frac{\kappa^{-\frac{1}{1+\gamma}}}{N} (1-\alpha)^2 (1-\rho) \\
 &\approx \kappa^{\frac{\scalingexp}{1+\gamma}}  + (1-\rho) (1-\alpha) \kappa^{\frac{\scalingexp'}{1+\gamma}} + (1-\alpha) (1-\rho) \frac{\kappa^{-\frac{1}{1+\gamma}}}{N}.
  \end{align*}  

\paragraph{Step 3: Substituting in $\kappa$.} Finally, we apply Lemma \ref{lemma:kappabasic} to see that 
\[\kappa = \kappa(\Sigma, N, \Sigma)  = \max(N^{-1-\gamma}, \lambda).\] Plugging this into the expression derived in Step 2, we obtain the desired expression. 
\end{proof}

\subsection{Proof of Corollary \ref{cor:scalinglawoptregexcess}}\label{appendix:proofscalingoptregexcess}

We prove Corollary \ref{cor:scalinglawoptregexcess} using Theorem \ref{thm:scalinglawexcess}.
\begin{proof}
We apply Theorem \ref{thm:scalinglaw} to see that:
\begin{align*}
 &\mathbb{E}_{\DC}[L_1^{\texttt{det}} - L_1(\beta(\alpha, 0))]  \\
 &=  \Theta\left( \underbrace{\max(\lambda^{\frac{\scalingexp}{1+\gamma}}, N^{-\scalingexp})}_{\text{finite data error}} + \underbrace{(1-\rho) (1-\alpha) \max(\lambda^{\frac{\scalingexp'}{1+\gamma}}, N^{-\scalingexp'})}_{\text{mixture finite data error}} +   \underbrace{(1-\alpha) \left(
 \frac{\min(\lambda^{-\frac{1}{1+\gamma}}, N)}{N}
\right) 
(1-\rho)
}_{\text{overfitting error}}\right).
\end{align*}

We split into three cases: $N \le (1-\alpha)^{-\frac{1}{\scalingexp}}(1-\rho)^{-\frac{1}{\scalingexp}}$, $(1-\alpha)^{-\frac{1}{\scalingexp}}(1-\rho)^{-\frac{1}{\scalingexp}} \le N \le  (1-\alpha)^{-\frac{\scalingexp'+1}{\scalingexp - \scalingexp'}}(1-\rho)^{-\frac{\scalingexp'+1}{\scalingexp - \scalingexp'}}$, and $N \ge (1-\alpha)^{-\frac{\scalingexp'+1}{\scalingexp - \scalingexp'}}(1-\rho)^{-\frac{\scalingexp'+1}{\scalingexp - \scalingexp'}}$.

 \paragraph{Case 1: $N \le (1-\alpha)^{-\frac{1}{\scalingexp}}(1-\rho)^{-\frac{1}{\scalingexp}}$.} We observe that the finite data error dominates regardless of $\reg$. This is because the condition implies that 
 \[\max(\lambda^{\frac{\scalingexp}{1+\gamma}}, N^{-\scalingexp}) \ge (1-\alpha)(1-\rho),\]
 which dominates both the mixture finite data error and the overfitting error. 

\paragraph{Case 2: $(1-\alpha)^{-\frac{1}{\scalingexp}}(1-\rho)^{-\frac{1}{\scalingexp}} \le N \le  (1-\alpha)^{-\frac{\scalingexp'+1}{\scalingexp - \scalingexp'}}(1-\rho)^{-\frac{\scalingexp'+1}{\scalingexp - \scalingexp'}}$.}  We show that the finite error term and overfitting error dominate. Let $\tilde{N} = \min(\lambda^{-\frac{1}{1+\gamma}}, N)$. We can bound the sum of the finite data error and the overfitting error as:
 \begin{align*}
 \max(\lambda^{\frac{\scalingexp}{1+\gamma}}, N^{-\scalingexp}) + (1-\alpha) \left(
 \frac{\min(\lambda^{-\frac{1}{1+\gamma}}, N)}{N}
\right) 
(1-\rho) &=  \tilde{N}^{-\scalingexp} + (1-\alpha) (1-\rho) \frac{\tilde{N}}{N}.
 \end{align*}
 Taking a derivative (and verifying the second order condition), we see that this expression is minimized when:
 \[ \scalingexp \cdot \tilde{N}^{-\scalingexp - 1} = \frac{(1-\alpha) (1-\rho)}{N} \]
 which solves to:
 \[ \tilde{N} = \Theta\left(\left(\frac{(1-\alpha) (1-\rho)}{N}\right)^{-\frac{1}{1+\scalingexp}}\right). \]
 The lower bound on $N$ guarantees that:
 \[ \tilde{N} = \Theta\left(\left(\frac{(1-\alpha) (1-\rho)}{N}\right)^{-\frac{1}{1+\scalingexp}}\right) = O \left(\left((1-\alpha)^{1 + \frac{1}{\scalingexp}} (1-\rho)^{1 + \frac{1}{\scalingexp}}\right)^{-\frac{1}{1+\scalingexp}} \right)= O \left((1-\alpha)^{-\frac{1}{\scalingexp}} (1-\rho)^{-\frac{1}{\scalingexp}} \right) = O(N) \]
 which ensures that $\tilde{N}$ can be achieved by some choice of $\lambda$. In particular, we can take $\reg = \Theta\left(\left(\frac{(1-\alpha) (1-\rho)}{N} \right)^{\frac{1+\gamma}{\scalingexp + 1}}\right)$. 

The resulting sum of the finite error and the overfitting error is:
 \[\max(\lambda^{\frac{\scalingexp}{1+\gamma}}, N^{-\scalingexp}) + (1-\alpha) \left(
 \frac{\min(\lambda^{-\frac{1}{1+\gamma}}, N)}{N}
\right) = \Theta\left(\left(\frac{(1-\alpha) (1-\rho)}{N} \right)^{\frac{\scalingexp}{\scalingexp + 1}} \right). \]

The upper bound on $N$ and the choice of $\lambda$ guarantees that this dominates the mixture finite data error, as shown below: 
\begin{align*}
 &(1-\rho) (1-\alpha) \max(\lambda^{\frac{\scalingexp'}{1+\gamma}}, N^{-\scalingexp'}) \\
 &= \Theta\left((1-\rho) (1-\alpha) \left(\frac{(1-\alpha) (1-\rho)}{N} \right)^{\frac{\scalingexp'}{\scalingexp + 1}}\right)\\
 &= \Theta\left(\left(\frac{(1-\alpha) (1-\rho)}{N} \right)^{\frac{\scalingexp}{\scalingexp + 1}} (1-\alpha) (1-\rho) \left(\frac{(1-\alpha) (1-\rho)}{N} \right)^{\frac{\scalingexp' - \scalingexp}{\scalingexp + 1}} \right) \\
 &= \Theta\left(\left(\frac{(1-\alpha) (1-\rho)}{N} \right)^{\frac{\scalingexp}{\scalingexp + 1}} (1-\alpha)^{\frac{\scalingexp'+1}{\scalingexp + 1}} (1-\rho)^{\frac{\scalingexp'+1}{\scalingexp + 1}} N^{\frac{\scalingexp - \scalingexp'}{\scalingexp + 1}} \right) \\
 &= O\left(\left(\frac{(1-\alpha) (1-\rho)}{N} \right)^{\frac{\scalingexp}{\scalingexp + 1}} (1-\alpha)^{\frac{\scalingexp'+1}{\scalingexp + 1}} (1-\rho)^{\frac{\scalingexp'+1}{\scalingexp + 1}} (1-\alpha)^{-\frac{\scalingexp'+1}{\scalingexp + 1}} (1-\rho)^{-\frac{\scalingexp'+1}{\scalingexp + 1}} \right) \\
 &=  O\left(\left(\frac{(1-\alpha) (1-\rho)}{N} \right)^{\frac{\scalingexp}{\scalingexp + 1}} \right)
\end{align*}
as desired.

\paragraph{Case 3: $N \ge (1-\alpha)^{-\frac{\scalingexp'+1}{\scalingexp - \scalingexp'}}(1-\rho)^{-\frac{\scalingexp'+1}{\scalingexp - \scalingexp'}}$.}
 We show that the mixture finite data error and the overfitting error terms dominate. First, we observe that the sum of the mixture error and the finite data error is:
 \begin{align*}
 &(1-\rho) (1-\alpha) \max(\lambda^{\frac{\scalingexp'}{1+\gamma}}, N^{-\scalingexp'}) + (1-\alpha) \left(
 \frac{\min(\lambda^{-\frac{1}{1+\gamma}}, N)}{N}
\right) 
(1-\rho) \\
&= \Theta\left( (1-\alpha)(1-\rho) \left(\lambda^{\frac{\scalingexp'}{1+\gamma}} + \frac{\min(\lambda^{-\frac{1}{1+\gamma}}, N)}{N} \right) \right)    
 \end{align*}
This is minimized by taking $\reg = \Theta(N^{-\frac{1+\gamma}{\scalingexp'+1}})$, which yields $\Theta((1-\alpha) (1-\rho) N^{-\frac{\scalingexp'}{\scalingexp'+1}})$.

The upper bound on $N$ and the setting of $\reg$ guarantees that this term dominates the finite data error: 
\begin{align*}
\max(\lambda^{\frac{\scalingexp}{1+\gamma}}, N^{-\scalingexp}) &= \Theta(N^{-\frac{\scalingexp}{\scalingexp'+1}}) \\
&\le \Theta\left((1-\alpha) (1-\rho) N^{-\frac{\scalingexp'}{\scalingexp'+1}} (1-\alpha)^{-1} (1-\rho)^{-1} N^{-\frac{\scalingexp - \scalingexp'}{\scalingexp'+1}}\right) \\
&= O\left((1-\alpha) (1-\rho) N^{-\frac{\scalingexp'}{\scalingexp'+1}} (1-\alpha)^{-1} (1-\rho)^{-1} (1-\alpha) (1-\rho) \right) \\
&= O\left((1-\alpha) (1-\rho) N^{-\frac{\scalingexp'}{\scalingexp'+1}} \right) 
\end{align*}
as desired. 
  
\end{proof}

\subsection{Auxiliary calculations under power scaling assumptions}\label{appendix:auxiliarycalculations}

We show the following auxiliary calculations which we use when analyzing the terms in Lemma \ref{lemma:sollichvariant} under the power scaling assumptions. Throughout this section, we again use the notation $F \approx F'$ to denote that $F = \Theta(F')$.

\begin{lemma}
\label{lemma:splitintegralbounds}
Suppose that power-law scaling  holds for eigenvalues and alignment coefficients with scaling exponents $\gamma, \delta >0$ and correlation coefficient $\rho \in [0, 1)$, and suppose that $P = \infty$. Let $\kappa = \kappa(\lambda, N, \Sigma)$ be defined according to Definition \ref{def:effectiveregularizer}. Then the following holds:
\begin{align*}
\sum_{i=1}^P \frac{i^{-\delta -1 - \gamma}}{(i^{-1-\gamma} +  \kappa)^2} &\approx \kappa^{-2} \kappa^{\frac{\min(2(1+\gamma), \gamma + \delta)}{1+\gamma}} \\
\sum_{i=1}^P \frac{i^{-\delta - 3(1+\gamma)}}{(i^{-1-\gamma} +  \kappa)^2} &\approx 1 \\
\sum_{i=1}^P  \frac{i^{-\delta - 2-2\gamma}}{i^{-1-\gamma} + \kappa} &\approx 1 \\
\sum_{i=1}^P \frac{i^{-\delta - 2(1+\gamma)}}{(i^{-1-\gamma} + \kappa)^2} &\approx \max(1, \kappa^{\frac{\delta - 1}{1+\gamma}}) \\
\sum_{i=1}^P \frac{i^{-\delta - 1 - \gamma}}{i^{-1-\gamma } + \kappa} &\approx \max(1, \kappa^{\frac{\delta - 1}{1+\gamma}}) \\
\sum_{i=1}^P \frac{i^{- 2 - 2 \gamma}}{(i^{-1-\gamma} + \kappa)^2} &\approx \kappa^{-\frac{1}{1+\gamma}} \\
\sum_{i=1}^P \frac{i^{- 1 - \gamma}}{i^{-1-\gamma} + \kappa} &\approx \kappa^{-\frac{1}{1+\gamma}} \\
\sum_{i=1}^P \frac{i^{-1 - \gamma}}{(i^{-1-\gamma} +  \kappa)^2} &\approx \kappa^{-2} \kappa^{\frac{\gamma}{1+\gamma}} \\
\end{align*}
\end{lemma}
\begin{proof}
To prove the first statement, observe that:
\begin{align*}
\sum_{i=1}^P \frac{i^{-\delta-1-\gamma}}{(i^{-1-\gamma} + \kappa)^2} &=  \sum_{i \le \kappa^{-\frac{1}{1+\gamma}}} \frac{i^{-\delta-1-\gamma}}{(i^{-1-\gamma} + \kappa)^2} + \sum_{i \ge \kappa^{-\frac{1}{1+\gamma}}} \frac{i^{-\delta-1-\gamma}}{(i^{-1-\gamma} + \kappa)^2} \\
&\approx \sum_{i \le \kappa^{-\frac{1}{1+\gamma}}} i^{1+\gamma-\delta} + \kappa^{-2} \sum_{i \ge \kappa^{-\frac{1}{1+\gamma}}} i^{-\delta-1-\gamma} \\
&\approx \max(1, \kappa^{-\frac{2+\gamma-\delta}{1+\gamma}}) + \kappa^{-2} \kappa^{\frac{\delta+\gamma}{1+\gamma}} \\
&= \kappa^{-2} \max(\kappa^2, \kappa^{\frac{\gamma+\delta}{1+\gamma}}) + \kappa^{-2} \kappa^{\frac{\delta+\gamma}{1+\gamma}} \\
&\approx \kappa^{-2} \max(\kappa^2, \kappa^{\frac{\gamma+\delta}{1+\gamma}}) \\
&\approx \kappa^{-2}  \kappa^{\frac{\min(2(1+\gamma), \gamma + \delta)}{1+\gamma}}.
\end{align*}

To prove the second statement, we use Lemma \ref{lemma:kappabasic} and the assumption that $\lambda \in (0,1)$ to see $\kappa = \Theta(\max(\lambda, N^{-1-\gamma})) = O(1)$. This means that 
\[\sum_{i=1}^P \frac{i^{-\delta - 3(1+\gamma)}}{(i^{-1-\gamma} + \kappa)^2} = \Omega \left(\sum_{i=1}^P i^{-\delta - 3(1+\gamma)}  \right) = \Omega(1).\]
Moreover, we see that:
\[\sum_{i=1}^P \frac{i^{-\delta - 3(1+\gamma)}}{(i^{-1-\gamma} + \kappa)^2} = O \left(\sum_{i=1}^P \frac{i^{-\delta - 3(1+\gamma)}}{(i^{-1-\gamma})^2}  \right)= O \left(\sum_{i=1}^P i^{-\delta - 1 -\gamma)}  \right) = \Omega(1).\]

To prove the third statement, we use Lemma \ref{lemma:kappabasic} and the assumption that $\lambda \in (0,1)$ to see $\kappa = \Theta(\max(\lambda, N^{-1-\gamma})) = O(1)$. This means that 
\[\sum_{i=1}^P  \frac{i^{-\delta - 2-2\gamma}}{i^{-1-\gamma} + \kappa} = \Omega \left(\sum_{i=1}^P i^{-\delta - 2(1+\gamma)}  \right) = \Omega(1).\]
Moreover, we see that:
\[\sum_{i=1}^P \frac{i^{-\delta - 2(1+\gamma)}}{i^{-1-\gamma} + \kappa} = O \left(\sum_{i=1}^P \frac{i^{-\delta - 2(1+\gamma)}}{i^{-1-\gamma}}  \right) = O \left(\sum_{i=1}^P i^{-\delta - 1  - \gamma}  \right) = O(1).\]

To prove the fourth statement, observe that:
\begin{align*}
\sum_{i=1}^P \frac{i^{-\delta-2-2\gamma}}{(i^{-1-\gamma} + \kappa)^2} &\approx  \sum_{i \le \kappa^{-\frac{1}{1+\gamma}}} \frac{i^{-\delta-2-2\gamma}}{(i^{-1-\gamma} + \kappa)^2} + \sum_{i \ge \kappa^{-\frac{1}{1+\gamma}}} \frac{i^{-\delta-2-2\gamma}}{(i^{-1-\gamma} + \kappa)^2} \\
&\approx \sum_{i \le \kappa^{-\frac{1}{1+\gamma}}} i^{-\delta} +  \kappa^{-2} \sum_{i \ge \kappa^{-\frac{1}{1+\gamma}}} i^{-\delta-2-2\gamma} \\
&\approx \max(1, \kappa^{-\frac{1-\delta}{1+\gamma}}) + \kappa^{-2} \kappa^{\frac{\delta+1+2\gamma}{1+\gamma}} \\
&\approx \max(1, \kappa^{\frac{\delta-1}{1+\gamma}}).
\end{align*}

To prove the fifth statement, observe that:
\begin{align*}
 \sum_{i=1}^P \frac{i^{-\delta-1-\gamma}}{i^{-1-\gamma} + \kappa} &= \sum_{i \le \kappa^{-\frac{1}{1+\gamma}}} \frac{i^{-\delta-1-\gamma}}{i^{-1-\gamma} + \kappa} + \sum_{i \ge \kappa^{-\frac{1}{1+\gamma}}} \frac{i^{-\delta-1-\gamma}}{i^{-1-\gamma} + \kappa} \\
&\approx  \sum_{i \le \kappa^{-\frac{1}{1+\gamma}}} i^{-\delta} + \kappa^{-1} \sum_{i \ge \kappa^{-\frac{1}{1+\gamma}}} i^{-\delta-1-\gamma} \\
&\approx \max(1, \kappa^{-\frac{1-\delta}{1+\gamma}}) + \kappa^{-1} \kappa^{\frac{\delta+\gamma}{1+\gamma}} \\
&\approx \max(1, \kappa^{\frac{\delta-1}{1+\gamma}}). 
\end{align*}

To prove the sixth statement, observe that:

\begin{align*}
\sum_{i=1}^P \frac{i^{-2-2\gamma}}{(i^{-1-\gamma} + \kappa)^2} &=  \sum_{i \le \kappa^{-\frac{1}{1+\gamma}}} \frac{i^{-2-2\gamma}}{(i^{-1-\gamma} + \kappa)^2} + \sum_{i \ge \kappa^{-\frac{1}{1+\gamma}}} \frac{i^{-2-2\gamma}}{(i^{-1-\gamma} + \kappa)^2} \\
&\approx \sum_{i \le \kappa^{-\frac{1}{1+\gamma}}} 1 + \kappa^{-2} \sum_{i \ge \kappa^{-\frac{1}{1+\gamma}}} i^{-2-2\gamma} \\
&\approx  \kappa^{-\frac{1}{1+\gamma}} + \kappa^{-2} \kappa^{\frac{1 + 2\gamma}{1+\gamma}} \\
&\approx  \kappa^{-\frac{1}{1+\gamma}}.
\end{align*}

To prove the seventh statement, observe that:

\begin{align*}
\sum_{i=1}^P \frac{i^{-1-\gamma}}{i^{-1-\gamma} + \kappa} &=  \sum_{i \le \kappa^{-\frac{1}{1+\gamma}}} \frac{i^{-1-\gamma}}{i^{-1-\gamma} + \kappa} + \sum_{i \ge \kappa^{-\frac{1}{1+\gamma}}} \frac{i^{-1-\gamma}}{i^{-1-\gamma} + \kappa} \\
&\approx \sum_{i \le \kappa^{-\frac{1}{1+\gamma}}} 1 + \kappa^{-1} \sum_{i \ge \kappa^{-\frac{1}{1+\gamma}}} i^{-1-\gamma} \\
&\approx  \kappa^{-\frac{1}{1+\gamma}} + \kappa^{-1} \kappa^{\frac{\gamma}{1+\gamma}} \\
&\approx  \kappa^{-\frac{1}{1+\gamma}}.
\end{align*}

To prove the eighth statement, observe that:
\begin{align*}
\sum_{i=1}^P \frac{i^{-1-\gamma}}{(i^{-1-\gamma} + \kappa)^2} &=  \sum_{i \le \kappa^{-\frac{1}{1+\gamma}}} \frac{i^{-1-\gamma}}{(i^{-1-\gamma} + \kappa)^2} + \sum_{i \ge \kappa^{-\frac{1}{1+\gamma}}} \frac{i^{-\delta-1-\gamma}}{(i^{-1-\gamma} + \kappa)^2} \\
&\approx \sum_{i \le \kappa^{-\frac{1}{1+\gamma}}} i^{1+\gamma} + \kappa^{-2} \sum_{i \ge \kappa^{-\frac{1}{1+\gamma}}} i^{-1-\gamma} \\
&\approx \max(1, \kappa^{-\frac{2+\gamma}{1+\gamma}}) + \kappa^{-2} \kappa^{\frac{\gamma}{1+\gamma}} \\
&= \kappa^{-2} \max(\kappa^2, \kappa^{\frac{\gamma}{1+\gamma}}) + \kappa^{-2} \kappa^{\frac{\gamma}{1+\gamma}} \\
&\approx \kappa^{-2} \max(\kappa^2, \kappa^{\frac{\gamma}{1+\gamma}}) \\
&\approx \kappa^{-2} \kappa^{\frac{\gamma}{1+\gamma}}
\end{align*}
\end{proof}

\begin{lemma}
\label{lemma:degreesoffreedom}
Suppose that the power-law scaling holds for the eigenvalues and alignment coefficients with scaling exponents $\gamma, \delta > 0$ and correlation coefficient $\rho \in [0, 1)$, suppose that $P = \infty$. Assume the notation from Lemma \ref{lemma:sollichvariant}, and similarly let
\[Q := 1 - \frac{1}{N} \Tr(\Sigma^2 \Sigma_{\kappa} ^{-2}). \]
Then it holds that $Q^{-1} = \Theta(1)$. 
\end{lemma}
\begin{proof}
Let $\Sigma = V \Lambda V^T$ be the eigendecomposition of $\Sigma$, where $\Lambda$ is a diagonal matrix consisting of the eigenvalues. By Definition \ref{def:effectiveregularizer}, we see that:
\[ \frac{\lambda}{\kappa} + \frac{1}{N} \Tr(\Sigma \Sigma_{\kappa} ^{-1}) = 1.\]
This implies that:
\begin{align*}
 Q &= 1- \frac{1}{N} \Tr(\Sigma \Sigma_{\kappa} ^{-1}) + \frac{1}{N} \left( \Tr(\Sigma \Sigma_{\kappa} ^{-1}) - \Tr(\Sigma^2 \Sigma_{\kappa}^{-2}) \right) \\
 &= \frac{\lambda}{\kappa} + \frac{1}{N} \left( \Tr(\Sigma \Sigma_{\kappa} ^{-1}) - \Tr(\Sigma^2 \Sigma_{\kappa}^{-2}) \right).  
\end{align*}
Observe that:
\begin{align*}
 \Tr(\Sigma \Sigma_{\kappa} ^{-1}) - \Tr(\Sigma^2 \Sigma_{\kappa} ^{-2}) &=  \Tr(\Lambda (\Lambda + \kappa I)^{-1}) - \Tr(\Lambda^2 (\Lambda + \kappa I)^{-2}) \\
 &= \sum_{i=1}^P \left(\frac{i^{-1-\gamma}}{i^{-1-\gamma} + \kappa} - \frac{i^{-2-\gamma}}{(i^{-1-\gamma} + \kappa)^2} \right) \\
 &= \kappa \sum_{i=1}^P \frac{i^{-1-\gamma}}{(i^{-1-\gamma} + \kappa)^2}.
\end{align*}
This means that:
\begin{align*}
Q &= \frac{\lambda}{\kappa} + \frac{\kappa}{N} \sum_{i=1}^P \frac{i^{-1-\gamma}}{(i^{-1-\gamma} + \kappa)^2} \\
&\approx_{(A)} \frac{\lambda}{\kappa} + \Theta(\left(\frac{\kappa}{N} \kappa^{-2} \kappa^{\frac{\gamma}{1+\gamma}}\right)) \\
&=\frac{\lambda}{\kappa} +  \Theta\left(\frac{\kappa^{-\frac{1}{1+\gamma}}}{N}\right).
\end{align*}
where (A) uses Lemma \ref{lemma:splitintegralbounds}. 

\paragraph{Case 1: $\kappa = \Theta(\lambda)$.} In this case, we see that 
\[Q = \frac{\lambda}{\kappa} +  \Theta\left(\frac{\kappa^{-\frac{1}{1+\gamma}}}{N}\right) = \Theta(1).\]
This means that  $Q^{-1} = \Theta(1)$. 

\paragraph{Case 2: $\kappa = \Theta(N^{-1-\gamma})$.} In this case, we see that 
\[Q = \frac{\lambda}{\kappa} +  \Theta\left(\frac{\kappa^{-\frac{1}{1+\gamma}}}{N}\right) = \Omega\left(\frac{\kappa^{-\frac{1}{1+\gamma}}}{N}\right) = \Omega(1).\]
This means that  $Q^{-1} = \Theta(1)$. 

\end{proof}

\begin{lemma}
\label{lemma:kappabasic}
Suppose that power-law scaling holds for the eigenvalues with scaling exponent $\gamma$, and suppose that $P = \infty$. Then it holds that $\kappa(\lambda, M, \Sigma) = \Theta(\max(\lambda, M^{-1-\gamma}))$. 
\end{lemma}
\begin{proof}
Let $\Sigma = V \Lambda V^T$ be the eigendecomposition of $\Sigma$, where $\Lambda$ is a diagonal matrix consisting of the eigenvalues. 
Observe that:
\begin{align*}
 \Tr((\Sigma + \kappa I)^{-1} \Sigma) &= \Tr(\Lambda (\Lambda + \kappa I)^{-1}) \\
 &= \sum_{i=1}^P \frac{i^{-1-\gamma}}{i^{-1-\gamma} + \kappa} \\
 &\approx_{(A)} \kappa^{-\frac{1}{1+\gamma}}.
\end{align*}
where (A) follows from Lemma \ref{lemma:splitintegralbounds}. Using Definition \ref{def:effectiveregularizer}, we see that for $\kappa = \kappa(\lambda, M, \Sigma)$, it holds that: 
\[ \frac{\lambda}{\kappa} + \frac{1}{M} \Theta(\kappa^{-1-\gamma}) = 1.\]
This implies that $\kappa = \Theta(\max(\lambda, M^{-1-\gamma}))$ as desired. 
\end{proof}

\section{Machinery from random matrix theory}\label{appendix:machinery}

In this section, we introduce machinery from random matrix theory that serves as the backbone for our analysis of multi-objective scaling laws in Appendix \ref{appendix:proofsmultiobjective}. In Appendix \ref{appendix:MP}, we give a recap of known Marčenko-Pastur properties. In  Appendix \ref{appendix:RMTextended}, we use these known properties to derive random matrix theory results which are tailored to our analysis.  

\subsection{Recap of Marčenko-Pastur properties}\label{appendix:MP}

We introduce Marčenko-Pastur properties, following the treatment in \citet{bach}. Informally speaking, Marčenko-Pastur laws show that a random matrix  $(\hat{\Sigma} + \lambda I)^{-1}$ (where $\hat{\Sigma}$ is a sample covariance) behaves similarly to a deterministic matrix of the form $(\hat{\Sigma} + \kappa I)^{-1}$, where $\kappa = \kappa(\lambda, M, \Sigma)$ is an \textit{effective regularizer}. 

Deriving this formally requires placing several structural assumptions on number of data points $N \ge 1$, the number of parameters $P \ge 1$, the distribution $\DF$, and the vectors $\beta_1$ and $\beta_2$. We adopt assumptions from \citet{bach} which guarantee that a Marčenko-Pastur law holds for $\Sigma$, and we further introduce a boundedness assumption for technical reasons. 
\begin{assumption}
\label{assumption:MP}
We assume that: (1) $X \sim \DF$ takes the form $X = Z \Sigma^{1/2} $ where $Z$ has bounded subgaussian i.i.d components with mean zero and unit variance, (2) $N$ and $P$ approach $\infty$ with $\frac{P}{N}$ tending to $\gamma > 0$, (3) the spectral measure $\frac{1}{P} \sum_{i=1}^P \delta_{\lambda_i}$ of $\Sigma$ converges to a probability measure with compact support, and $\Sigma$ is invertible and bounded in operator norm, and (4) for $j \in \left\{1,2\right\}$, the measure $\sum_{i=1}^P \langle v_i, \beta_j \rangle^2$ converges to a measure with bounded mass, and $\beta_j$ has bounded $\ell_2$ norm. 
\end{assumption}

The effective regularizer $\kappa(\lambda, M, \Sigma)$ is defined as follows. 
\begin{definition}[Effective regularizer]
\label{def:effectiveregularizer}
For $\lambda \ge 0$, $M \ge 1$, and a $P$-dimensional positive semidefinite matrix $\Sigma$ with eigenvalues $\lambda_i$ for $1 \le i \le P$, the value $\kappa(\lambda, M, \Sigma)$ is the unique value $\kappa \ge 0$ such that:
\[\frac{\lambda}{\kappa} + \frac{1}{N} \sum_{i=1}^P \frac{\lambda_i}{\lambda_i + \kappa} = 1. \]
\end{definition}

We are now ready to state the key random matrix theory results proven in \citet{bach}. Following \citet{bach}, the asymptotic equivalence notation $u \sim v$ means that $u/v$ tends to $1$ as $N$ and $P$ go to $\infty$.
\begin{lemma}[Restatement of Proposition 1 in \citet{bach}]
\label{lemma:bach}
Let $\hat{\Sigma} = \frac{1}{M} \sum_{i=1}^M X_i X_i^T$ be the sample covariance matrix from $M$ i.i.d. samples from $X_1, \ldots, X_M \sim \DF$. Let $\kappa = \kappa(\lambda, N, \Sigma)$. Suppose that $A$ and $B$ have bounded operator norm. 
Then it holds that:
\begin{align*}
\Tr\left((\hat{\Sigma} + \lambda I)^{-1} A \right) &\sim \kappa \Tr\left((\Sigma + \kappa I)^{-1} A \right) \\
\kappa^2  \Tr\left((\hat{\Sigma} + \lambda I)^{-1} A (\hat{\Sigma} + \lambda I)^{-1} B \right) &\sim \Tr\left((\Sigma + \kappa I)^{-1} A (\Sigma + \kappa I)^{-1} B \right) \\
&+ \kappa^2 \frac{\frac{1}{N} \Tr\left(A \Sigma (\Sigma + \kappa I)^{-2} \right)}{1 - \frac{1}{N} \Tr\left(\Sigma^2 (\Sigma + \kappa I)^{-2} \right)} \Tr\left((\Sigma + \kappa I)^{-1} \Sigma (\Sigma + \kappa I)^{-1} B \right).
\end{align*}
\end{lemma}

We note that the requirement that $B$ has bounded operator norm in Lemma \ref{lemma:bach} is what forces us to require that $\|\beta_1\|$ and $\|\beta_2\|$ are bounded. However, \citet{WHS22} showed that the norm can be unbounded in several real-world settings, and thus instead opt to assume a local Marčenko-Pastur law and derive scaling laws based on this assumption. We suspect it may be possible to derive our scaling law with an appropriate analogue of the local Marčenko-Pastur law, which would also have the added benefit of allowing one to relax other requirements in Assumption \ref{assumption:MP} such as gaussianity. We view such an extension as an interesting direction for future work. 

\subsection{Useful random matrix theory facts}\label{appendix:RMTextended}

We derive several corollaries of Lemma \ref{lemma:bach} tailored to random matrices that arise in our analysis of multi-objective scaling laws. 

\begin{lemma}
\label{lemma:usefulmatrixproperties}
Assume that $\DF$ satisfies the Marčenko-Pastur property (Assumption \ref{assumption:MP}). Let $Z$ be a positive definite matrix such that $Z^{-1}$ has bounded operator norm, and let $A$ be a matrix with bounded operator norm. Let $\hat{\Sigma} = \frac{1}{M} \sum_{i=1}^M X_i X_i^T$ be the sample covariance matrix from $M$ i.i.d. samples from $X_1, \ldots, X_M \sim \DF$. Then it holds that: 
\begin{equation}
\label{eq:MPregular}
\lambda \cdot \Tr((\hat{\Sigma}  + \lambda Z)^{-1} A) \sim \kappa \cdot \Tr((\Sigma + \kappa Z)^{-1} A). 
\end{equation}
If $A$ also has bounded trace and $Z$ has bounded operator norm, then it holds that: 
\begin{equation}
\label{eq:MPreformulated}
\Tr(\hat{\Sigma} (\hat{\Sigma} + \lambda Z)^{-1} A) \sim \Tr(\Sigma \cdot (\Sigma + \kappa Z)^{-1} A) 
\end{equation}
where $\kappa = \kappa(\lambda, M, Z^{-1/2} \Sigma Z^{-1/2})$. 
\end{lemma}
\begin{proof}

For \eqref{eq:MPregular}, observe that:
\begin{align*}
 \lambda \cdot \Tr((\hat{\Sigma} + \lambda Z)^{-1} A) 
 &=  \lambda \cdot \Tr(Z^{-1/2} (Z^{-1/2} \hat{\Sigma} Z^{-1/2}  + \lambda I)^{-1} Z^{-1/2} A) \\
 &=  \lambda \cdot \Tr((Z^{-1/2} \hat{\Sigma} Z^{-1/2}  + \lambda I)^{-1} Z^{-1/2} A Z^{-1/2}) \\
 &\sim_{(A)} \kappa \cdot \Tr((Z^{-1/2} \Sigma Z^{-1/2}  + \kappa I)^{-1} Z^{-1/2} A Z^{-1/2}) \\
 &= \kappa \cdot \Tr(Z^{-1/2} (Z^{-1/2}  \Sigma Z^{-1/2}  + \kappa I)^{-1} Z^{-1/2} A) \\
&= \kappa \cdot \Tr((\Sigma+ \kappa Z)^{-1} A). 
\end{align*}
where (A) applies Lemma \ref{lemma:bach} (using the fact that since $A$ and $Z^{-1}$ have bounded operator norm, it holds that $Z^{-1/2} A Z^{-1/2}$ has bounded operator norm). 

For \eqref{eq:MPreformulated}, observe that:
\begin{align*}
 \Tr(\hat{\Sigma} (\hat{\Sigma} + \lambda Z)^{-1} A) &=_{(A)}  \Tr\left(\left(I - \lambda  Z^{1/2} \left( Z^{-1/2} \hat{\Sigma} Z^{-1/2} + \lambda I\right)^{-1} Z^{-1/2}\right) A \right) \\
  &=_{(B)} \Tr(A) - \lambda \cdot \Tr\left(\left( Z^{-1/2} \hat{\Sigma} Z^{-1/2} + \lambda I\right)^{-1} Z^{-1/2} A Z^{1/2} \right) \\
  &\sim_{(C)} \Tr(A) - \kappa \cdot \Tr\left(\left( Z^{-1/2} \Sigma Z^{-1/2} + \kappa I\right)^{-1} Z^{-1/2} A Z^{1/2} \right) \\
 &=_{(D)} \Tr\left(\left(I - \kappa  Z^{1/2} \left( Z^{-1/2} \Sigma Z^{-1/2} + \kappa I\right)^{-1} Z^{-1/2}\right) A \right) \\
&=_{(E)} \Tr(\Sigma (\Sigma + \kappa Z)^{-1} A)
\end{align*}
where (A) and (E) follows from Claim \ref{claim:difference}, (B) and (D) use the fact that $\Tr(A)$ is bounded, and (C) follows from Lemma \ref{lemma:bach} (using the fact that since $A$, $Z$, and $Z^{-1}$ have bounded operator norm, it holds that $Z^{-1/2} A Z^{1/2}$ has bounded operator norm).

\end{proof}

\begin{lemma}
\label{lemma:quadraticformbasic}
Assume that $\DF$ satisfies the Marčenko-Pastur property (Assumption \ref{assumption:MP}). Let $Z$ be any positive definite matrix such that $Z$ and $Z^{-1}$ have bounded operator norm, and let $A$ and $B$ have bounded operator norm. Let $\hat{\Sigma} = \frac{1}{M} \sum_{i=1}^M X_i X_i^T$ be the sample covariance matrix from $M$ i.i.d. samples from $X_1, \ldots, X_M \sim \DF$. Then it holds that:
\begin{align*}
 &\lambda^2 \Tr((\hat{\Sigma}  + \lambda Z)^{-1} A (\hat{\Sigma}  + \lambda Z)^{-1} B) \\
 &= \lambda^2  \Tr(Z^{-1/2} (Z^{-1/2}\hat{\Sigma} Z^{-1/2} + \lambda I)^{-1} Z^{-1/2} A Z^{-1/2}  (Z^{-1/2}\hat{\Sigma} Z^{-1/2} + \lambda I)^{-1} B)  \\
 &\sim \kappa^2 \Tr((\Sigma  + \kappa Z)^{-1} A (\Sigma  + \kappa Z)^{-1} B) \\
 &+ \kappa^2\frac{\frac{1}{M} \Tr((\Sigma + \kappa Z)^{-1} \Sigma (\Sigma + \kappa Z)^{-1}  A)}{1 - \frac{1}{M} \Tr((\Sigma + \kappa Z)^{-1} \Sigma (\Sigma + \kappa Z)^{-1}  \Sigma)} \Tr((\Sigma  + \kappa Z)^{-1} \Sigma (\Sigma  + \kappa Z)^{-1} B) 
\end{align*}
where $\kappa = \kappa(\lambda, M, Z^{-1/2} \Sigma Z^{-1/2})$. 
\end{lemma}
\begin{proof}

Let $q = \frac{\frac{1}{M} \Tr(Z^{-1/2} \Sigma Z^{-1/2} (Z^{-1/2} \Sigma Z^{-1/2} + \kappa I)^{-2} Z^{-1/2} A Z^{-1/2})}{1 - \frac{1}{M} \Tr(Z^{-1/2} \Sigma Z^{-1/2} (Z^{-1/2} \Sigma Z^{-1/2} + \kappa I)^{-2} Z^{-1/2} \Sigma Z^{-1/2})}$. 

Observe that: 
\begin{align*}
&\lambda^2 \Tr((\hat{\Sigma}  + \lambda Z)^{-1} A (\hat{\Sigma}  + \lambda Z)^{-1} B) \\
&\lambda^2 \Tr\left(Z^{-1/2} \left( Z^{-1/2} \hat{\Sigma} Z^{-1/2} + \lambda I\right)^{-1} Z^{-1/2} A Z^{-1/2} \left( Z^{-1/2} \hat{\Sigma} Z^{-1/2} + \lambda I\right)^{-1} Z^{-1/2} B\right)  \\
&= \lambda^2 \Tr\left(\left( Z^{-1/2} \hat{\Sigma} Z^{-1/2} + \lambda I\right)^{-1} Z^{-1/2} A Z^{-1/2} \left( Z^{-1/2} \hat{\Sigma} Z^{-1/2} + \lambda I\right)^{-1} Z^{-1/2} B Z^{-1/2}\right)  \\
&\sim_{(A)} \kappa^2 \Tr\left(\left( Z^{-1/2} \Sigma Z^{-1/2} + \kappa I\right)^{-1} Z^{-1/2} A Z^{-1/2} \left( Z^{-1/2} \Sigma Z^{-1/2} + \kappa I\right)^{-1} Z^{-1/2} B Z^{-1/2}\right)  \\
&+ \kappa^2 q  \Tr\left(\left( Z^{-1/2} \Sigma  Z^{-1/2} + \kappa I\right)^{-1} Z^{-1/2} \Sigma  Z^{-1/2} \left( Z^{-1/2} \Sigma Z^{-1/2} + \kappa I\right)^{-1} Z^{-1/2} B Z^{-1/2}\right) \\
&= \kappa^2 \Tr\left(Z^{-1/2} \left( Z^{-1/2} \Sigma Z^{-1/2} + \kappa I\right)^{-1} Z^{-1/2} A Z^{-1/2} \left( Z^{-1/2} \Sigma Z^{-1/2} + \kappa I\right)^{-1} Z^{-1/2} B\right)  \\
&+ \kappa^2 q  \Tr\left(Z^{-1/2} \left( Z^{-1/2} \Sigma  Z^{-1/2} + \kappa I\right)^{-1} Z^{-1/2} \Sigma  Z^{-1/2} \left( Z^{-1/2} \Sigma Z^{-1/2} + \kappa I\right)^{-1} Z^{-1/2} B \right) \\
&= \kappa^2 \Tr\left(\left(\Sigma + \kappa Z \right)^{-1} A \left(\Sigma + \kappa Z \right)^{-1} B\right)  + q \kappa^2 \Tr\left(\left(\Sigma + \kappa Z \right)^{-1} \Sigma \left(\Sigma + \kappa Z \right)^{-1} B \right),  
\end{align*}
where (A) follows from Lemma \ref{lemma:bach} (using the fact that since $A$, $B$, $Z$, and $Z^{-1}$ have bounded operator norm, it holds that $Z^{-1/2} A Z^{1/2}$, $\Sigma$, and $Z^{-1/2} B Z^{1/2}$ have bounded operator norm).

We can simplify $q$ as follows: 
\begin{align*}
q &= \frac{\frac{1}{M} \Tr(Z^{-1/2} \Sigma Z^{-1/2} (Z^{-1/2} \Sigma Z^{-1/2} + \kappa I)^{-2} Z^{-1/2} A Z^{-1/2})}{1 - \frac{1}{M} \Tr(Z^{-1/2} \Sigma Z^{-1/2} (Z^{-1/2} \Sigma Z^{-1/2} + \kappa I)^{-2} Z^{-1/2} \Sigma Z^{-1/2})}  \\
&= \frac{\frac{1}{M} \Tr((Z^{-1/2} \Sigma Z^{-1/2} + \kappa I)^{-1} 
 Z^{-1/2} \Sigma Z^{-1/2} (Z^{-1/2} \Sigma Z^{-1/2} + \kappa I)^{-1} Z^{-1/2} A Z^{-1/2})}{1 - \frac{1}{M} \Tr((Z^{-1/2} \Sigma Z^{-1/2} + \kappa I)^{-1} 
 Z^{-1/2} \Sigma Z^{-1/2} (Z^{-1/2} \Sigma Z^{-1/2} + \kappa I)^{-1} Z^{-1/2} \Sigma Z^{-1/2})}  \\
 &= \frac{\frac{1}{M} \Tr(Z^{-1/2}(Z^{-1/2} \Sigma Z^{-1/2} + \kappa I)^{-1} 
 Z^{-1/2} \Sigma Z^{-1/2} (Z^{-1/2} \Sigma Z^{-1/2} + \kappa I)^{-1} Z^{-1/2} A)}{1 - \frac{1}{M} \Tr(Z^{-1/2}(Z^{-1/2} \Sigma Z^{-1/2} + \kappa I)^{-1} 
 Z^{-1/2} \Sigma Z^{-1/2} (Z^{-1/2} \Sigma Z^{-1/2} + \kappa I)^{-1} Z^{-1/2} \Sigma)}  \\
&= \frac{\frac{1}{M} \Tr((\Sigma + \kappa Z)^{-1} \Sigma (\Sigma + \kappa Z)^{-1}  A)}{1 - \frac{1}{M} \Tr((\Sigma + \kappa Z)^{-1} \Sigma (\Sigma + \kappa Z)^{-1}  \Sigma)}.
\end{align*}
\end{proof}

\begin{lemma}
\label{lemma:doublequadraticformMP}
Assume that $\DF$ satisfies the Marčenko-Pastur property (Assumption \ref{assumption:MP}). Let $Z$ be any positive definite matrix such that $Z$ and $Z^{-1}$ have bounded operator norm. Let $A$ and $B$ have bounded operator norm, and suppose also that $\Tr(A B)$ is bounded. Let $\hat{\Sigma} = \frac{1}{M} \sum_{i=1}^M X_i X_i^T$ be the sample covariance matrix from $M$ i.i.d. samples from $X_1, \ldots, X_M \sim \DF$.  Then it holds that: 
\begin{equation}
\label{eq:MPquadratic}
\Tr(\hat{\Sigma} (\hat{\Sigma} + \lambda Z)^{-1} A (\hat{\Sigma} +  \lambda Z)^{-1} \hat{\Sigma} B) \sim \Tr(\Sigma (\Sigma + \kappa Z)^{-1}  A (\Sigma + \kappa Z)^{-1} \Sigma B) + E, 
\end{equation}
where:
\begin{align*}
E &:= \frac{\frac{1}{M} \Tr((\Sigma + \kappa Z)^{-1} \Sigma (\Sigma + \kappa Z)^{-1}  A)}{1 - \frac{1}{M} \Tr((\Sigma + \kappa Z)^{-1} \Sigma (\Sigma + \kappa Z)^{-1}  \Sigma)} \cdot \kappa^2 \Tr \left(\left( \Sigma  + \kappa Z \right)^{-1} \Sigma \left( \Sigma + \kappa Z \right)^{-1} Z B Z \right),    
\end{align*}
and $\kappa = \kappa(\lambda, M, Z^{-1/2} \Sigma Z^{-1/2})$. 
\end{lemma}
\begin{proof}

Observe that: 
\begin{align*}
& \Tr(\hat{\Sigma} (\hat{\Sigma} + \lambda Z)^{-1}  A (\hat{\Sigma} + \lambda Z)^{-1} \hat{\Sigma} B) \\
&= \Tr(\hat{\Sigma} (\hat{\Sigma} + \lambda Z)^{-1}  A  \left(\hat{\Sigma} (\hat{\Sigma} + \lambda Z)^{-1}\right)^T B) \\
&=_{(A)} \Tr\left(\left(I - \lambda  Z^{1/2} \left( Z^{-1/2} \hat{\Sigma} Z^{-1/2} + \lambda I\right)^{-1} Z^{-1/2} \right) A \left(I - \lambda  Z^{1/2} \left( Z^{-1/2} \hat{\Sigma} Z^{-1/2} + \lambda I\right)^{-1} Z^{-1/2}\right)^T B \right) \\
&=_{(B)} \Tr(A B) - \lambda   \Tr\left(A \left(Z^{1/2} \left( Z^{-1/2} \hat{\Sigma} Z^{-1/2} + \lambda I\right)^{-1} Z^{-1/2}\right)^T B \right) \\
&- \lambda \Tr\left(Z^{1/2}\left( Z^{-1/2} \hat{\Sigma} Z^{-1/2} + \lambda I\right)^{-1} Z^{-1/2} A B \right) \\
&+ \lambda^2 \Tr\left(Z^{1/2}\left( Z^{-1/2} \hat{\Sigma} Z^{-1/2} + \lambda I\right)^{-1} Z^{-1/2} A \left( Z^{1/2} \left( Z^{-1/2} \hat{\Sigma} Z^{-1/2} + \lambda I\right)^{-1} Z^{-1/2}\right)^T B \right)  \\
&= \Tr(A B) - \lambda   \Tr\left(A Z^{-1/2} \left( Z^{-1/2} \hat{\Sigma} Z^{-1/2} + \lambda I\right)^{-1} Z^{1/2} B \right) \\
&- \lambda \Tr\left(Z^{1/2}\left( Z^{-1/2} \hat{\Sigma} Z^{-1/2} + \lambda I\right)^{-1} Z^{-1/2} A B \right) \\
&+ \lambda^2 \Tr\left(Z^{1/2}\left( Z^{-1/2} \hat{\Sigma} Z^{-1/2} + \lambda I\right)^{-1} Z^{-1/2} A Z^{-1/2} \left( Z^{-1/2} \hat{\Sigma} Z^{-1/2} + \lambda I\right)^{-1} Z^{1/2}  B \right)  \\
&= \Tr(A B) -  \underbrace{\lambda  \Tr\left(\left( \hat{\Sigma}  + \lambda Z \right)^{-1}  Z B A \right)}_{(1)} - \underbrace{\lambda \Tr\left(\left( \hat{\Sigma} + \lambda Z \right)^{-1} A B Z \right)}_{(2)} \\
&+ \underbrace{\lambda^2 \Tr\left(\left(\hat{\Sigma}  + \lambda Z \right)^{-1} A \left(\hat{\Sigma}  + \lambda Z \right)^{-1}  Z B Z \right)}_{(3)}  \\
\end{align*}
where (A) follows from Claim \ref{claim:difference}, (B) uses that $\Tr(AB)$ is bounded,  

For term (1) and term (2), we apply Lemma \ref{lemma:usefulmatrixproperties} to see that: 
\begin{align*}
\lambda  \Tr\left(\left( \hat{\Sigma}  + \lambda Z \right)^{-1}  Z B A \right)  &\sim \kappa   \lambda  \Tr\left(\left( \Sigma + \kappa Z \right)^{-1}  Z B A \right) \\
\lambda \Tr\left(\left( \hat{\Sigma} + \lambda Z \right)^{-1} A B Z \right)&\sim \kappa \Tr\left(\left( \Sigma + \kappa Z \right)^{-1} A B Z \right).
\end{align*}

For term (3), we apply Lemma \ref{lemma:quadraticformbasic} to see that 
\begin{align*}
&\lambda^2 \Tr\left(\left(\hat{\Sigma}  + \lambda Z \right)^{-1} A \left(\hat{\Sigma}  + \lambda Z \right)^{-1}  Z B Z \right)  \\
&\sim \kappa^2 \Tr\left(\left( \Sigma + \kappa Z \right)^{-1} A \left( \Sigma + \kappa Z \right)^{-1} Z B Z \right)  \\
&+ \kappa^2 \frac{\frac{1}{M} \Tr((\Sigma + \kappa Z)^{-1} \Sigma (\Sigma + \kappa Z)^{-1}  A)}{1 - \frac{1}{M} \Tr((\Sigma + \kappa Z)^{-1} \Sigma (\Sigma + \kappa Z)^{-1}  \Sigma)}\Tr\left((\Sigma  + \kappa Z)^{-1} \Sigma (\Sigma  + \kappa Z)^{-1} Z B Z \right) \\
&\sim \kappa^2 \Tr\left(\left( \Sigma + \kappa Z \right)^{-1} A \left( \Sigma + \kappa Z \right)^{-1} Z B Z \right)  + E
\end{align*}

This means that: 
\begin{align*}
\Tr\left(\hat{\Sigma} (\hat{\Sigma} + \lambda Z)^{-1}  A (\hat{\Sigma} + \lambda Z)^{-1} \hat{\Sigma}\right) &\sim \Tr(A B) -  \kappa  \Tr\left(\left(\Sigma + \kappa Z \right)^{-1}  Z B A \right) - \kappa \Tr\left(\left( \Sigma + \kappa Z \right)^{-1} A B Z \right) \\
&+ \kappa^2 \Tr\left(\left(\Sigma  + \kappa Z \right)^{-1} A \left(\Sigma + \kappa Z \right)^{-1}  Z B Z \right) + E \\
&=_{(C)} \Tr(\Sigma(\Sigma + \kappa Z)^{-1}  A  (\Sigma + \kappa Z)^{-1} \Sigma B) + E,
\end{align*}
where (C) uses an analogous analysis to the beginning of the proof. 
\end{proof}

\begin{lemma}
\label{lemma:singlequadraticformMP}
Assume that $\DF$ satisfies the Marčenko-Pastur property (Assumption \ref{assumption:MP}). Let $Z$ be any positive definite matrix such that $Z$ and $Z^{-1}$ have bounded operator norm, and let $A$ and $B$ have bounded operator norm. Let $\hat{\Sigma} = \frac{1}{M} \sum_{i=1}^M X_i X_i^T$ be the sample covariance matrix from $M$ i.i.d. samples from $X_1, \ldots, X_M \sim \DF$.  Then it holds that: 
\begin{equation}
\label{eq:MPquadraticsingle}
\lambda \Tr\left( (\hat{\Sigma} + \lambda Z)^{-1} A (\hat{\Sigma} +  \lambda Z)^{-1} \hat{\Sigma}  B \right) \sim  \kappa \Tr\left((\Sigma + \kappa Z)^{-1}  A (\Sigma + \kappa Z)^{-1} \Sigma  B \right)  - E, 
\end{equation}
where:
\begin{align*}
 E &:= \frac{\frac{1}{M} \Tr((\Sigma + \kappa Z)^{-1} \Sigma (\Sigma + \kappa Z)^{-1}  A)}{1 - \frac{1}{M} \Tr((\Sigma + \kappa Z)^{-1} \Sigma (\Sigma + \kappa Z)^{-1}  \Sigma)} \cdot \kappa^2 \Tr\left(\left( \Sigma  + \kappa Z \right)^{-1} \Sigma \left( \Sigma + \kappa Z \right)^{-1} Z B \right) 
\end{align*}
and $\kappa = \kappa(\lambda, N, Z^{-1/2} \Sigma Z^{-1/2})$. 
\end{lemma}
\begin{proof}
Observe that: 
\begin{align*}
&\lambda \Tr\left((\hat{\Sigma} + \lambda Z)^{-1}  A (\hat{\Sigma} + \lambda Z)^{-1} \hat{\Sigma}  B \right) \\
&=_{(A)}  \lambda \Tr\left(Z^{-1/2}(Z^{-1/2} \hat{\Sigma} Z^{-1/2} + \lambda I)^{-1} Z^{-1/2} A  \left(I - \lambda  Z^{-1/2} \left( Z^{-1/2} \hat{\Sigma} Z^{-1/2} + \lambda I\right)^{-1} Z^{1/2} \right) B \right) \\
&= \lambda \Tr\left(Z^{-1/2} (Z^{-1/2} \hat{\Sigma} Z^{-1/2} + \lambda I)^{-1} Z^{-1/2} A B \right)  \\
&- \lambda^2 \Tr\left( Z^{-1/2}\left( Z^{-1/2} \hat{\Sigma} Z^{-1/2} + \lambda I\right)^{-1} Z^{-1/2} A Z^{-1/2} (Z^{-1/2} \hat{\Sigma} Z^{-1/2} + \lambda I)^{-1} Z^{1/2} B \right) \\
&= \underbrace{\lambda \Tr\left((\hat{\Sigma} + \lambda Z)^{-1} A B  \right)}_{(1)} - \underbrace{\lambda^2  \Tr\left(\left( \hat{\Sigma}  + \lambda  Z \right)^{-1} A  (\hat{\Sigma} + \lambda Z)^{-1} Z B \right)}_{(2)} \\
\end{align*}
where (A) follows from Claim \ref{claim:difference}. 

For term (1), we apply Lemma \ref{lemma:usefulmatrixproperties} see that:
\begin{align*}
\lambda \Tr\left((\hat{\Sigma} + \lambda Z)^{-1} A B \right) &\sim \kappa \Tr\left((\Sigma + \kappa Z)^{-1} A B \right).
\end{align*}

For term (2), we apply Lemma \ref{lemma:quadraticformbasic} to see that 
\begin{align*}
&\lambda^2  \Tr\left(\left( \hat{\Sigma}  + \lambda  Z \right)^{-1} A  (\hat{\Sigma} + \lambda Z)^{-1} Z B  \right) \\
&\sim \kappa^2 \Tr\left(\left( \Sigma + \kappa Z \right)^{-1} A \left( \Sigma + \kappa Z \right)^{-1} Z B \right)  \\
&+ \kappa^2 \frac{\frac{1}{M} \Tr((\Sigma + \kappa Z)^{-1} \Sigma (\Sigma + \kappa Z)^{-1}  A)}{1 - \frac{1}{M} \Tr((\Sigma + \kappa Z)^{-1} \Sigma (\Sigma + \kappa Z)^{-1}  \Sigma)}\Tr\left((\Sigma  + \kappa Z)^{-1} \Sigma (\Sigma + \kappa Z)^{-1} Z B  \right) \\
&\sim \kappa^2 \Tr\left(\left( \Sigma + \kappa Z \right)^{-1} A \left( \Sigma + \kappa Z \right)^{-1} Z B \right)  + E.
\end{align*}

This means that: 
\begin{align*}
&\lambda \Tr\left((\hat{\Sigma} + \lambda Z)^{-1}  A (\hat{\Sigma} + \lambda Z)^{-1} \hat{\Sigma}  B \right) \\
&\sim \kappa \Tr\left((\Sigma + \kappa Z)^{-1} A B \right)  + \kappa^2 \Tr\left(\left( \Sigma + \kappa Z \right)^{-1} A \left( \Sigma + \kappa Z \right)^{-1} Z B \right)  - E \\
&= \kappa \Tr\left(Z^{-1/2} (Z^{-1/2} \Sigma Z^{-1/2} + \kappa I)^{-1} Z^{-1/2}  A  \left(I - \kappa  Z^{1/2} \left( Z^{-1/2} \Sigma Z^{-1/2} + \kappa I\right)^{-1} Z^{-1/2} \right) B \right) - E \\
&=_{(A)} \kappa \Tr\left((\Sigma + \kappa Z)^{-1}  A (\Sigma + \kappa Z)^{-1} \Sigma  B \right)  - E, 
\end{align*}
where (A) uses an analogous analysis to the beginning of the proof. 
\end{proof}

The proofs of these results relied on the following basic matrix fact.
\begin{claim}
\label{claim:difference} 
Let $A$ be any matrix and let $B$ be any symmetric positive definite matrix. Then it holds that:
\[ A (A + \lambda B)^{-1} = I - \lambda B^{1/2} \left( B^{-1/2} A B^{-1/2} + \lambda I \right)^{-1} B^{-1/2}. \]
\end{claim}
\begin{proof}
 Observe that:
\begin{align*}
&A (A + \lambda B)^{-1}  \\
 &= A B^{-1/2} \left( B^{-1/2} A B^{-1/2} + \lambda I\right)^{-1} B^{-1/2} \\
 &= B^{1/2} \left(B^{-1/2} A B^{-1/2} \right) \left( B^{-1/2} A B^{-1/2} + \lambda I\right)^{-1} B^{-1/2} \\
 &=  B^{1/2} \left(B^{-1/2} A B^{-1/2} + \lambda I \right) \left( B^{-1/2} A B^{-1/2} + \lambda I\right)^{-1} B^{-1/2} - B^{1/2} \lambda \left( B^{-1/2} A B^{-1/2} + \lambda I\right)^{-1} B^{-1/2} \\
  &=  I - \lambda  B^{1/2} \left( B^{-1/2} A B^{-1/2} + \lambda I\right)^{-1} B^{-1/2}.
\end{align*}
\end{proof}

\section{Extension: Market-entry threshold with richer form for $L_2^*$}\label{appendix:extension}

In this section, we modify the safety requirement to take into account the impact of dataset size $N$ and regularization parameter $\lambda$, and we extend our model and analysis of the market-entry threshold accordingly. We show that the characterization in Theorem \ref{thm:tradeoffwarmup} directly applies to this setting, and we also show relaxed versions of Theorem \ref{thm:finitedata} and Theorem \ref{thm:alignment}. Altogether, these extended results illustrate that our qualitative insights from Sections \ref{sec:warmup}-\ref{sec:general} hold more generally. 

We define a modified approximation of the safety violation $\tilde{L}_2(\beta_1, \beta_2, \DF, \reg, N, \alpha)$. This modified approximation is defined analogously to $\LossPerf^*(\beta_1, \beta_2, \DF, \reg, N, \alpha)$. To formalize this, we define a deterministic equivalent $L_2^{\texttt{det}}$ for the safety violation to be  
\begin{equation}
\label{eq:deterministicequivalentsafety}
  L_2^{\texttt{det}}(\beta_1, \beta_2, \DF, \reg, N, \alpha) := L_1^{\texttt{det}}(\beta_2, \beta_1, \DF, \reg, N, 1-\alpha). 
\end{equation}
It follows from Lemma \ref{lemma:sollichvariant} that $L_2(\hat{\beta}(\alpha, \lambda, X)) \sim L_2^{\texttt{det}}(\beta_1, \beta_2, \DF, \reg, N, \alpha)$: here, we use the fact that $L_2(\hat{\beta}(\alpha, \lambda, X))$ is distributed identically to $L_1(\hat{\beta}(1-\alpha, \lambda, X))$. Now, using this deterministic equivalent, we define $\tilde{L}_2(\beta_1, \beta_2, \DF, \reg, N, \alpha) = L_2^{\texttt{det}}(\beta_1, \beta_2, \DF, \reg, N, \alpha)$.  

Using this formulation of $\tilde{L}_2$, we define a modified market entry threshold where we replace all instances of original approximation $L_2^*$ with the modified approximation $\tilde{L}_2$. In particular, a company $C$ faces reputational damage if: 
\[\mathbb{E}_{(\beta_1, \beta_2) \sim \DC} \tilde{L}_2(\beta_1, \beta_2, \DF, \alpha_C) \ge \tau_C.\]
The company selects $\mix \in [0.5,1]$ and $\reg \in (0,1)$ to maximize their performance subject to their safety constraint, as formalized by the following optimization program:\footnote{Unlike in Section \ref{sec:model}, there might not exist $\mix \in [0.5,1]$ and $\reg \in (0,1)$ which satisfy the safety constraint, if $N_C$ is too small.}
\begin{equation*}
(\tilde{\alpha}_C, \tilde{\lambda}_C) = \argmin_{\mix \in [0.5, 1], \reg \in (0,1)} \mathbb{E}_{\DC}[\LossPerf^*(\beta_1, \beta_2, \DF, \reg, N_C, \alpha)] \text{ s.t. } \mathbb{E}_{\DC}[\tilde{L}_2(\beta_1, \beta_2, \DF, \alpha)] \le \tau_C.
\end{equation*}
We define the modified market-entry threshold as follows. 
\begin{definition}
\label{def:modifiedmarketentrythreshold}
The \textit{modified market-entry threshold} $\Nentroptmodified(\Nlead, \constrlead, \constrentr, \DC, \DF)$
is the \textit{minimum value} of $\Nentr \in \mathbb{Z}_{\ge 1}$ such that $\mathbb{E}_{\DC}[\LossPerf^*(\beta_1, \beta_2, \DF, \tilde{\lambda}_E, \Nentr, \tilde{\alpha}_E)] \le \mathbb{E}_{\DC}[\LossPerf^*(\beta_1, \beta_2, \DF, \tilde{\lambda}_I, \Nlead, \tilde{\alpha}_I)]$.  
\end{definition}

In this section, we analyze the modified market entry threshold $\Nentroptmodified(\Nlead,\constrlead, \constrentr, \DC,  \DF)$. We show an extension of Theorem \ref{thm:tradeoffwarmup} (Appendix \ref{appendix:extensionwarmup}). We then derive a simplified version of the deterministic equivalent $L_2^{\texttt{det}}$ (Appendix \ref{appendix:simplifieddeterministic}). Finally, we show a weakened extension of Theorem \ref{thm:finitedata} (Appendix \ref{appendix:extensionfinitedata}) and a weakened extension of Theorem \ref{thm:alignment} (Appendix \ref{appendix:extensionalignment}). These weakened extensions derive upper bounds (rather than tight bounds) on the modified market entry threshold, and also assume that $\delta \le 1$. 

\subsection{Extension of Theorem \ref{thm:tradeoffwarmup}}\label{appendix:extensionwarmup}
We study the market entry $\Nentroptmodified$ threshold in the environment of Theorem \ref{thm:tradeoffwarmup} where the incumbent has infinite data and the new company faces no safety constraint. We show that the modified  market entry threshold takes the same form as the market entry threshold in Theorem \ref{thm:tradeoffwarmup}.  
\begin{theorem}[Extension of Theorem \ref{thm:tradeoffwarmup}]
\label{thm:extensiontradeoffwarmup}
Suppose that power-law scaling holds for the eigenvalues and alignment coefficients, with scaling exponents $\gamma, \delta > 0$ and correlation coefficient $\rho \in [0, 1)$, and suppose that $P = \infty$. 
Suppose that the incumbent company has infinite data (i.e., $\Nlead= \infty$), and that the entrant faces no constraint on their safety (i.e., $\constrentr = \infty$). Suppose that the safety constraint $\constrlead$ satisfies \eqref{eq:safetythreshold}. Then, it holds that:
\[\Nentroptmodified(\infty, \constrlead, \infty, \DC, \DF) = \Theta\left(\left(\sqrt{L^*(\rho)} - \sqrt{\min(\constrlead, L^*(\rho))} \right)^{-2/\scalingexp} \right),\]
where $L^*(\rho) = \mathbb{E}_{\DC}[(\beta_1 - \beta_2)^T \Sigma (\beta_1 - \beta_2)] = \Theta(1 - \rho)$, and where $\scalingexp := \min(2(1+\gamma), \delta + \gamma)$. 
\end{theorem}

Theorem \ref{thm:extensiontradeoffwarmup} shows that the qualitative insights from Theorem \ref{thm:tradeoffwarmup}---including that the new company can enter with finite data---readily extend to this setting.   

To prove Theorem \ref{thm:extensiontradeoffwarmup}, we build on the notation and analysis from Appendix \ref{appendix:proofssec3}. It suffices to show that each company $C$ will select $\mix_C = \tilde{\mix}_C$ and $\lambda_C = \tilde{\lambda}_C$. This follows trivially for the entrant $C = E$ since they face no safety constraint, and there is no different between the two settings. The key ingredient of the proof is to compute $\tilde{\mix}_I$ and $\tilde{\reg}_I$ for the incumbent (i.e., an analogue of Lemma \ref{lemma:ridgelessoptimalinfinitedata} in Appendix \ref{appendix:proofssec3}). 

To do this, we first upper bound the following function of the safety loss and performance loss for general parameters $\reg$ and $\mix$.
\begin{lemma}
\label{lemma:paretofrontiernew}
For any $\alpha$ and $\reg$, it holds that:
\[ \sqrt{\mathbb{E}_{\DC}[L_1(\beta(\alpha, \lambda))]} + \sqrt{\mathbb{E}_{\DC}[L_2(\beta(\alpha, \lambda))]} \ge \sqrt{\mathbb{E}_{\DC}[(\beta_1 - \beta_2)^T \Sigma (\beta_1 - \beta_2)^T]}.  \]
\end{lemma}
\begin{proof}
Note that:
\begin{align*}
T &:= \sqrt{\mathbb{E}_{\DC}[L_1(\beta(\alpha, \lambda))]} + \sqrt{\mathbb{E}_{\DC}[L_2(\beta(\alpha, \lambda))]} \\
&= \sqrt{(\beta_1 - \beta(\alpha, \lambda))^T \Sigma (\beta_1 - \beta(\alpha, \lambda))} + \sqrt{(\beta_2 - \beta(\alpha, \lambda))^T \Sigma (\beta_2 - \beta(\alpha, \lambda))} \\
&= \sqrt{(\lambda \beta_1 + (1-\alpha) \Sigma (\beta_1 - \beta_2))^T \Sigma (\Sigma + \lambda I)^{-2} (\lambda \beta_1 + (1-\alpha) \Sigma (\beta_1 - \beta_2))} \\
&+ \sqrt{(\lambda \beta_2 +  \alpha \Sigma (\beta_2 - \beta_1))^T \Sigma (\Sigma + \lambda I)^{-2} (\lambda \beta_2 +  \alpha \Sigma (\beta_2 - \beta_1))} \\
&= \sqrt{(\lambda \beta_1 + (1-\alpha) \Sigma (\beta_1 - \beta_2))^T \Sigma (\Sigma + \lambda I)^{-2} (\lambda \beta_1 + (1-\alpha) \Sigma (\beta_1 - \beta_2)) } \\
&+ \sqrt{(-\lambda \beta_2 +  \alpha\Sigma (\beta_1 - \beta_2))^T \Sigma (\Sigma + \lambda I)^{-2} (-\lambda \beta_2 + \alpha \Sigma (\beta_1 - \beta_2))}.
\end{align*}
Now note that for any PSD matrix $\Sigma'$ and any distribution, note that the following triangle inequality holds: 
\[\sqrt{\mathbb{E}[(X_1 + X_2)^T \Sigma' (X_1 + X_2)]} \le \sqrt{\mathbb{E}[X_1^T \Sigma' X_1]} + \sqrt{\mathbb{E}[X_2^T \Sigma' X_2]}.  \] We apply this for $X_1 = \lambda \beta_1 + (1-\alpha) \Sigma (\beta_1 - \beta_2)$, $X_2 = -\lambda \beta_2 +  \alpha\Sigma (\beta_1 - \beta_2)$, and distribution $\DC$. This means that we can lower bound: 
\begin{align*}
  T &\ge \sqrt{\mathbb{E}_{\DC}[((\Sigma + \lambda I) (\beta_1 - \beta_2))^T \Sigma (\Sigma + \lambda I)^{-2} ((\Sigma + \lambda I) (\beta_1 - \beta_2))]} \\
  &= \sqrt{\mathbb{E}_{\DC}[(\beta_1 - \beta_2))^T \Sigma (\beta_1 - \beta_2))]}
\end{align*}
as desired. 
\end{proof} 
Now, we are ready to compute $\tilde{\mix}_I$ and $\tilde{\reg}_I$ for the incumbent. 
\begin{lemma}
\label{lemma:extensionridgelessoptimalinfinitedata}
Let $L^*(\rho) = \mathbb{E}_{\DC}[(\beta_1 - \beta_2)^T \Sigma (\beta_1 - \beta_2)^T]$. Suppose that $\Nlead = \infty$, and suppose that the safety constraint $\constrlead$ satisfies \eqref{eq:safetythreshold}. Then it holds that $\alpha_I =  \sqrt{\frac{\min(\constrlead, L^*(\rho))}{L^*(\rho)}}$, and $\reg_I = 0$ is optimal for the incumbent. Moreover, it holds that:
\[\mathbb{E}_{\DC}[L^*_1(\beta_1, \beta_2, \DF, \tilde{\reg}_I, \infty, \tilde{\alpha}_I)] = \left(\sqrt{L^*(\rho)} - \sqrt{\min(L^*(\rho), \constrlead)}\right)^2. \]
\end{lemma}
\begin{proof}
First, we apply Lemma \ref{lemma:extensionsollichmoreprecise} with $N = \infty$ to see that:
\[\mathbb{E}_{\DC}[L^*_1(\beta_1, \beta_2, \DF, \reg, \infty, \alpha)] = \mathbb{E}_{\DC}[L_1(\beta(\alpha, \lambda))] \]
and 
\[\mathbb{E}_{\DC}[L^*_2(\beta_1, \beta_2, \DF, \reg, \infty, \alpha)] = \mathbb{E}_{\DC}[L_2(\beta(\alpha, \lambda))]. \]

Let $\alpha^* =\sqrt{\frac{\min(\constrlead, L^*(\rho))}{L^*(\rho)}}$. By the assumption in the lemma statement, we know that:
\[\alpha^* \ge \sqrt{\frac{\mathbb{E}_{\DC}[\LossAlign^*(\beta_1, \beta_2, \DF, 0.5)]}{L^*(\rho)}} = 0.5. \]
Observe that:
\begin{align*}
&\sqrt{\mathbb{E}_{\DC}[ L_1(\beta(\alpha^*, 0))]} + \sqrt{\min(\constrlead, L^*(\rho))} \\
&=\sqrt{\mathbb{E}_{\DC}[ L_1(\beta(\alpha^*, 0))]} + \sqrt{\mathbb{E}_{\DC}[ L_2(\beta(\alpha^*, 0))]} \\
&= \sqrt{(1-\alpha^*)^2 \mathbb{E}_{\DC}[(\beta_1 - \beta_2)^T \Sigma (\beta_1 - \beta_2)^T]} + \sqrt{ (\alpha^*)^2 \mathbb{E}_{\DC}[(\beta_1 - \beta_2)^T \Sigma (\beta_1 - \beta_2)^T]}  \\
&= \sqrt{\mathbb{E}_{\DC}[(\beta_1 - \beta_2)^T \Sigma (\beta_1 - \beta_2)^T]} 
\end{align*}

We show that $(\tilde{\alpha}_I, \tilde{\reg}_I) = (\alpha^*, 0)$. Assume for sake of contradiction that $(\alpha, \reg) \neq (\alpha^*, 0)$ satisfies the safety constraint $\mathbb{E}_{\DC}[\tilde{L}_2(\beta_1, \beta_2, \DF, \alpha)] \le \min(\constrlead, L^*(\rho))$ and achieves strictly better performance loss:
\[\mathbb{E}_{\DC}[\LossPerf^*(\beta_1, \beta_2, \DF, \reg, \infty, \mix)] <  \mathbb{E}_{\DC}[\LossPerf^*(\beta_1, \beta_2, \DF, 0, \infty, \mix^*)].\] Then it would hold that:
\begin{align*}
    \sqrt{\mathbb{E}_{\DC}[ L_1(\beta(\mix, \reg))]} + \sqrt{\mathbb{E}_{\DC}[ L_2(\beta(\mix, \reg))]} &< \sqrt{\mathbb{E}_{\DC}[ L_1(\beta(\alpha^*, 0))]} + \sqrt{\min(\constrlead, L^*(\rho))}  \\
    &= \sqrt{\mathbb{E}_{\DC}[(\beta_1 - \beta_2)^T \Sigma (\beta_1 - \beta_2)^T]},
\end{align*}
which contradicts Lemma \ref{lemma:paretofrontiernew}.

To analyze the loss, note that:
\begin{align*}
\mathbb{E}_{\DC}[L^*_1(\beta_1, \beta_2, \DF, \tilde{\reg}_I, \infty, \tilde{\alpha}_I)] \\
&= \mathbb{E}_{\DC}[L_1(\beta(\tilde{\alpha}_I, \tilde{\lambda}_I))] \\
&= (1-\tilde{\alpha}_I)^2 L^*(\rho) \\ 
&= (\sqrt{L^*(\rho)} - \sqrt{\min(L^*(\rho), \constrlead)})^2 
\end{align*}
    
\end{proof}

We now prove Theorem \ref{thm:extensiontradeoffwarmup}.  
\begin{proof}[Proof of Theorem \ref{thm:extensiontradeoffwarmup}]

We analyze $(\tilde{\alpha}_C, \tilde{\reg}_C)$ first for the incumbent $C = I$ and then for the entrant $C = E$. 

\paragraph{Analysis of the incumbent $C = I$.}
By Lemma \ref{lemma:extensionridgelessoptimalinfinitedata}, we see that:
\[\mathbb{E}_{\DC}[L^*_1(\beta_1, \beta_2, \DF, \tilde{\reg}_I, \infty, \tilde{\alpha}_I)] =  \left(\sqrt{L^*(\rho)} - \sqrt{\min(\constrlead, L^*(\rho))}\right)^2. \]

\paragraph{Analysis of the entrant $C = E$.} This analysis follows identically to the analogous case in the proof of Theorem \ref{thm:tradeoffwarmup}, and we repeat the proof for completeness. Since the entrant faces no safety constraint, the entrant can choose any $\alpha \in [0.5, 1]$. We apply Corollary \ref{cor:scalinglawoptreg} to see that:
\[ \mathbb{E}_{\DC}[L^*_1(\beta_1, \beta_2, \DF, \reg_E, N, \alpha_E)] = \inf_{\alpha \in [0.5, 1]} \inf_{\reg > 0} \mathbb{E}_{\DC}[L^*_1(\beta_1, \beta_2, \DF, \reg, N, \alpha)] = \Theta\left(
N^{-\scalingexp} \right), \]
which means that:
\[ \Nentr^*(\infty, \constrlead, \infty, \DC, \DF) = \Theta \left(\left(\sqrt{L^*(\rho)} - \sqrt{\min(\constrlead, L^*(\rho)}\right)^{-2/\scalingexp} \right)\]
as desired. 
We can further apply Claim \ref{claim:boundLstar} to see that $L^*(\rho) = \Theta(1-\rho)$. 
\end{proof}

\subsection{Bounds on the excess loss for safety}\label{appendix:simplifieddeterministic}

We bound the excess loss $\alpha^2 L^*(\rho)  - \mathbb{E}_{\DC}[L_2^{\texttt{det}}]$. We assume that $\alpha \ge 0.5$ and we further assume that $\delta \le 1$. 
\begin{lemma}
\label{lemma:boundsextensions}
Suppose that power scaling holds for the eigenvalues and alignment coefficients with scaling $\gamma > 0$ and $\delta \in (0,1]$, and correlation coefficient $\rho \in [0, 1)$, and suppose that $P = \infty$. Suppose that $\alpha \ge 0.5$, $\reg \in (0,1)$, and $N \ge 1$. Let $L_2^{\texttt{det}} :=  L_2^{\texttt{det}}(\beta_1, \beta_2, \DF, \reg, N, \alpha) $ be defined according to \eqref{eq:deterministicequivalentsafety}. Let $L^*(\rho) = \mathbb{E}_{\DC}[(\beta_1 - \beta_2)^T \Sigma (\beta_1 - \beta_2)]$. Then it holds that:
\[ \alpha^2 L^*(\rho)  - \mathbb{E}_{\DC}[L_2^{\texttt{det}}] = O\left( \max(\lambda^{\frac{\scalingexp}{1 + \gamma}}, N^{-\scalingexp}) 
\right)  \]
and 
\[ \mathbb{E}_{\DC}[L_2^{\texttt{det}}] - \alpha^2 L^*(\rho) = O\left( \max(\lambda^{\frac{\scalingexp}{1 + \gamma}}, N^{-\scalingexp}) + (1-\alpha) (1-\rho) \frac{\min(\lambda^{-\frac{1}{1 + \gamma}}, N)}{N} \right), \]
where $\scalingexp = \min(2(1+\gamma), \delta + \gamma) = \delta + \gamma$. 
\end{lemma}

To prove Lemma \ref{lemma:boundsextensions}, we first simplify the deterministic equivalent $L_2^{\texttt{det}}(\beta_1, \beta_2, \DF, \reg, N, \alpha)$ using the assumptions from Section \ref{subsec:assumptions}. 
\begin{lemma}
\label{lemma:extensionsollichmoreprecise}
Suppose that power scaling holds for the eigenvalues and alignment coefficients with scaling $\gamma, \delta > 0$ and correlation coefficient $\rho \in [0, 1)$, and suppose that $P = \infty$. Suppose that $\reg \in (0,1)$, and $N \ge 1$. Let $L_2^{\texttt{det}} :=  L_2^{\texttt{det}}(\beta_1, \beta_2, \DF, \reg, N, \alpha) $ be defined according to \eqref{eq:deterministicequivalentsafety}. Let $\kappa = \kappa(\lambda, N, \Sigma)$ from Definition \ref{def:effectiveregularizer}. Let $L^*(\rho) = \mathbb{E}_{\DC}[(\beta_1 - \beta_2)^T \Sigma (\beta_1 - \beta_2)]$. Then it holds that:
  \begin{align*}
  \mathbb{E}_{\DC}[L_2^{\texttt{det}}] - L^*(\rho) &= 
Q^{-1} \cdot \kappa^2 \sum_{i=1}^P \frac{i^{-\delta -1- \gamma}}{(i^{-1-\gamma} + \kappa)^2}  + Q^{-1} 2\kappa \alpha (1-\alpha) (1-\rho) \sum_{i=1}^P \frac{i^{-\delta - 2(1+\gamma)}}{(i^{-1-\gamma}+\kappa)^2} \\
&+ Q^{-1} 2 \alpha (1-\alpha) (1-\rho) \frac{1}{N} \left(\sum_{i=1}^P \frac{i^{-2-2\gamma}}{(i^{-1-\gamma} + \kappa)^2} \right) \cdot \sum_{i=1}^P  \frac{i^{-\delta - 2-2\gamma}}{i^{-1-\gamma} + \kappa}  \\
&- 2 \alpha^2 \kappa (1-\rho) \sum_{i=1}^P  \frac{i^{-\delta - 1 -\gamma}}{i^{-1-\gamma} + \kappa},
\end{align*}
  where $Q = 1 - \frac{1}{N} \sum_{i=1}^P \frac{i^{-2-2\gamma}}{(i^{-1-\gamma} + \kappa)^2}$. 
\end{lemma}
\begin{proof}
First, we apply  Lemma \ref{lemma:sollichmoreprecise}, coupled with the fact that $ L_2^{\texttt{det}}(\beta_1, \beta_2, \DF, \reg, N, \alpha) := L_1^{\texttt{det}}(\beta_2, \beta_1, \DF, \reg, N, 1-\alpha)$, to see that:
\begin{align*}
  Q \cdot \mathbb{E}_{\DC}[L_2^{\texttt{det}}] &= \kappa^2 (1 - 2 \alpha^2 (1-\rho)) \sum_{i=1}^P \frac{i^{-\delta - 1- \gamma}}{(i^{-1-\gamma} + \kappa)^2} + \alpha^2 L^*(\rho) \\
  &+ 2\kappa (1-\rho) \alpha (1 - 2  \alpha) \sum_{i=1}^P \frac{i^{-\delta - 2(1+\gamma)}}{(i^{-1-\gamma}+\kappa)^2} \\
  &+ 2 \alpha (1-\rho) \frac{1}{N} \left(\sum_{i=1}^P \frac{i^{-2-2\gamma}}{(i^{-1-\gamma} + \kappa)^2} \right) \cdot (1 - 2\alpha)  \sum_{i=1}^P  \frac{i^{-\delta - 2-2\gamma}}{i^{-1-\gamma} + \kappa},
  \end{align*}
   where $Q = 1 - \frac{1}{N} \sum_{i=1}^P \frac{i^{-2-2\gamma}}{(i^{-1-\gamma} + \kappa)^2}$. Using that $(Q^{-1} - 1) \alpha^2 L^*(\rho) = Q^{-1} \frac{1}{N} \left(\sum_{i=1}^P \frac{i^{-2-2\gamma}}{(i^{-1-\gamma} + \kappa)^2} \right) 2 \alpha^2 (1-\rho)  \left( \sum_{i=1}^P i^{-\delta-1-\gamma} \right)$, this means that:
\begin{align*}
\mathbb{E}_{\DC}[L_2^{\texttt{det}}] -  \alpha^2 L^*(\rho) &= Q^{-1} \frac{1}{N} \left(\sum_{i=1}^P \frac{i^{-2-2\gamma}}{(i^{-1-\gamma} + \kappa)^2} \right) 2 \alpha^2 (1-\rho)  \left( \sum_{i=1}^P i^{-\delta-1-\gamma} \right) \\
&+ 
Q^{-1} \cdot \kappa^2 (1 - 2 \alpha^2 (1-\rho)) \sum_{i=1}^P \frac{i^{-\delta - 1- \gamma}}{(i^{-1-\gamma} + \kappa)^2} \\
&+ Q^{-1} 2\kappa (1-\rho) \alpha (1 - 2  \alpha) \sum_{i=1}^P \frac{i^{-\delta - 2(1+\gamma)}}{(i^{-1-\gamma}+\kappa)^2} \\
&+ Q^{-1} 2 \alpha (1-\rho) \frac{1}{N} \left(\sum_{i=1}^P \frac{i^{-2-2\gamma}}{(i^{-1-\gamma} + \kappa)^2} \right) \cdot (1 - 2\alpha)  \sum_{i=1}^P  \frac{i^{-\delta - 2-2\gamma}}{i^{-1-\gamma} + \kappa}
\end{align*}
By expanding some of these terms, we see that:
\begin{align*}
\mathbb{E}_{\DC}[L_2^{\texttt{det}}] -  \alpha^2 L^*(\rho) &= 
Q^{-1} 2 \alpha^2 (1-\rho) \frac{1}{N} \left(\sum_{i=1}^P \frac{i^{-2-2\gamma}}{(i^{-1-\gamma} + \kappa)^2} \right) \cdot \sum_{i=1}^P i^{-\delta-1-\gamma}  \\
&+ Q^{-1} \cdot \kappa^2 \sum_{i=1}^P \frac{i^{-\delta - 1-\gamma}}{(i^{-1-\gamma} + \kappa)^2} - Q^{-1}2 \alpha^2 (1-\rho) \cdot \kappa^2 \sum_{i=1}^P \frac{i^{-\delta -1- \gamma}}{(i^{-1-\gamma} + \kappa)^2} \\
&+ Q^{-1} 2\kappa (1-\rho) \alpha (1 -  \alpha) \sum_{i=1}^P \frac{i^{-\delta - 2(1+\gamma)}}{(i^{-1-\gamma}+\kappa)^2} - Q^{-1} 2\kappa (1-\rho) \alpha^2 \sum_{i=1}^P \frac{i^{-\delta - 2(1+\gamma)}}{(i^{-1-\gamma}+\kappa)^2}  \\
&+ Q^{-1} 2 \alpha (1 - \alpha)(1-\rho) \frac{1}{N} \left(\sum_{i=1}^P \frac{i^{-2-2\gamma}}{(i^{-1-\gamma} + \kappa)^2} \right) \cdot \sum_{i=1}^P  \frac{i^{-\delta - 2-2\gamma}}{i^{-1-\gamma} + \kappa}\\
&- Q^{-1} 2 \alpha^2 (1-\rho) \frac{1}{N} \left(\sum_{i=1}^P \frac{i^{-2-2\gamma}}{(i^{-1-\gamma} + \kappa)^2} \right) \cdot \sum_{i=1}^P  \frac{i^{-\delta - 2-2\gamma}}{i^{-1-\gamma} + \kappa} .
\end{align*}

When we collect terms, we obtain: 
\begin{align*}
\mathbb{E}_{\DC}[L_2^{\texttt{det}}] -  \alpha^2 L^*(\rho) 
&= Q^{-1} \cdot \kappa^2 \sum_{i=1}^P \frac{i^{-\delta - 1-\gamma}}{(i^{-1-\gamma} + \kappa)^2} + Q^{-1} 2\kappa (1-\rho) \alpha (1 -  \alpha) \sum_{i=1}^P \frac{i^{-\delta - 2(1+\gamma)}}{(i^{-1-\gamma}+\kappa)^2} \\
&+ Q^{-1} 2 \alpha (1 - \alpha)(1-\rho) \frac{1}{N} \left(\sum_{i=1}^P \frac{i^{-2-2\gamma}}{(i^{-1-\gamma} + \kappa)^2} \right) \cdot \sum_{i=1}^P  \frac{i^{-\delta - 2-2\gamma}}{i^{-1-\gamma} + \kappa}\\
&- Q^{-1} 2\kappa (1-\rho) \alpha^2 \left(\sum_{i=1}^P \frac{i^{-\delta - 2(1+\gamma)}}{(i^{-1-\gamma}+\kappa)^2}  + \sum_{i=1}^P \frac{\kappa \cdot i^{-\delta -1- \gamma}}{(i^{-1-\gamma} + \kappa)^2} \right) \\
&+ 
Q^{-1} 2 \alpha^2 (1-\rho) \frac{1}{N} \left(\sum_{i=1}^P \frac{i^{-2-2\gamma}}{(i^{-1-\gamma} + \kappa)^2} \right) \cdot \left(\sum_{i=1}^P i^{-\delta-1-\gamma} - \sum_{i=1}^P  \frac{i^{-\delta - 2-2\gamma}}{i^{-1-\gamma} + \kappa}\right) \\
&= Q^{-1} \cdot \kappa^2 \sum_{i=1}^P \frac{i^{-\delta - 1-\gamma}}{(i^{-1-\gamma} + \kappa)^2} + Q^{-1} 2\kappa (1-\rho) \alpha (1 -  \alpha) \sum_{i=1}^P \frac{i^{-\delta - 2(1+\gamma)}}{(i^{-1-\gamma}+\kappa)^2} \\
&+ Q^{-1} 2 \alpha (1 - \alpha)(1-\rho) \frac{1}{N} \left(\sum_{i=1}^P \frac{i^{-2-2\gamma}}{(i^{-1-\gamma} + \kappa)^2} \right) \cdot \sum_{i=1}^P  \frac{i^{-\delta - 2-2\gamma}}{i^{-1-\gamma} + \kappa}\\
&- Q^{-1} 2\kappa (1-\rho) \alpha^2 \left(\sum_{i=1}^P \frac{i^{-\delta - 1 - \gamma}}{(i^{-1-\gamma}+\kappa)} \right) \\
&+ 
Q^{-1} 2 \kappa \alpha^2 (1-\rho) \frac{1}{N} \left(\sum_{i=1}^P \frac{i^{-2-2\gamma}}{(i^{-1-\gamma} + \kappa)^2} \right) \cdot \frac{i^{-\delta - 1-\gamma}}{i^{-1-\gamma} + \kappa}.
\end{align*}
Combining the last two terms gives us the desired statement. 
\end{proof}

Now, we are ready to prove Lemma \ref{lemma:boundsextensions}.
\begin{proof}

For the first bound, we observe that:
\begin{align*}
   &\alpha^2 L^*(\rho) - \mathbb{E}_{\DC}[L_2^{\texttt{det}}]  \\
  &\le_{(A)} 2 \alpha^2 \kappa (1-\rho) \sum_{i=1}^P  \frac{i^{-\delta - 1 -\gamma}}{i^{-1-\gamma} + \kappa} \\
  &=_{(B)} O\left(\alpha^2 (1-\rho) \kappa^{\frac{\min(1+\gamma, \delta+\gamma)}{1+\gamma}} \right) \\
  &=_{(C)} O\left(\kappa^{\frac{\scalingexp}{1+\gamma}} \right) \\
  &=_{(D)} O\left(\max(\lambda^{\frac{\scalingexp}{1+\gamma}}, N^{-\scalingexp}) \right)
\end{align*}
where (A) uses Lemma \ref{lemma:extensionsollichmoreprecise}, (B) uses Lemma \ref{lemma:splitintegralbounds}, (C) uses that $\delta \le 1$ and $\rho \in [0, 1)$, and (D) uses Lemma \ref{lemma:kappabasic}.

For the second bound, we observe that:
\begin{align*}
  &\mathbb{E}_{\DC}[L_2^{\texttt{det}}] -  \alpha^2 L^*(\rho) \\
  &\le_{(A)}  
Q^{-1} \cdot \kappa^2 \sum_{i=1}^P \frac{i^{-\delta -1- \gamma}}{(i^{-1-\gamma} + \kappa)^2}  \\
&+ Q^{-1} 2\kappa \alpha (1-\alpha) (1-\rho) \sum_{i=1}^P \frac{i^{-\delta - 2(1+\gamma)}}{(i^{-1-\gamma}+\kappa)^2} \\
&+ Q^{-1} 2 \alpha (1-\alpha) (1-\rho) \frac{1}{N} \left(\sum_{i=1}^P \frac{i^{-2-2\gamma}}{(i^{-1-\gamma} + \kappa)^2} \right) \cdot \sum_{i=1}^P  \frac{i^{-\delta - 2-2\gamma}}{i^{-1-\gamma} + \kappa}  \\ 
&=_{(B)}  
O\left( \kappa^{\frac{\min(2(1+\gamma), \gamma + \delta)}{1+\gamma}} + \alpha (1-\alpha) (1-\rho) \kappa^{\frac{\min(1+\gamma, \gamma + \delta)}{1+\gamma}} + \alpha (1-\alpha) (1-\rho) \frac{\kappa^{-\frac{1}{1+\gamma}}}{N} \right) \\
&=_{(C)}  
O\left( \kappa^{\frac{\gamma + \delta}{1+\gamma}} + (1-\alpha) (1-\rho) \kappa^{\frac{\gamma + \delta}{1+\gamma}} + (1-\alpha) (1-\rho) \frac{\kappa^{-\frac{1}{1+\gamma}}}{N} \right) \\
&= O\left( \kappa^{\frac{\gamma + \delta}{1+\gamma}} + (1-\alpha) (1-\rho) \frac{\kappa^{-\frac{1}{1+\gamma}}}{N} \right) \\
&=_{(D)} O\left( \max(\lambda^{\frac{\scalingexp}{1+\gamma}}, N^{-\scalingexp}) + (1-\alpha) (1-\rho) \frac{\min(\lambda^{-\frac{1}{1+\gamma}}, N)}{N} \right) \\
\end{align*}
where (A) uses Lemma \ref{lemma:extensionsollichmoreprecise}, (B) uses Lemma \ref{lemma:splitintegralbounds} and Lemma \ref{lemma:degreesoffreedom}, (C) uses that $\delta \le 1$ and $\alpha \ge 0.5$, and (D) uses Lemma \ref{lemma:kappabasic}.
\end{proof}

\subsection{Extension of Theorem \ref{thm:finitedata}}
\label{appendix:extensionfinitedata}

We next study the market entry $\Nentroptmodified$ threshold in the environment of Theorem \ref{thm:finitedata} where the incumbent has \textit{finite data} and the new company faces no safety constraint. We place the further assumption that $\delta \le 1$. We compute the following upper bound on the modified market entry threshold.

\begin{theorem}[Extension of Theorem \ref{thm:finitedata}]
\label{thm:extensionfinitedata}
Suppose that the power-law scaling holds for the eigenvalues and alignment coefficients with scaling exponents $\gamma > 0, \delta \in (0, 1]$ and correlation coefficient $\rho \in [0, 1)$, and suppose that $P = \infty$. Assume that $\constrentr = \infty$. Suppose that the safety constraint $\constrlead$ satisfies \eqref{eq:safetythreshold}. Then we have that $\Nentroptmodified = \Nentroptmodified(\Nlead, \constrlead, \infty, \DC, \DF)$ satisfies: 
\[
\Nentroptmodified := 
\begin{cases}
O\left(\Nlead\right) &\text{ if }  \Nlead \le \tilde{G}_I^{-\frac{1}{2 \scalingexp}} (1-\rho)^{-\frac{1}{2\scalingexp}}  \\
O\left(\Nlead^{\frac{1}{\scalingexp+1}} \cdot \tilde{G}_I^{-\frac{1}{2(\scalingexp+1)}} (1-\rho)^{-\frac{1}{2(\scalingexp+1)}}\right)  &\text{ if } \tilde{G}_I^{-\frac{1}{2 \scalingexp}} (1-\rho)^{-\frac{1}{2\scalingexp}} \le \Nlead  \le \tilde{G}_I^{-\frac{1}{2} - \frac{1}{\scalingexp}}(1-\rho)^{\frac{1}{2}}  \\
O\left(\tilde{G}_I^{-\frac{1}{\scalingexp}}\right)  &\text{ if }  \Nlead \ge \tilde{G}_I^{-\frac{1}{2} - \frac{1}{\scalingexp}}(1-\rho)^{\frac{1}{2}},
\end{cases}
\]
where $L^*(\rho) = \mathbb{E}_{\DC}[(\beta_1 - \beta_2)^T \Sigma (\beta_1 - \beta_2)] = \Theta(1 - \rho)$, where $\alpha^* = \sqrt{\frac{\min(\constrlead, L^*(\rho))}{L^*(\rho)}}$, where $\tilde{\alpha} := \sqrt{(1-\alpha^*) + (\alpha^*)^2}$, where $\tilde{G}_I = (1-\tilde{\alpha})^2 (1-\rho)$, and where $\scalingexp = \min(2(1+\gamma), \gamma + \delta) = \gamma + \delta$.  
\end{theorem}

Theorem \ref{thm:extensionfinitedata} shows that the key qualitative finding from Theorem \ref{thm:finitedata}---that the new company can enter with $\Nentr = o(\Nlead)$ data as long as the incumbent's dataset size is sufficiently large---readily extends to this setting. We note that the bound in Theorem \ref{thm:extensionfinitedata} and the bound in Theorem \ref{thm:finitedata} take slightly different forms: the term $G_I = (\sqrt{L^*(\rho)}- \sqrt{\min(L^*(\rho), \constrlead)})^2 = \Theta((1-\alpha^*)^2 (1-\rho))$ is replaced by $\tilde{G}_I = (1-\tilde{\alpha})^2(1-\rho)$. We expect some of these differences arise because the bound in Theorem \ref{thm:extensionfinitedata} is not tight, rather than fundamental distinctions between the two settings. Proving a tight bound on the modified market entry threshold is an interesting direction for future work. 

To prove this, we compute a lower bound on the incumbent's loss $\mathbb{E}_{\DC}[\LossPerf^*(\beta_1, \beta_2, \DF, \tilde{\lambda}_I, \Nlead, \tilde{\alpha}_I)]$. 
\begin{lemma}
\label{lemma:incumbentfinitedata}
Suppose that the power-law scaling holds for the eigenvalues and alignment coefficients with scaling exponents $\gamma > 0, \delta \in (0, 1]$ and correlation coefficient $\rho \in [0, 1)$, and suppose that $P = \infty$. Assume that $\constrentr = \infty$. Suppose that the safety constraint $\constrlead$ satisfies \eqref{eq:safetythreshold}. 
Then we have that: 
\begin{align*}
  &\mathbb{E}_{\DC}[L^*_1(\beta_1, \beta_2, \DF, \tilde{\reg}_I, \Nlead, \tilde{\alpha}_I)] \\
  &=  \begin{cases}
\Omega\left( \Nlead^{-\scalingexp} \right) &\text{ if }  \Nlead \le \tilde{G}_I^{-\frac{1}{2\scalingexp}} (1-\rho)^{-\frac{1}{2\scalingexp}}  \\
\Omega\left(\Nlead^{-\frac{\scalingexp}{\scalingexp+1}} \cdot \tilde{G}_I^{\frac{\scalingexp}{2(\scalingexp+1)}} (1-\rho)^{\frac{\scalingexp}{2(\scalingexp+1)}}\right)  &\text{ if  } \tilde{G}_I^{-\frac{1}{2\scalingexp}} (1-\rho)^{-\frac{1}{2\scalingexp}} \le \Nlead  \le \tilde{G}_I^{-\frac{1}{2} - \frac{1}{\scalingexp}}(1-\rho)^{\frac{1}{2}}  \\
\Omega\left(\tilde{G}_I\right)  &\text{ if  }  \Nlead \ge \tilde{G}_I^{-\frac{1}{2} - \frac{1}{\scalingexp}}(1-\rho)^{\frac{1}{2}}.
\end{cases}
\end{align*}
where $L^*(\rho) = \mathbb{E}_{\DC}[(\beta_1 - \beta_2)^T \Sigma (\beta_1 - \beta_2)] = \Theta(1 - \rho)$, where $\alpha^* = \sqrt{\frac{\min(\constrlead, L^*(\rho))}{L^*(\rho)}}$, where $\tilde{\alpha} := \sqrt{(1-\alpha^*) + (\alpha^*)^2}$, where $\tilde{G}_I = (1-\tilde{\alpha})^2 (1-\rho)$
and where $\scalingexp = \min(2(1+\gamma), \gamma + \delta) = \gamma + \delta$.  
\end{lemma}
\begin{proof}
By Corollary \ref{cor:scalinglawoptreg} and Lemma \ref{lemma:kappabasic}, we know that:
\[
  \mathbb{E}_{\DC}[\LossPerf^*(\beta_1, \beta_2, \DF, \tilde{\lambda}_I, \Nlead, \tilde{\alpha}_I)]  = \Omega(\kappa^{\frac{\scalingexp}{1+\gamma}}) = \Omega(\max(\lambda^{\frac{\scalingexp}{1+\gamma}}, \Nlead^{-\scalingexp})). 
\]
Let $C_{\delta, \gamma}$ be an implicit constant\footnote{We need to introduce an implicit constant because of $O()$ is permitted to hide constants that depend on $\delta$ and $\gamma$.} such that:
\begin{equation}
\label{eq:lowerbound}
\mathbb{E}_{\DC}[\LossPerf^*(\beta_1, \beta_2, \DF, \tilde{\lambda}_I, \Nlead, \tilde{\alpha}_I)]  \ge C_{\delta, \gamma} \max(\lambda^{\frac{\scalingexp}{1+\gamma}}, \Nlead^{-\scalingexp})
\end{equation}
By Lemma \ref{lemma:boundsextensions}, there also exists an implicit constant $C'_{\delta, \gamma}$  such that:
\begin{equation}
\label{eq:boundsextensions}
   \alpha^2 L^*(\rho) - \mathbb{E}_{\DC}[L_2^{\texttt{det}}(\beta_1, \beta_2, \DF, \lambda, \Nlead, \alpha)] \le C'_{\delta, \gamma} \max(\lambda^{\frac{\scalingexp}{1+\gamma}}, \Nlead^{-\scalingexp}).
\end{equation}

We now split into two cases: (1) $\frac{C'_{\delta, \gamma}}{C_{\delta, \gamma}} \mathbb{E}_{\DC}[\LossPerf^*(\beta_1, \beta_2, \DF, \tilde{\lambda}_I, \Nlead, \tilde{\alpha}_I)] \ge (1-\alpha^*)L^*(\rho) $, and (2)  $\frac{C'_{\delta, \gamma}}{C_{\delta, \gamma}} \mathbb{E}_{\DC}[\LossPerf^*(\beta_1, \beta_2, \DF, \tilde{\lambda}_I, \Nlead, \tilde{\alpha}_I)] \le (1-\alpha^*)L^*(\rho) $. 

\paragraph{Case 1: $\frac{C'_{\delta, \gamma}}{C_{\delta, \gamma}}  \mathbb{E}_{\DC}[\LossPerf^*(\beta_1, \beta_2, \DF, \tilde{\lambda}_I, \Nlead, \tilde{\alpha}_I)] \ge (1-\alpha^*)L^*(\rho)$.} It follows from \eqref{eq:lowerbound} that: 
\[ \mathbb{E}_{\DC}[\LossPerf^*(\beta_1, \beta_2, \DF, \tilde{\lambda}_I, \Nlead, \tilde{\alpha}_I)]  \ge C_{\delta, \gamma} \max(\lambda^{\frac{\scalingexp}{1+\gamma}}, \Nlead^{-\scalingexp}) \ge C_{\delta, \gamma} \Nlead^{-\scalingexp}. \] 
Using the condition for this case, this implies that:
\begin{align*}
 \Nlead &\le \left(\frac{1}{C_{\delta, \gamma}}  \mathbb{E}_{\DC}[\LossPerf^*(\beta_1, \beta_2, \DF, \tilde{\lambda}_I, \Nlead, \tilde{\alpha}_I)]\right)^{-\frac{1}{\scalingexp}} \\
 &\le  \left(\frac{1}{C'_{\delta, \gamma}} (1-\alpha^*) L^*(\rho) \right)^{-\frac{1}{\scalingexp}} \\
 &= O\left(\left((1-\tilde{\alpha}) (1-\rho) \right)^{-\frac{1}{\scalingexp}}\right) \\
 &= O\left(\tilde{G}_I^{-\frac{1}{2\scalingexp}} (1-\rho)^{-\frac{1}{2 \scalingexp}} \right).    
\end{align*}
This proves that $\Nlead$ is up to constants within the first branch of the expression in the lemma statement. Since the bound in the lemma statement only changes by constants (that depend on $\delta$ and $\gamma$) between the first branch and second branch, this proves the desired expression for this case.

\paragraph{Case 2: $\frac{C'_{\delta, \gamma}}{C_{\delta, \gamma}} \mathbb{E}_{\DC}[\LossPerf^*(\beta_1, \beta_2, \DF, \tilde{\lambda}_I, \Nlead, \tilde{\alpha}_I)]  \le (1-\alpha^*)L^*(\rho)$.} Note that $\alpha^* = \sqrt{\frac{\min(\constrlead, L^*(\rho))}{L^*(\rho)}}$ is the mixture parameter that achieves the safety constraint in the infinite-data ridgeless setting. The incumbent's safety constraint means that: 
\[\mathbb{E}_{\DC}[L_2^{\texttt{det}}(\beta_1, \beta_2, \DF, \tilde{\lambda}_I, \Nlead, \tilde{\alpha}_I)] \le (\alpha^*)^2 L^*(\rho).\]
By \eqref{eq:boundsextensions}, this implies that 
\[(\tilde{\alpha}_I)^2 L^*(\rho) \le C'_{\delta, \gamma} \cdot \max(\lambda^{\frac{\delta+ \gamma}{1+\gamma}}, \Nlead^{-\delta-\gamma}) + (\alpha^*)^2 L^*(\rho).\]
Now, applying \eqref{eq:lowerbound} and the assumption for this case, we see that:
\begin{align*}
 (\tilde{\alpha}_I)^2 L^*(\rho) &\le \frac{C'_{\delta, \gamma}}{C_{\delta, \gamma}} \cdot \mathbb{E}_{\DC}[\LossPerf^*(\beta_1, \beta_2, \DF, \tilde{\lambda}_I, \Nlead, \tilde{\alpha}_I)]  + (\alpha^*)^2 L^*(\rho)   \\
 &\le (1-\alpha^*)L^*(\rho) + (\alpha^*)^2 L^*(\rho). 
\end{align*}
This implies that:
\[\tilde{\alpha}_I \le \sqrt{(1-\alpha^*) + (\alpha^*)^2}. \] Let $\tilde{\alpha} := \sqrt{(1-\alpha^*) + (\alpha^*)^2}$. Plugging this into Corollary \ref{cor:scalinglawoptreg}, we see that:
\begin{align*}
  &\mathbb{E}_{\DC}[L^*_1(\beta_1, \beta_2, \DF, \tilde{\reg}_I, \Nlead, \tilde{\alpha}_I)] \\
  &\ge \inf_{\alpha \in \left[0.5, \tilde{\alpha} \right]} \inf_{\reg > 0} \mathbb{E}_{\DC}[L^*_1(\beta_1, \beta_2, \DF, \reg,  \Nlead, \alpha)]  \\
  &= \Theta\left(
\inf_{\reg > 0} \mathbb{E}_{\DC}\left[L^*_1\left(\beta_1, \beta_2, \Sigma, \reg,  \Nlead, \tilde{\alpha} \right)\right] \right) \\
&= \begin{cases}
 \Theta\left(\Nlead^{-\scalingexp}\right) &\text{ if } \Nlead \le (1-\tilde{\alpha} )^{-\frac{1}{\scalingexp}}(1-\rho)^{-\frac{1}{\scalingexp}} \\
 \Theta\left(\left(\frac{\Nlead}{(1-\tilde{\alpha} )(1-\rho)}\right)^{-\frac{\scalingexp}{\scalingexp + 1}}\right) &\text{ if } (1-\tilde{\alpha} )^{-\frac{1}{\scalingexp}}(1-\rho)^{-\frac{1}{\scalingexp}} 
 \le \Nlead \le  (1-\tilde{\alpha} )^{-\frac{2+\scalingexp}{\scalingexp}} (1-\rho)^{-\frac{1}{\scalingexp}}
\\
\Theta((1-\tilde{\alpha} )^2(1-\rho)) &\text{ if } \Nlead \ge (1-\tilde{\alpha} )^{-\frac{2+\scalingexp}{\scalingexp}} (1-\rho)^{-\frac{1}{\scalingexp}},
\end{cases}\\
&= \begin{cases}
\Theta\left( \Nlead^{-\scalingexp} \right) &\text{ if }  \Nlead \le \tilde{G}_I^{-\frac{1}{2\scalingexp}} (1-\rho)^{-\frac{1}{2\scalingexp}}  \\
\Theta\left(\Nlead^{-\frac{\scalingexp}{\scalingexp+1}} \cdot \tilde{G}_I^{\frac{\scalingexp}{2(\scalingexp+1)}} (1-\rho)^{\frac{\scalingexp}{2(\scalingexp+1)}}\right)  &\text{ if } \tilde{G}_I^{-\frac{1}{2\scalingexp}} (1-\rho)^{-\frac{1}{2\scalingexp}} \le \Nlead  \le \tilde{G}_I^{-\frac{1}{2} - \frac{1}{\scalingexp}}(1-\rho)^{\frac{1}{2}}  \\
\Theta\left(\tilde{G}_I\right)  &\text{ if }  \Nlead \ge \tilde{G}_I^{-\frac{1}{2} - \frac{1}{\scalingexp}}(1-\rho)^{\frac{1}{2}}.
\end{cases}
\end{align*}
The statement follows in this case. 
\end{proof}

We are now ready to prove Theorem \ref{thm:extensionfinitedata}.
\begin{proof}[Proof of Theorem \ref{thm:extensionfinitedata}]

We analyze $(\tilde{\alpha}_C, \tilde{\reg}_C)$ first for the incumbent $C = I$ and then for the entrant $C = E$. Like in the theorem statement, let $L^*(\rho) = \mathbb{E}_{\DC}[(\beta_1 - \beta_2)^T \Sigma (\beta_1 - \beta_2)] = \Theta(1 - \rho)$ (Claim \ref{claim:boundLstar}) and $G_I := (\sqrt{L^*(\rho)} - \sqrt{\min(\constrlead, L^*(\rho))})^2$, and $\scalingexp = \min(2(1+\gamma), \delta + \gamma)$.

\paragraph{Analysis of the incumbent $C = I$.} 
We apply Lemma \ref{lemma:incumbentfinitedata} to see that: 
\begin{align*}
  &\mathbb{E}_{\DC}[L^*_1(\beta_1, \beta_2, \DF, \tilde{\reg}_I, \Nlead, \tilde{\alpha}_I)] \\
  &= \begin{cases}
\Omega\left( \Nlead^{-\scalingexp} \right) &\text{ if }  \Nlead \le \tilde{G}_I^{-\frac{1}{2\scalingexp}} (1-\rho)^{-\frac{1}{2\scalingexp}}  \\
\Omega\left(\Nlead^{-\frac{\scalingexp}{\scalingexp+1}} \cdot \tilde{G}_I^{\frac{\scalingexp}{2(\scalingexp+1)}} (1-\rho)^{\frac{\scalingexp}{2(\scalingexp+1)}}\right)  &\text{ if } \tilde{G}_I^{-\frac{1}{2\scalingexp}} (1-\rho)^{-\frac{1}{2\scalingexp}} \le \Nlead  \le \tilde{G}_I^{-\frac{1}{2} - \frac{1}{\scalingexp}}(1-\rho)^{\frac{1}{2}}  \\
\Omega\left(\tilde{G}_I\right)  &\text{ if }  \Nlead \ge \tilde{G}_I^{-\frac{1}{2} - \frac{1}{\scalingexp}}(1-\rho)^{\frac{1}{2}}.
\end{cases}.
\end{align*}

\paragraph{Analysis of the entrant $C = E$.} Since the entrant faces no safety constraint, the entrant can choose any $\alpha \in [0.5, 1]$.  We apply Corollary \ref{thm:scalinglaw} to see that:
\[\mathbb{E}_{\DC}[L^*_1(\beta_1, \beta_2, \DF, \tilde{\reg}_E, N, \tilde{\alpha}_E)] = \inf_{\alpha \in [0.5, 1]} \inf_{\reg > 0} \mathbb{E}_{\DC}[L^*_1(\beta_1, \beta_2, \DF, \reg, N, \alpha)] = \Theta\left(
N^{-\scalingexp} \right), \]
which means that:
\[ \Nentr^*(\Nlead, \constrlead, \infty, \DC, \DF) = \begin{cases}
O\left(\Nlead\right) &\text{ if }  \Nlead \le \tilde{G}_I^{-\frac{1}{2\scalingexp}} (1-\rho)^{-\frac{1}{2\scalingexp}}  \\
O\left(\Nlead^{\frac{1}{\scalingexp+1}} \cdot \tilde{G}_I^{-\frac{1}{2(\scalingexp+1)}} (1-\rho)^{-\frac{1}{2(\scalingexp+1)}}\right)  &\text{ if } \tilde{G}_I^{-\frac{1}{2\scalingexp}} (1-\rho)^{-\frac{1}{2\scalingexp}} \le \Nlead  \le \tilde{G}_I^{-\frac{1}{2} - \frac{1}{\scalingexp}}(1-\rho)^{\frac{1}{2}}  \\
O\left(\tilde{G}_I^{-\frac{1}{\scalingexp}}\right)  &\text{ if }  \Nlead \ge \tilde{G}_I^{-\frac{1}{2} - \frac{1}{\scalingexp}}(1-\rho)^{\frac{1}{2}}
\end{cases}\]
as desired.

\end{proof}

\subsection{Extension of Theorem \ref{thm:alignment}}\label{appendix:extensionalignment}

We next study the market entry $\Nentroptmodified$ threshold in the environment of Theorem \ref{thm:alignment} where the incumbent has infinite data and the new company faces a \textit{nontrivial safety constraint}. We place the further assumption that $\delta \le 1$. We compute the following upper bound on the modified market entry threshold.

\begin{theorem}[Extension of Theorem \ref{thm:alignment}]
\label{thm:extensionalignment}
Suppose that the power-law scaling holds for the eigenvalues and alignment coefficients with scaling exponents $\gamma > 0$, $\delta \in (0,1]$, and correlation coefficient $\rho \in [0, 1)$, and suppose that $P = \infty$. Suppose that the safety constraints $\constrlead$ and $\constrentr$ satisfy \eqref{eq:safetythresholdnew}. 
Then it holds that $\Nentroptmodified = \Nentroptmodified(\infty, \constrlead, \constrentr, \DC, \DF)$ satisfies: 
\[
\Nentroptmodified := 
O\left(\max\left(\tilde{D}^{-\frac{1}{\scalingexp}}, \tilde{D}^{-\frac{\scalingexp + 1}{\scalingexp}}\left( G_E^{\frac{1}{2}} (1-\rho)^{\frac{1}{2}} + \frac{1}{2} G_I - \frac{1}{2} G_E \right)  \right) \right),
\]
where $L^*(\rho) = \mathbb{E}_{\DC}[(\beta_1 - \beta_2)^T \Sigma (\beta_1 - \beta)] = \Theta(1 - \rho)$, where $\scalingexp = \min(2(1+\gamma), \delta + \gamma) = \delta + \gamma$, where 
$G_I := \left(\sqrt{L^*(\rho)} - \sqrt{\min(\constrlead, L^*(\rho))}\right)^2$ and $G_E := \left(\sqrt{L^*(\rho)} - \sqrt{\min(\constrentr, L^*(\rho))}\right)^2$, and where:  
\[\tilde{D} := \alpha^*_E \cdot (G_I - G_E) - \frac{(G_I - G_E)^2}{4 \cdot L^*(\rho)} .\] 
\end{theorem}

Theorem \ref{thm:extensionalignment} shows that the key qualitative finding from Theorem \ref{thm:alignment}---that the new company can enter with finite data, as long as they face a strictly weaker safety constraint than the incumbent company---readily extends to this setting. We note that the bound in Theorem \ref{thm:extensionalignment} and the bound in Theorem \ref{thm:alignment} take slightly different forms. Some of these differences are superficial: while the bound in Theorem \ref{thm:extensionalignment} contains two---rather than three---regimes, the third regime in Theorem \ref{thm:alignment} does not exist in the case where $\delta \le 1$. Other differences are more substantial: for example, the bound in Theorem \ref{thm:extensionalignment} scales with $\tilde{D}$ while the bound in Theorem \ref{thm:alignment} scales with $D$. However, we expect some of this difference arises because the bound in Theorem \ref{thm:extensionalignment} is not tight, rather than fundamental distinctions between the two settings. Proving a tight bound on the modified market entry threshold is an interesting direction for future work.

We compute an upper bound on the number of data points $\Nentr$ that the new company needs to achieve at most loss $\left(\sqrt{L^*(\rho)} - \sqrt{\min(\constrlead, L^*(\rho))}\right)^2$ on performance.
\begin{lemma}
\label{lemma:newcompanyalignment}
Suppose that the power-law scaling holds for the eigenvalues and alignment coefficients with scaling exponents $\gamma > 0, \delta \in (0, 1]$ and correlation coefficient $\rho \in [0, 1)$, and suppose that $P = \infty$. Suppose that the safety constraints $\constrlead$ and $\constrentr$ satisfy \eqref{eq:safetythreshold}. For sufficiently large constant $C_{\delta, \gamma}$, if 
\[ \Nentr \ge 
C_{\delta, \gamma} \cdot \max\left(\tilde{D}^{-\frac{1}{\scalingexp}}, \tilde{D}^{-\frac{\scalingexp + 1}{\scalingexp}}\left( G_E^{\frac{1}{2}} (1-\rho)^{\frac{1}{2}} + \frac{1}{2} G_I - \frac{1}{2} G_E \right)  \right),\]
then it holds that: 
\[\mathbb{E}_{\DC}[\LossPerf^*(\beta_1, \beta_2, \DF, \tilde{\reg}_E, \Nentr, \tilde{\alpha}_E)] \le G_I, \]
where $L^*(\rho) = \mathbb{E}_{\DC}[(\beta_1 - \beta_2)^T \Sigma (\beta_1 - \beta)] = \Theta(1 - \rho)$, where $\scalingexp = \min(2(1+\gamma), \delta + \gamma) = \delta + \gamma$, where 
$G_I := \left(\sqrt{L^*(\rho)} - \sqrt{\min(\constrlead, L^*(\rho))}\right)^2$ and $G_E := \left(\sqrt{L^*(\rho)} - \sqrt{\min(\constrentr, L^*(\rho))}\right)^2$, and where:  
\[\tilde{D} := \alpha^*_E \cdot (G_I - G_E) - \frac{(G_I - G_E)^2}{4 \cdot L^*(\rho)} .\] 
\end{lemma}
\begin{proof}
It suffices to construct $\tilde{\alpha}$ and $\tilde{\lambda}$ such that 
\[\mathbb{E}_{\DC}[\tilde{L}_2(\beta_1, \beta_2, \DF, \tilde{\reg}, \Nentr, \tilde{\alpha})] \le \constrentr\] and 
\[\mathbb{E}_{\DC}[\LossPerf^*(\beta_1, \beta_2, \DF, \tilde{\reg}, \Nentr, \tilde{\alpha})] \le G_I\]
for $\Nentr = \Omega\left(\max\left(\tilde{D}^{-\frac{1}{\scalingexp}}, \tilde{D}^{-\frac{\scalingexp + 1}{\scalingexp}}\left( G_E^{\frac{1}{2}} (1-\rho)^{\frac{1}{2}} + \frac{1}{2} G_I - \frac{1}{2} G_E \right)  \right)\right)$. 

To define $\tilde{\alpha}$ and $\tilde{\reg}$, it is inconvenient to work with the following intermediate quantities. Let $\alpha^*_E = \left(\sqrt{L^*(\rho)} - \sqrt{\min(\constrentr, L^*(\rho))}\right)^2$ and let $\alpha^*_I = \left(\sqrt{L^*(\rho)} - \sqrt{\min(\constrlead, L^*(\rho))}\right)^2$. We define an error function: 
\[f(\Nentr, \alpha, \lambda) := \max(\lambda^{\frac{\scalingexp}{1 + \gamma}}, \Nentr^{-\scalingexp}) + (1-\alpha) (1-\rho) \frac{\min(\lambda^{-\frac{1}{1 + \gamma}}, \Nentr)}{\Nentr}\]
We define:
\[ \tilde{\alpha} := \alpha^*_E + \frac{1}{2} (1 - \alpha^*_E)^2  - \frac{1}{2} (1 - \alpha^*_I)^2 = \alpha^*_I + \frac{\alpha^*_E - \alpha^*_I}{2}. \]
and 
\[\tilde{\reg} := \inf_{\reg \in (0,1)} f(\Nentr, \tilde{\alpha}, \lambda).\]
At these values of $\tilde{\alpha}$ and $\tilde{\reg}$ and under the condition on $\Nentr$, observe that:
\begin{align*}
 f(\Nentr, \tilde{\alpha}, \tilde{\reg}) &= \Theta\left(\max\left(\Nentr^{-\scalingexp}, \left(\frac{\Nentr}{(1-\tilde{\alpha})(1-\rho)} \right)^{-\frac{\scalingexp}{\scalingexp+1}} \right) \right) \\
 &= \Theta\left(\max\left(\Nentr^{-\scalingexp}, \left(\frac{\Nentr}{G_E^{\frac{1}{2}} (1-\rho)^{\frac{1}{2}} + \frac{1}{2}G_I + \frac{1}{2} G_E} \right)^{-\frac{\scalingexp}{\scalingexp+1}} \right) \right) \\
 &= O\left(\tilde{D}\right), 
\end{align*}
where the implicit constant can be reduced by increasing the implicit constant on $\Nentr$. 

The remainder of the analysis boils down to showing that $\mathbb{E}_{\DC}[\tilde{L}_2(\beta_1, \beta_2, \DF, \tilde{\reg}, \Nentr, \tilde{\alpha})] \le \constrentr$ and $\mathbb{E}_{\DC}[\LossPerf^*(\beta_1, \beta_2, \DF, \tilde{\reg}, \Nentr, \tilde{\alpha})] \le G_I$. To show this, we first derive an error function and bound these losses in terms of the error function.

\paragraph{Bounding $\mathbb{E}_{\DC}[\tilde{L}_2(\beta_1, \beta_2, \DF, \tilde{\reg}, \Nentr, \tilde{\alpha})] \le \constrentr$.} Observe that:
\begin{align*}
 &\mathbb{E}_{\DC}[\tilde{L}_2(\beta_1, \beta_2, \DF, \tilde{\reg}, \Nentr, \tilde{\alpha})] \\
 &=_{(A)} \tilde{\alpha}^2 L^*(\rho) + O\left( \max(\lambda^{\frac{\scalingexp}{1 + \gamma}}, \Nentr^{-\scalingexp}) + (1-\alpha) (1-\rho) \frac{\min(\lambda^{-\frac{1}{1 + \gamma}}, \Nentr)}{\Nentr} \right) \\
 &= (\alpha^*_E + \frac{1}{2} (1-\alpha^*_E)^2 - \frac{1}{2} (1-\alpha^*_I)^2) L^*(\rho) + O\left(f(\Nentr, \tilde{\alpha})\right) \\
 &\le \left((\alpha^*_E)^2 L^*(\rho) + \frac{((1-\alpha^*_I)^2 - (1-\alpha^*_E)^2)^2}{4} - \alpha^*_E ((1-\alpha^*_I)^2 - (1-\alpha^*_E)^2)  \right) L^*(\rho) + \tilde{D} \\
 &= \constrentr + \frac{(G_I - G_E)^2}{4 \cdot L^*(\rho)} - \alpha^*_E(G_I - G_E) \alpha^*_E \cdot (G_I - G_E) - \frac{(G_I - G_E)^2}{4 \cdot L^*(\rho)} \\
 &= \constrentr
\end{align*}
where (A) follows from Lemma \ref{lemma:boundsextensions}. This gives us the desired bound.  

\paragraph{Bounding $\mathbb{E}_{\DC}[\LossPerf^*(\beta_1, \beta_2, \DF, \tilde{\reg}, \Nentr, \tilde{\alpha})]$.}
Observe that: 
\begin{align*}
 &\mathbb{E}_{\DC}[\LossPerf^*(\beta_1, \beta_2, \DF, \tilde{\reg}, \Nentr, \tilde{\alpha})] \\
 &=_{(A)} (1-\tilde{\alpha})^2 L^*(\rho) + O\left( \max(\lambda^{\frac{\scalingexp}{1 + \gamma}}, \Nentr^{-\scalingexp}) + (1-\alpha) (1-\rho) \frac{\min(\lambda^{-\frac{1}{1 + \gamma}}, \Nentr)}{\Nentr} \right) \\
 &\le (1- \alpha^*_E - \frac{1}{2} (1-\alpha^*_E)^2 + \frac{1}{2} (1-\alpha^*_I)^2)^2 L^*(\rho) + O\left(f(\Nentr, \tilde{\alpha})\right) \\
  &\le \left((1-\alpha^*_E)^2 + \frac{((1-\alpha^*_I)^2 - (1-\alpha^*_E)^2)^2}{4} - (1-\alpha^*_E) ((1-\alpha^*_I)^2 - (1-\alpha^*_E)^2) \right) L^*(\rho) + \tilde{D} \\
 &\le 
G_E + (G_I - G_E) (1-\alpha^*_E) + \frac{(G_I - G_E)^2}{4 L^*(\rho)}  + \alpha^*_E \cdot (G_I - G_E) - \frac{(G_I - G_E)^2}{4 \cdot L^*(\rho)} \\
&= G_I.
\end{align*}
where (A) uses Theorem \ref{thm:scalinglawexcess}, coupled with the fact that $\delta \le 1$ (which means that $\scalingexp' = \scalingexp$, so the mixture finite data error is subsumed by the finite data error) and coupled with Lemma \ref{lemma:kappabasic}. This gives us the desired bound. 

\end{proof}

We are now ready to prove Theorem \ref{thm:extensionalignment}.

\begin{proof}[Proof of Theorem \ref{thm:extensionalignment}]

We analyze $(\tilde{\alpha}_C, \tilde{\reg}_C)$ first for the incumbent $C = I$ and then for the entrant $C = E$. Like in the theorem statement, let $L^*(\rho) = \mathbb{E}_{\DC}[(\beta_1 - \beta_2)^T \Sigma (\beta_1 - \beta)] = \Theta(1 - \rho)$, let $\scalingexp = \min(2(1+\gamma), \delta + \gamma) = \delta + \gamma$, let 
$G_I := \left(\sqrt{L^*(\rho)} - \sqrt{\min(\constrlead, L^*(\rho))}\right)^2$ and $G_E := \left(\sqrt{L^*(\rho)} - \sqrt{\min(\constrentr, L^*(\rho))}\right)^2$, and let:  
\[\tilde{D} := \alpha^*_E \cdot (G_I - G_E) - \frac{(G_I - G_E)^2}{4 \cdot L^*(\rho)}.\] 

\paragraph{Analysis of the incumbent $C = I$.}
To compute $\tilde{\alpha}_I$ and $\tilde{\reg}_I$, we apply Lemma \ref{lemma:extensionridgelessoptimalinfinitedata}. 
The assumption $\constrlead \ge \mathbb{E}_{\DC}[\LossAlign(\beta_1, \beta_2, \Sigma, 0.5)]$ in the lemma statement can be rewritten as $\constrlead \ge 0.25 L^*(\rho)$, which guarantees the assumptions in Lemma \ref{lemma:extensionridgelessoptimalinfinitedata} are satisfied. By Lemma \ref{lemma:extensionridgelessoptimalinfinitedata}, we see that:
\[\mathbb{E}_{\DC}[L^*_1(\beta_1, \beta_2, \DF, \tilde{\reg}_I, \infty, \tilde{\alpha}_I)] =  \left(\sqrt{L^*(\rho)} - \sqrt{\min(\constrlead, L^*(\rho)}\right)^2 = G_I. \]

\paragraph{Analysis of the entrant $C = E$.} 
We apply Lemma \ref{lemma:newcompanyalignment} to see 
for sufficiently large constant $C_{\delta, \gamma}$, if 
\[ \Nentr \ge 
C_{\delta, \gamma} \cdot \max\left(\tilde{D}^{-\frac{1}{\scalingexp}}, \tilde{D}^{-\frac{\scalingexp + 1}{\scalingexp}}\left( G_E^{\frac{1}{2}} (1-\rho)^{\frac{1}{2}} + \frac{1}{2} G_I - \frac{1}{2} G_E \right)  \right),\]
then it holds that: 
\[\mathbb{E}_{\DC}[\LossPerf^*(\beta_1, \beta_2, \DF, \tilde{\reg}_E, \Nentr, \tilde{\alpha}_E)] \le G_I =  \mathbb{E}_{\DC}[L^*_1(\beta_1, \beta_2, \DF, \tilde{\reg}_I, \infty, \tilde{\alpha}_I)]. \]
This means that:
\[ \Nentroptmodified = O\left( \max\left(\tilde{D}^{-\frac{1}{\scalingexp}}, \tilde{D}^{-\frac{\scalingexp + 1}{\scalingexp}}\left( G_E^{\frac{1}{2}} (1-\rho)^{\frac{1}{2}} + \frac{1}{2} G_I - \frac{1}{2} G_E \right)  \right)\right)\]
as desired. 
\end{proof}
\end{document}